\documentclass{amsart}
\usepackage{amssymb,latexsym,bbm,geometry}
\usepackage{graphicx}
\usepackage{color}
\usepackage{subfigure}
\usepackage{multirow}
\usepackage{mathptmx}
\usepackage{url,bm}
\usepackage{enumerate}
\usepackage{float}
\usepackage{natbib}
\usepackage{hyperref}
\usepackage[ruled,vlined,linesnumbered]{algorithm2e}
\SetKwInOut{Parameter}{Parameters}

\makeatother
\newcommand{\Second}{\textup{I}\!\textup{I}}
\newtheorem{theorem}{Theorem}[section]
\newtheorem{assumption}{Assumption}[section]
\newtheorem{definition}{Definition}[section]
\newtheorem{lemma}{Lemma}[section]

\newtheorem{proposition}{Proposition}[section]
\newtheorem{remark}{Remark}[section]

\title{Inferring manifolds using Gaussian processes}

\author{David B Dunson}
\address{David B Dunson\\
Department of Statistical Science\\
Duke University, Durham, NC}
\email{dunson@duke.edu}

\author{Nan Wu}
\address{Nan Wu\\
Department of Mathematical Sciences\\
The University of Texas at Dallas, Richardson, TX}
\email{nan.wu@utdallas.edu}

\begin{document}

\maketitle

\begin{abstract}
It is often of interest to infer lower-dimensional structure underlying complex data.  As a flexible class of non-linear structures, it is common to focus on Riemannian manifolds.  Most existing manifold learning algorithms replace the original data with lower-dimensional coordinates without providing an estimate of the manifold or using the manifold to denoise the original data.  This article proposes a new methodology to address these problems, allowing interpolation of the estimated manifold between the fitted data points.  The proposed approach is motivated by the novel theoretical properties of local covariance matrices constructed from samples near a manifold.  Our results enable us to turn a global manifold reconstruction problem into a local regression problem, allowing for the application of Gaussian processes for probabilistic manifold reconstruction.  In addition to the theory justifying our methodology, we provide simulated and real data examples to illustrate the performance.
\end{abstract}

{\it Keywords:} Data denoising; Dimension reduction; Manifold learning; Manifold reconstruction; Nonparametric regression

\section{Introduction}\label{introduction}

It is common to characterize low-dimensional structures in data using Riemannian manifolds. 
The literature on manifold learning focuses mainly on dimensionality reduction algorithms; some well-known examples include Isomap \citep{tenenbaum2000global}, Laplacian Eigenmaps \citep{belkin2003laplacian}, Locally Linear Embeddings \citep{roweis2000nonlinear}, and Diffusion Maps \citep{coifman2006diffusion}.  The focus of this article is instead on developing a new local regression-based approach to manifold reconstruction, which goes beyond dimensionality reduction to exploit the manifold structure in modeling the original data. 

We provide a concrete illustration through bird vocalization data. Audio files are converted via discrete short-time Fourier transform into spectrograms. In Figure \ref{spectrograms1}, we present five annotated spectrograms of a specific call type of Anthus trivialis (tree pipit) from \cite{ovaskainen2018animal}. {Similar spectrogram data are routinely collected in studying birds, and it is of interest to develop realistic statistical models for such data that can be used to infer variation in vocalizations between individuals and according to geographical location and habitat conditions. With this goal in mind, we propose a method based on characterizing the data as noisy observations around a manifold.}

\begin{figure}[h!]
\centering
\includegraphics[width=14cm,height=5.5 cm]{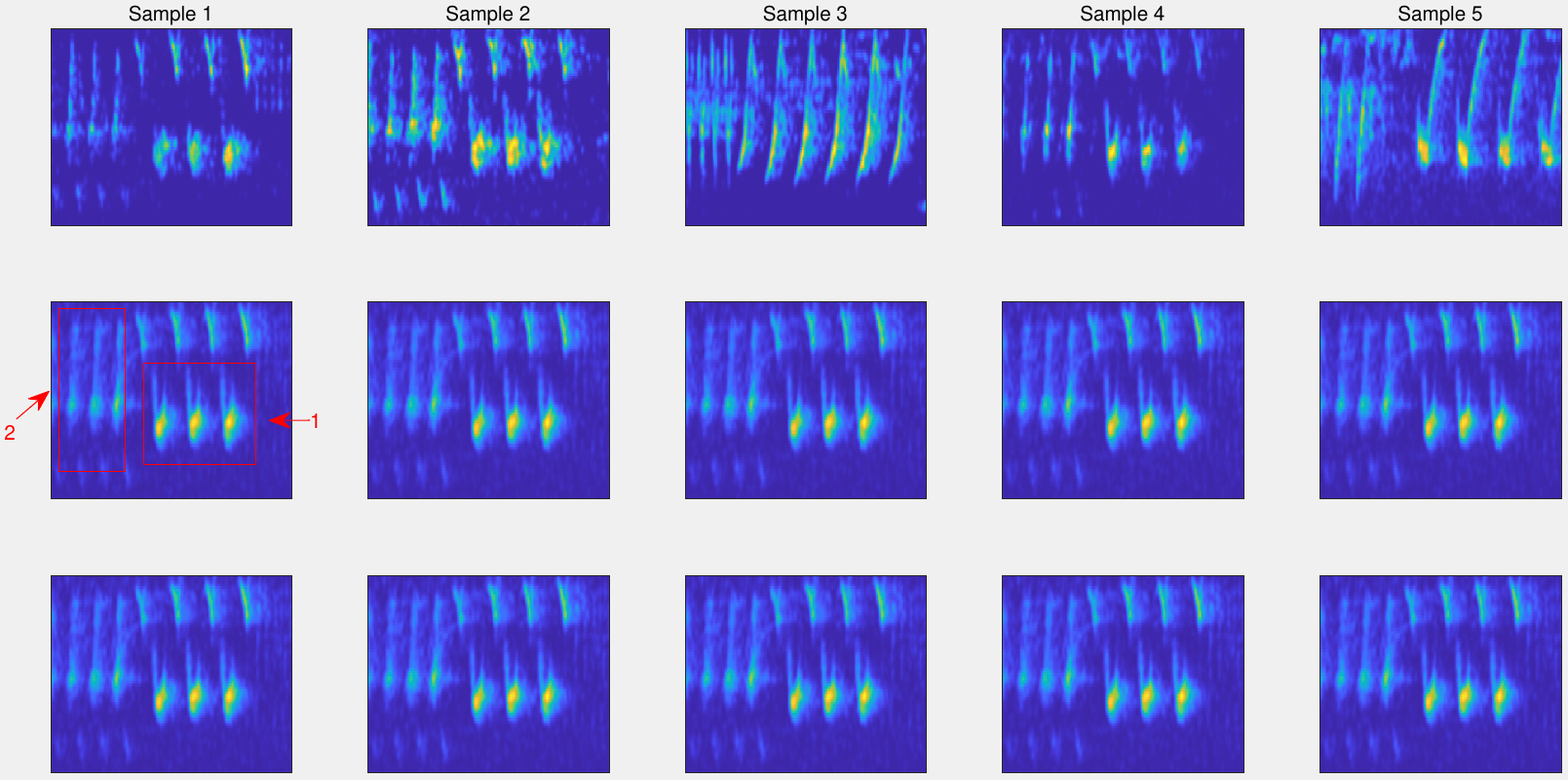}
\caption{Top row: $5$ out of $83$ spectrograms corresponding to a call type of Anthus trivialis. The rows of the $\mathbb{R}^{75 \times 197}$ matrices correspond to frequency (kHz), and the columns correspond to time (ms) and entries measure amplitude of the signal.  
 Samples 2 and 4 are near Sample 1. Second and third rows: $10$ generated samples near Sample 1 using \texttt{MrGap}. Region 2 in the generated sample is similar to Sample 1, while region 1 is similar to Samples 2 and 4.}\label{spectrograms1}
\end{figure}

This article proposes Manifold reconstruction via Gaussian processes (\texttt{MrGap}). Our main contributions are: (1) bias and variance analyses and spectral properties of the local covariance matrix constructed from samples near a manifold; and (2) a novel local regression formulation of manifold reconstruction, enabling application of Gaussian process regression and avoidance of restrictive assumptions on the noise or distribution of the denoised data on the manifold. We apply \texttt{MrGap} to the bird vocalization spectrogram data illustrated in Figure \ref{spectrograms1}, interpolating $10$ points near each sample in the manifold and showing the corresponding spectrograms near sample 1 in Figure \ref{spectrograms1}. The generated samples incorporate the audio information from sample 1 and its neighboring samples.  {Such interpolation provides insights into variation across birds and provides a natural starting point for extensions for inference on covariate effects.} Refer to Section \ref{bird data} for details and results. 

We briefly review the literature relevant to our first main contribution. \citep{singer2012vector} analyzes the bias and variance of the local covariance matrix constructed through a smooth kernel, assuming samples without noise from a closed submanifold.  Focusing on the local covariance constructed from a $0-1$ kernel, \citep{Tyagi2013} consider data having specific distributions on closed manifolds while \citep{wu2018think, wu2018locally} assume samples without noise from a closed submanifold and a compact submanifold with boundary. \citep{kaslovsky2014non} studies the local covariance constructed from $k$ nearest neighbors under heavy restrictions on the embedded submanifold and sampling distribution.  We consider much more realistic sampling conditions, and our theory generalizes \citep{wu2018think} to allow Gaussian noise.

Relevant to our second main contribution, there is a literature on manifold reconstruction. Principal curves provide a fitted value for a one-dimensional manifold \citep{hastie1989principal,kirov2017multiple,mao2016principal}, while principal manifolds extend this to higher dimensions \citep{smola2001regularized,meng2021principal}. Assuming Gaussian noise, \citep{pmlr-v75-fefferman18a, fefferman2019fitting} reconstruct a manifold by approximating its normal bundle using noisy samples. \citep{aamari2019nonasymptotic} and \citep{aizenbud2021non} apply local polynomial fitting to reconstruct a manifold. \citep{aamari2019nonasymptotic} assumes bounded noise in the normal direction, while \citep{aizenbud2021non} consider uniformly distributed noise in a tubular neighborhood.  \citep{puchkin2022structure} assume samples with bounded noise and estimate a projection operator onto the manifold. Assuming uniform samples on the manifold and Gaussian noise with known variance, \citep{yao2023manifold} estimate the orthogonal projection map. 

Most of these approaches are highly restrictive in terms of the distributions of data over the manifold and/or noise. Our proposed local regression approach motivated by novel theoretical findings represents an innovative general solution to the manifold modeling problem, which can avoid these restrictive assumptions {and provide a useful starting point for regression with noisy outcomes on an unknown manifold or manifolds.}
We summarize our notations in Table \ref{Table:Notations}. 
\begin{table}[h!]
\caption{Commonly used notations in this paper.}\label{Table:Notations}
\begin{tabular}{|l|l|} 
\hline $Symbol$ & $Meaning$\\ 
\hline\hline 
$M$ &  $d$-dimensional smooth Riemannian manifold \\ 
$\iota$ & An isometric embedding of $M$ in $\mathbb{R}^D$\\
$\{e'_i\}_{i=1}^d$, $\{e_i\}_{i=1}^D$ & The standard orthonormal basis of $\mathbb{R}^d$ and $\mathbb{R}^D$ respectively \\
$P$ & Probability density function on $M$\\
$\{x_i\}_{i=1}^n$ & Points sampled based on $P$ from $M$ \\ 
$\{\eta_i\}_{i=1}^n$ & Noise sampled from a Gaussian distribution in $\mathbb{R}^D$ \\
$\{ y_i\}_{i=1}^n$ & Noisy data points around $\iota(M)$\\ 
$\{\hat{y}_i\}_{i=1}^n$ & Denoised data points corresponding to $\{y_i\}_{i=1}^n$\\ 
$\epsilon$ & The bandwidth of the local covariance matrix\\
$C_{n,\epsilon}(y_k)$ & The local covariance matrix at $y_k$ 
with bandwidth $\epsilon$ \\
$J$, $\bar{J}$ & The projections from $\mathbb{R}^D$ to the subspaces $\mathbb{R}^d$ and $\mathbb{R}^{D-d}$ respectively  \\
$\mathcal{P}_{y_k}$, $\mathcal{P}^\bot_{y_k}$ & Two major operators in constructing local regression \\
$\delta$ &  A scale in $\mathbb{R}^D$ to perform $\mathcal{P}_{y_k}$, $\mathcal{P}^\bot_{y_k}$  in the algorithm\\
$G(S_1, S_2)$ & The geometric root mean square error between subsets $S_1$ and $S_2$ in $\mathbb{R}^D$\\
\hline 
\end{tabular}
\end{table}

\section{Manifold reconstruction via Gaussian processes}
\subsection{Setup of the manifold reconstruction problem}
In this section, we propose a method to reconstruct an unknown embedded submanifold of Euclidean space from noisy samples. We first make the following assumption about the samples.
\begin{assumption}\label{manifod with noise}
$M$ is a $d$-dimensional smooth, closed, and connected Riemannian manifold isometrically embedded in $\mathbb{R}^D$ though $\iota: M\to \mathbb{R}^D$. $P$ is a probability density function on $M$, which is smooth, bounded from below by $P_{m}>0$ and bounded from above by $P_{M}$. The observed data follow
$y_i=\iota(x_i)+\eta_i$, with $x_i \sim P$ and $\eta_i \sim \mathcal{N}(0, \sigma^2 I_{D \times D})$ independently for $i=1,\ldots,n$.  
\end{assumption}

The information of $\iota(M)$ is not accessible and we are only given the Euclidean coordinates of $\{y_i\}_{i=1}^n$.  For each $y_k$, our goals are to (i) estimate a fitted value $\hat{y}_k \in \iota(M)$ while controlling $\frac{1}{n} \sum_{i=1}^n \|\iota(x_i)-\hat{y}_i\|^2_{\mathbb{R}^D}$; (ii) interpolate around
$\hat{y}_k$ using a local parameterization of $\iota(M)$. 

{Under Assumption \ref{manifod with noise}, we provide a method to estimate the dimension $d$ in Section \ref{determine the dimension} of the Supplementary Material. The assumption that $M$ is connected is made for simplicity in the statements of the main theorems and is not necessary for \texttt{MrGap} to function.  When $M$ is the disjoint union of  smooth and closed Riemannian manifolds of the same dimension, our algorithm can be applied directly to denoise and interpolate on $\iota(M)$, without the need to test whether $M$ is connected. In contrast, when $M$  is the disjoint union of  smooth and closed Riemannian manifolds of different dimensions, we propose a method to reconstruct $\iota(M)$ by applying \texttt{MrGap} to the samples corresponding to each connected component.  Refer to Section \ref{section:disconnected} of the Supplementary Material for details and numerical simulations.} In Section \ref{subsection: performance on manifold with boundary} of the Supplementary Material, we illustrate the performance of \texttt{MrGap} when $M$ is a compact manifold with boundary.

There are two main ingredients in our methodology: the local covariance matrix and Gaussian process regression. We review these briefly in the following two subsections. 

\subsection{Local covariance matrix}
Let $\{e_i\}_{i=1}^D$ be the standard orthonormal basis of $\mathbb{R}^D$, where $e_i=[0,\cdots, 1, \cdots, 0]^\top$ with $1$ in the $i$-th entry. Similarly, let $\{e'_i\}_{i=1}^d$ be the standard orthonormal basis of $\mathbb{R}^d$.  Letting $\chi(t) =1$ for $t \in [0,1]$ and $\chi(t)=0$ for $t>1$ denote a $0-1$ kernel, we define a local covariance matrix at $y_k$, 
\begin{align}\label{Gaussian local covariance matrix}
C_{n,\epsilon}(y_k)=\frac{1}{n} \sum_{i=1}^n (y_i-y_k)(y_i-y_k)^\top \chi\Big(\frac{\|y_i-y_k\|_{\mathbb{R}^D}}{\epsilon}\Big),
\end{align}
where $\epsilon>0$. Consider the eigen decomposition 
\begin{align}\label{eigendecompPCA}
C_{n,\epsilon}(y_k)  =U_{n,\epsilon}(y_k)\Lambda_{n,\epsilon}(y_k) U_{n,\epsilon}(y_k)^\top,
\end{align}
with $U_{n,\epsilon}(y_k) \in \mathbb{O}(D)$  and $\Lambda_{n,\epsilon}(y_k) \in \mathbb{R}^{D \times D}$ a diagonal matrix of eigenvalues of $C_{n,\epsilon}(y_k)$ in decreasing order,  $e_1^\top\Lambda_{n,\epsilon}(y_k)e_1 \geq e_2^\top\Lambda_{n,\epsilon}(y_k)e_2 \geq \cdots \geq e_D^\top\Lambda_{n,\epsilon}(y_k)e_D$. The column vectors of $U_{n,\epsilon}(y_k)$ are orthonormal eigenvectors in decreasing order of eigenvalues. The eigenvectors form an orthonormal basis of $\mathbb{R}^{D}$.  We define two operators in $\mathbb{R}^D$ associated with $U_{n,\epsilon}(y_k)$  as follows.  

Define projection matrices $J \in \mathbb{R}^{D \times d}$ such that $J_{ij}=1$ when $i=j$ and $J_{ij}=0$ when $i \not= j$ and  $\bar{J} \in \mathbb{R}^{D \times (D-d)}$ such that $\bar{J}_{ij}=1$ when $i=d+j$ and $\bar{J}_{ij}=0$ otherwise, so that  
$J$ and $\bar{J}$ are projections from $\mathbb{R}^D$ to subspaces $\mathbb{R}^d$ and $\mathbb{R}^{D-d}$ respectively.
For any $y \in \mathbb{R}^D$, let $\mathcal{P}_{y_k}(y): \mathbb{R}^D \rightarrow \mathbb{R}^d$ with
\begin{align}\label{regression PYK 0}
\mathcal{P}_{y_k}(y)=J^\top U_{n,\epsilon}(y_k)^\top (y-y_k),
\end{align}
 and let $\mathcal{P}^\bot_{y_k}(y): \mathbb{R}^D \rightarrow \mathbb{R}^{D-d}$ with
\begin{align}\label{regression PbotYK 0}
\mathcal{P}^\bot_{y_k}(y)=\bar{J}^\top U_{n,\epsilon}(y_k)^\top (y-y_k).
\end{align}
The operators $\mathcal{P}_{y_k}$ and $\mathcal{P}^\bot_{y_k}$ depending on $d$ are key components in our methodology.

\subsection{Gaussian process regression }
Suppose $F:\mathbb{R}^d \rightarrow \mathbb{R}^q$ is an unknown regression function with $F=(f_1, \cdots, f_q)^\top$. Letting $z_i\in \mathbb{R}^q$ denote the response vector and $u_i \in \mathbb{R}^d$ the predictor vector, for $i=1,\ldots, N$, we have
\begin{equation}\label{standard GP model}
z_i=F (u_i) +\eta_i,\qquad 
\eta_i \sim \mathcal{N}(0, \sigma^2 I_{q \times q}).
\end{equation}
We assign independent Gaussian process priors to each $f_j$ with mean $0$ and covariance function 
\begin{equation}\label{GP cov}
\mathsf{C}(u, u')=A \exp\Big(-\frac{\|u-u'\|^2_{\mathbb{R}^d}}{\rho}\Big),
\end{equation}
implying $\tilde{z}_j^\top=\{f_j(u_1), \ldots f_j(u_N)\}^\top \sim \mathcal{N}(0, \Sigma_1)$, for $j=1,\ldots,q$, where 
$\Sigma_1 \in \mathbb{R}^{N \times N}$ has $(j,k)$ element $\mathsf{C}(u_j,u_k)$, for $1 \leq j,k \leq N$.  
Let $Z \in \mathbb{R}^{N \times q}$ with $i$th row $z^\top_i$ for $i=1,\ldots,N$ and $j$th column $\tilde{z}_j$ for $j=1,\ldots, q$. Marginalizing over the Gaussian process priors, we obtain the log-likelihood 
\begin{align}\label{marginal likelihood}
\log p(Z| A, \rho, \sigma) = -\mathrm{tr}\{ z^\top (\Sigma_1+\sigma^2 I_{N \times N})^{-1}z\} -q \log\{ \det(\Sigma_1+\sigma^2 I_{N \times N})\}-\frac{q N}{2}\log(2\pi).
\end{align}
To estimate the covariance parameters $A,\rho,\sigma$ via an empirical Bayes approach that protects against overfitting, we maximize this marginal likelihood. Refer to Section \ref{review GPR}  of the Supplementary Material for a detailed discussion and the derivation of \eqref{marginal likelihood}.

\subsection{Main challenges and motivation}\label{Main challenges and the motivation}
Suppose $y$ is a point on the embedded manifold $\iota(M)$. The tangent space of $\iota(M)$ at $y$, $T_y \iota(M)$, is a subspace of $\mathbb{R}^D.$ $y+T_y \iota(M)$ is the affine subspace tangent to $\iota(M)$ at $y$. It is well known that $\iota(M)$ can be locally parametrized as a graph of a function over $y+T_y \iota(M)$. Consider neighborhood $B^{\mathbb{R}^D}_{\delta}(y) \cap \iota(M)$ for some $\delta>0$. Then, there is a function $F(u):O \subset  \mathbb{R}^{d} \rightarrow \mathbb{R}^{D-d}$ with $F(0)=0$ and $\nabla F(0)=0$ such that the graph of $F$ over $O$, after a rotation and a translation in $\mathbb{R}^D$, is a parametrization of $B^{\mathbb{R}^D}_{\delta}(y_k) \cap \iota(M)$. The parametrization can be expressed as $\Phi(u)=y+U {\scriptsize\begin{bmatrix}
u \\
F(u) \\
\end{bmatrix}}: O \rightarrow \mathbb{R}^D$, 
with $U \in \mathbb{O}(D)$ an orthogonal matrix mapping from the subspace generated by $e_1, \cdots, e_d$ to $T_y \iota(M)$. Refer to Figure \ref{illustration 1} and the technical description in Proposition \ref{uniform coordinates} in the Supplementary Material. Since $\iota(M)$ is unknown, we cannot recover the tangent spaces of $\iota(M)$ from noisy data. Thus, we do not know the matrix $U$ or the function $F$. The key idea of the paper is to take advantage of the properties of local covariance to construct alternative affine subspaces where $\iota(M)$ can be parametrized locally by the graph of a function. This allows us to conduct manifold denoising and interpolation by locally applying nonparametric regression. 
\begin{figure}[h!]
\centering
\includegraphics[width=0.6 \columnwidth]{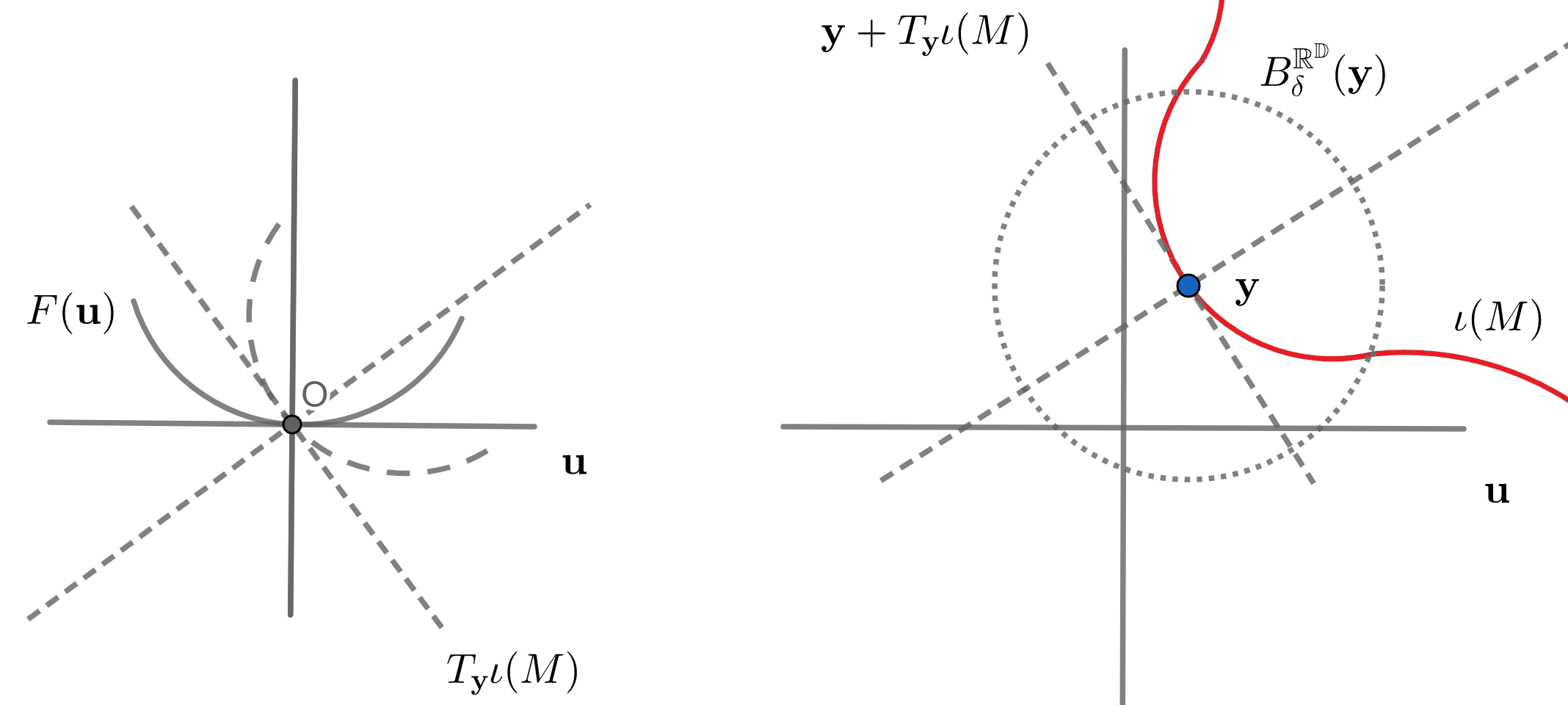}
\caption{Left Panel: The solid lines are the standard coordinates of $\mathbb{R}^D$. The solid curve is the plot of $[u, 
F(u)]^\top$.  The dashed lines show the rotation of the coordinates under $U$ and the plot of $U [u, 
F(u)]^\top$. Right Panel: The solid curve is $\iota(M)$. After a translation by $y$, the graph of $y+U [u, 
F(u)]^\top$  gives a parametrization of $B^{\mathbb{R}^D}_{\delta}(y) \cap \iota(M)$.
}\label{illustration 1}
\end{figure}

\subsection{Denoising phase of MrGap}\label{denoise phase introduction}
Let $\mathcal{H}_k=y_k+\mathcal{T}_k$ be an affine subspace through $y_k$, with $\mathcal{T}_k$ the subspace generated by the first $d$ eigenvectors of $C_{n,\epsilon}(y_k)$ 
  in \eqref{Gaussian local covariance matrix}.  In Section 3, we show that under certain relations between bandwidth $\epsilon$, noise variance $\sigma^2$,  and sample size $n$, with high probability, the neighborhood $B^{\mathbb{R}^D}_{3\delta}(y_k) \cap \iota(M)$ can be parametrized as the graph of a function over $\mathcal{H}_k$ for $\delta>\epsilon$.  Refer to Section \ref{tangent space comparison} in the Supplementary Material for an illustration of $\mathcal{H}_k$ and more discussions. Importantly, the affine subspace $\mathcal{H}_k$ is not  required to accurately approximate the 
affine subspace tangent to $\iota(M)$ at $\iota(x_k)$. However, the parameterization
can be expressed similarly as when the affine subspace is tangent to the manifold:
\begin{align}\label{Section 2: eq local parametrization}
\Phi_k(u)=y_k+U_{n,\epsilon}(y_k){\scriptsize\begin{bmatrix}
u \\
F_k(u)\\
\end{bmatrix}}
\end{align} 
where $F_k(u): O_k \rightarrow \mathbb{R}^{D-d}$ for a connected open set $O_k$ containing $0$ in $\mathbb{R}^d$ and  $U_{n,\epsilon}(y_k)$ is the orthogonal matrix that maps the subspace generated by $e_1, \cdots, e_d$ to $\mathcal{T}_k$ as defined in \eqref{eigendecompPCA}.  
The parameterization $\Phi_k(u)$ is constructed from a rotation under $U_{n,\epsilon}(y_k)$ and a translation under $y_k$ of the graph of $F_k(u)$, and we have 
\begin{align}\label{denoised equation 00}
\hat{y}_k = y_k+U_{n,\epsilon}(y_k){\scriptsize\begin{bmatrix}
0 \\
F_k(0)\\
\end{bmatrix}}
\end{align}
is a point on $\iota(M)$. Hence, if we can estimate $F_k(0)$, then \eqref{denoised equation 00} provides a fitted value for $y_k$.

We rely on Gaussian process regression to infer $F_k$ and hence $F_k(0)$. The operations  $\mathcal{P}_{y_k}$  and $\mathcal{P}^\bot_{y_k}$ in \eqref{regression PYK 0} and \eqref{regression PbotYK 0} are compositions of
(a) a translation under $-y_k$,
(b) a rotation under $U_{n,\epsilon}(y_k)^\top$, and
(c) projections onto $\mathbb{R}^d$ and $\mathbb{R}^{D-d}$, respectively. 
Since $\Phi_k(u)$ is a parametrization of $B^{\mathbb{R}^D}_{3\delta}(y_k) \cap \iota(M)$ and there is noise in the data, we use points in a smaller neighborhood of $y_k$ to estimate $F_k(0)$: $B^{\mathbb{R}^D}_{\delta}(y_k) \setminus y_k \cap \{y_1 , \ldots, y_n\}=\{y_{k,1} , \ldots, y_{k,N_k}\}$. Refer to Section \ref{Theoretical analysis main article} for details.  Then, $w_{k,j}=\mathcal{P}_{y_k}(y_{k,j})$ are the inputs for $F_k$ and $z_{k,j}=\mathcal{P}^\bot_{y_k}(y_{k,j})$  are the response variables for $j=1,\ldots,N_k$. Refer to Figure \ref{illustration 2}.

\begin{figure}[ht]
\centering
\includegraphics[width=0.7 \columnwidth]{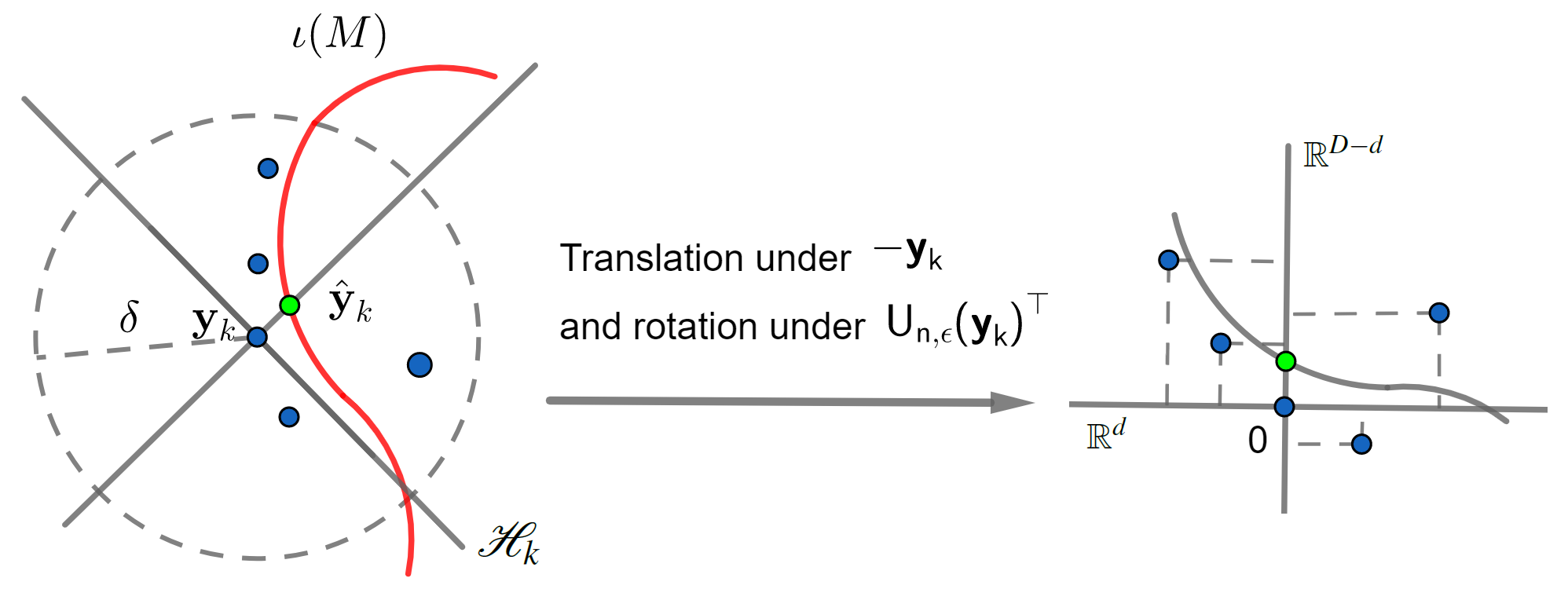}
\caption{Left Panel: The solid curve is $\iota(M)$.  The solid lines are affine subspace $\mathcal{H}_k$ and affine normal subspace of $\mathcal{H}_k$. The blue points are noisy data points in the $\delta$ neighborhood of $y_k$.  Right Panel: The solid lines are coordinates in $\mathbb{R}^D$ and the graph of function $F_k$. After (a) a translation under $-y_k$ and (b) a rotation under $U_{n,\epsilon}(y_k)^\top$, blue points in the left panel (the noisy data points) are mapped to blue points in the right panel. In particular, $y_k$ is mapped to the origin.  If we project the blue points except $0 \in \mathbb{R}^D$ in the right panel onto $\mathbb{R}^d$ and $\mathbb{R}^{D-d}$ which are the last operations (c) in $\mathcal{P}_{y_k}$  and $\mathcal{P}^\bot_{y_k}$ , then we have inputs $w_{k,j}$ and response variables $z_{k,j}$ for $F_k$ respectively. The green point in the right panel is $[0, F_k(0)]^\top$. We apply \eqref{denoised equation 00} to get the denoised output $\hat{y}_k$ indicated by the green point in the left panel.
}\label{illustration 2}
\end{figure}

Hence, the denoising phase of \texttt{MrGap} consists of the following steps.  Inputs include data $\{y_i\}_{i=1}^n$, a bandwidth $\epsilon$ of the local covariance matrix $C_{n,\epsilon}(y_k)$, and a scale $\delta$ of the maps $\mathcal{P}_{y_k}$  and $\mathcal{P}^\bot_{y_k}$. The output is the denoised data $\{ \hat{y}_1, \cdots, \hat{y}_n \}$.  In Step 1, for each data point $y_k$, we flag nearby points $\{y_{k,1} , \cdots, y_{k,N_k}\}$ in $B^{\mathbb{R}^D}_{\delta}(y_k) \setminus y_k$. We construct the predictors and responses for a regression function $F_k$ as described above. In Step 2, we calculate the covariance matrix between $0$ and $\{w_{k,j}\}_{j=1}^{N_k}$. 
In Step 3, we apply Gaussian process regression to estimate $F_k(0)$ and denoise the data, using the covariance matrix in Step 2, while estimating the tuning parameters. Let $\mathcal{A}_k$ be an affine subspace defined as 
\begin{equation}\label{affine space of posterior}
\mathcal{A}_k=y_k+V_k,
\end{equation} 
where $V_k$ is the $D-d$ dimensional subspace generated by the last $D-d$ column vectors of $U_{n,\epsilon}(y_k)$. The predictive distribution of the denoised $y_k$  is Gaussian in $\mathcal{A}_k$ centered at $\hat{y}_k$. We iteratively apply steps 1 to 3 to refine the results. Figure \ref{illustration 2} illustrates the first 3 steps, and Step 4 is discussed in more detail in Section \ref{Theoretical analysis main article}.  Detailed steps are summarized in Algorithm \ref{MrGap1}.

\begin{algorithm}[h!]
\SetAlgoLined
For each $y_k$, construct $C_{n,\epsilon}(y_k)$ as in \eqref{Gaussian local covariance matrix}. 
Let $\{y_{k,1} , \cdots, y_{k,N_k}\}$ denote samples in $B^{\mathbb{R}^D}_{\delta}(y_k) \setminus y_k$, 
 $w_{k,j}=\mathcal{P}_{y_k}(y_{k,j})$ and $z_{k,j}=\mathcal{P}^\bot_{y_k}(y_{k,j})$ for $j=1, \cdots, N_k$, with  $\mathcal{P}_{y_k}(y)$ and $\mathcal{P}^\bot_{y_k}(y)$ defined in \eqref{regression PYK 0} and \eqref{regression PbotYK 0}.  The $k$th regression has responses $z_{k,j}$ and predictors $w_{k,j}$.

For each $y_k$, construct  covariance $\Sigma_k= {\scriptsize\begin{bmatrix}
\Sigma_{k,1} & \Sigma_{k,2}  \\
\Sigma_{k,3} & \Sigma_{k,4}
\end{bmatrix}} \in \mathbb{R}^{(N_k+1) \times (N_k+1)}$ over $\{w_{k,1}, \cdots, w_{k,N_k}, 0 \}$ induced by $\mathsf{C}$ in \eqref{GP cov}, with $\Sigma_{k,1} \in \mathbb{R}^{N_k \times N_k}$.  
Denote $Z_k \in \mathbb{R}^{N_k \times (D-d)}$ with $i$th row $z_{k,i}^\top$.

For the $k$th local regression, calculate the log marginal likelihood $ \log  p_k (Z_k| A, \rho, \sigma)$ using \eqref{marginal likelihood} with $Z=Z_k$, $\Sigma_1 = \Sigma_{k,1}$ and $N=N_k$. Estimate $A,\rho,\sigma$ by maximizing the sum of log marginal likelihoods $L=\sum_{k=1}^n \log  p_k (Z_k| A, \rho, \sigma)$.
The predictive distribution of the denoised output of $y_k$  is Gaussian in $\mathcal{A}_k$ with mean
\begin{align} \label{denoised yk}
\hat{y}_k=y_k+U_{n,\epsilon}(y_k){\scriptsize\begin{bmatrix}
0 \\
\{\Sigma_{k,3}(\Sigma_{k,1}+{\sigma}^2 I_{N_k \times N_k})^{-1}Z_k\}^\top \\
\end{bmatrix}},
\end{align}
and covariance $\big\{ \Sigma_{k,4}-\Sigma_{k,3} (\Sigma_{k,1}+\sigma^2 I_{N_k \times N_k})^{-1}\Sigma_{k,2} \big\}  I_{(D-d) \times (D-d)}$,
where $U_{n,\epsilon}(y_k)$ and $\mathcal{A}_k$ are defined in \eqref{eigendecompPCA} and \eqref{affine space of posterior},  respectively. 

Repeat Steps 1-3 using the denoised data in Step 3 as the input in Step 1. Stop iterating when the change in $\sigma$ is below a small tolerance and output the final denoised data. 
\caption{MrGap denoising steps to produce estimates of $\hat{y}_1, \cdots, \hat{y}_n$ from $y_1,\ldots,y_n$.}\label{MrGap1}
\end{algorithm}

\subsection{Interpolation phase of the MrGap algoirthm }

Interpolation produces additional points on $\iota(M)$.  In the presence of non-negligible noise, Algorithm \ref{MrGap1} needs to be iterated until convergence; refer to Section \ref{Discussion of the main algorithm}. For illustration purposes, here we assume that the noise is small.  Consequently, we are allowed to use the noisy data $\{y_i\}_{i=1}^n$ directly as inputs in the interpolation phase and adopt the same notation as in the previous section for clarity. Recall that \eqref{Section 2: eq local parametrization} is a local parametrization of $\iota(M)$. A straightforward approach to interpolation is to generate new inputs $\{\tilde{u}_{k,1},\cdots, \tilde{u}_{k, K}\}$ randomly within domain $O_k$ and then apply 
\begin{align}\label{denoised equation 11}
y_k+U_{n,\epsilon}(y_k){\scriptsize\begin{bmatrix}
\tilde{u}_{k,j} \\
F_k(\tilde{u}_{k,j}) \\
\end{bmatrix}}
\end{align}
to produce the $j$th point on $\iota(M)$ around the denoised point $\hat{y}_k$, for $j=1,\ldots,K$. 

There are two issues involved. First, we need to ensure that the inputs $\{\tilde{u}_{k,i}\}_{i=1}^K$ are distributed within the domain $O_k$. To achieve this, we use the predictors $\{w_{k,j}\}_{j=1}^{N_k}$ to estimate a ball in $O_k$ and then generate $\{\tilde{u}_{k,i}\}_{i=1}^K$ uniformly in that ball. The second issue is that the interpolated points around $\hat{y}_k$ and around $\hat{y}_{k'}$ may have noticable gaps in the intersection region of their local neighborhoods due to statistical errors in the estimation of $F_k$ and $F_{k'}$. To address this issue, we perform iterative interpolation around each $\hat{y}_k$ for $k$ from $1$ to $n$.   Specifically, we identify all interpolated points around every $\hat{y}_{k'}$ with $k' < k$, falling within $B^{\mathbb{R}^D}_{\delta}(y_k)$, denoted as $\tilde{M}_{k-1}=\{\tilde{y}_{k,j}\}_{j=1}^{L_k}$. In fitting the regression for $F_k$, we augment the predictors $\{w_{k,j}\}_{j=1}^{N_k}$ with $\tilde{w}_{k,j}=\mathcal{P}_{y_k}(\tilde{y}_{k,j})$ for $j=1 ,\ldots, L_k$ and the responses $\{z_{k,j}\}_{j=1}^{N_k}$ with $\tilde{z}_{k,j}=\mathcal{P}^\bot _{y_k}(\tilde{y}_{k,j})$  for  $j=1 ,\ldots, L_k$. Then, by applying \eqref{denoised equation 11} we produce smoother interpolations in the intersection regions. Refer to Figure \ref{illustration 3 interpolation phase} for an illustration. The interpolation steps are summarized in Algorithm \ref{MrGap2} in Section \ref{Interpolation phase} of the Supplementary Material.  In Section \ref{numerical example curve} of the Supplementary Material, we demonstrate that results of interpolation are robust to permutations of the order of $\{\hat{y}_i\}_{i=1}^n$. 

\begin{figure}[h!]
\centering
\includegraphics[width=0.7 \columnwidth]{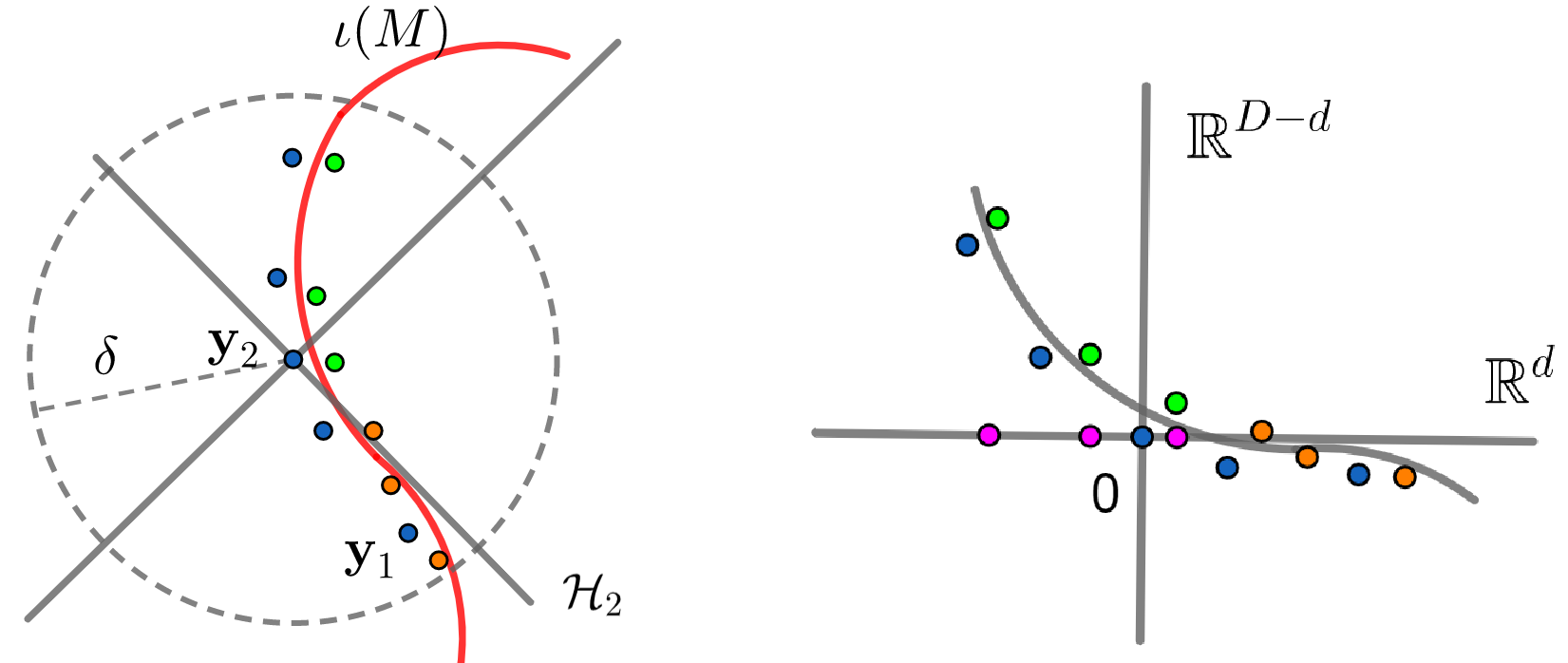}
\caption{We focus our discussion on the point $\textbf{y}_2$.  Left Panel: The blues points are $B^{\mathbb{R}^D}_{\delta}(\textbf{y}_2)\cap\{y_i\}_{i=1}^n$. Suppose we interpolated $K$ points around $\textbf{y}_1$ on $\iota(M)$ whose union is $\tilde{M}_{1}$.  We find  $\tilde{M}_{1} \cap B^{\mathbb{R}^D}_{\delta}(\textbf{y}_2)$ (orange points). Right Panel: The solid curve is the graph of the function $F_2$. After a translation under $-\textbf{y}_2$ and a rotation under $U_{n,\epsilon}(\textbf{y}_2)^\top$, the blue points and the orange points in the left panel are mapped to the blue points and orange points in the right panel. The purple points $\{\tilde{\textbf{u}}_{2,j}\}_{j=1}^K$ are generated in the domain $O_2$ of $F_2$ in $\mathbb{R}^d$ for the prediction. The blue points except $0$ and the orange points are projected onto  $\mathbb{R}^d$ and $\mathbb{R}^{D-d}$ to be the predictors and the response variables for the predictions whose coordinates are $\{[\tilde{\textbf{u}}_{2, j}, F_2(\tilde{\textbf{u}}_{2, j})]^\top\}_{j=1}^K$ (green points in the right panel). The green points in the right panel are mapped back to green points in the left panel through \eqref{denoised equation 11}  which are the interpolations around $\textbf{y}_2$.
}\label{illustration 3 interpolation phase}
\end{figure}

\subsection{Selection of parameters $\epsilon$ and $\delta$}
Let $L$ denote the sum of log marginal likelihood in Step 3 of Algorithm \ref{MrGap1}. We propose estimating $\epsilon$ and $\delta$ by maximizing $L$ over all $\epsilon \leq \delta$ in the first round of Algorithm.  $L$ has an analytic expression in terms of the covariance parameters  $A, \rho$, and $\sigma$ for a fixed pair $(\epsilon,\delta)$, which allows us to apply gradient descent to the covariance parameters along with a grid search on $(\epsilon,\delta)$ to maximize $L$. To improve efficiency, a lower bound $b$ on $\epsilon$ is imposed such that each $C_{n, b}(y_k)$ has at least $d$ positive eigenvalues for all $k$. Refer to Section \ref{bandwidth selection example} of the Supplementary Material for an example.


\section{Theoretical analysis}\label{Theoretical analysis main article}
\subsection{Geometric preliminaries}
Suppose $x \in M$. Let $T_x M$ be the tangent space of $M$ at $x$. Let $T_{\iota(x)}\iota(M)$ be the tangent space of $\iota(M)$ at $\iota(x)$ which is a subspace of $\mathbb{R}^D$. Let $T_{\iota(x)}\iota(M)^{\bot}$ be the normal subspace of $T_{\iota(x)}\iota(M)$ in $\mathbb{R}^D$.  Recall the definition of a chart for $\iota(M)$. For any $x \in M$, we can find an open set $U_x \subset \mathbb{R}^D$ with $\iota(x) \in U_x$.  $V_x$ is an open topological ball in $\mathbb{R}^d$. A chart of $\iota(M)$ over $U_x \cap \iota(M)$ is a map $\Phi_x: V_x \rightarrow U_x \cap \iota(M)$ such that $\Phi_x$ is a diffeomorphism, a smooth bijective map with a smooth inverse. 
The reach $\tau_{\iota(M)}$ of $\iota(M)$ is the largest number such that any point in $\mathbb{R}^D$ at distance less than the reach from $\iota(M)$ has a unique nearest point on $\iota(M)$ \citep{federer1959curvature}. Refer to Section \ref{Basic concepts} of the Supplementary Material for detailed definitions.

\subsection{Local covariance matrix on manifold with noise}
$C_{n,\epsilon}(y_k)$ from \eqref{eigendecompPCA} is invariant under the translation of $\iota(M)$ in $\mathbb{R}^D$.  Hence, to simplify the notation, we make the following assumption for the results in this subsection.
\begin{assumption}\label{assumption trans and rot}
For any fixed $x_k \in M$, we first translate $\iota(M)$ so that  $\iota(x_k)=0 \in \mathbb{R}^D$. Then, we apply an orthogonal transformation in  $\mathbb{R}^D$ so that $\{e_i\}_{i=1}^d$ form a basis of $T_{\iota(x_k)}\iota(M)$.
\end{assumption}

In the following theorem, we provide bias and variance analyses of $C_{n,\epsilon}(y_k)$. The proofs are in Sections \ref{proof of bias analysis PCA } and  \ref{proof of variance analysis PCA } of the Supplementary Material. 

\begin{theorem}\label{local PCA  variance analysis}
Under Assumptions \ref{manifod with noise}-\ref{assumption trans and rot},  suppose $\epsilon$ is small enough depending on $d$, $D$, scalar curvature of $M$ and second fundamental form of $\iota(M)$. For $\beta>1$, if $\sigma \leq \min \big\{\frac{1}{\sqrt{-4(d+5)\log \epsilon}}, \frac{1}{\sqrt{12\log(2n)}} \big\} \epsilon^\beta$, then for all $x_k$, with probability greater than $1-\frac{1}{n^2}$, we have $\|\eta_k\| _{\mathbb{R}^D}\leq \epsilon^\beta$ and
\begin{align}
C_{n,\epsilon}(y_k)
=\epsilon^{d+2} \bigg\{ \frac{|S^{d-1}|{P}(x_k)}{d(d+2)} 
{\scriptsize 
\begin{bmatrix}
I_{d \times d} & 0 \\
0& 0  \\
\end{bmatrix}}+\mathcal{E}\bigg\}, \nonumber 
\end{align} 
where 
\begin{align}
\mathcal{E}={\scriptsize \begin{bmatrix}
O\{\epsilon^{\min(\beta-1,2)}+\sqrt{\frac{\log n}{n \epsilon^d}}\} & O\{\epsilon^{2\min(\beta-1,1)}+\sqrt{\frac{\log n}{n \epsilon^{d-2\min(\beta-1,1)}}}\} \\
O\{\epsilon^{2\min(\beta-1,1)}+\sqrt{\frac{\log n}{n \epsilon^{d-2\min(\beta-1,1)}}}\} & O\{\epsilon^{2\min(\beta-1,1)}+\sqrt{\frac{\log n}{n \epsilon^{d-2\min(\beta-1,1)}}}\} 
\end{bmatrix}}. \nonumber 
\end{align}
The top left block of $\mathcal{E}$ is a $d$ by $d$ matrix. The constant factors in $\mathcal{E}$ depend on $d$, the $C^2$ norm of $P$, the second fundamental form of $\iota(M)$ and its derivative, and the Ricci curvature of $M$.
\end{theorem}

The following theorem characterizes the orthonormal eigenvector matrix of $C_{n,\epsilon}(y_k)$. The proof relies on applying the Davis-Kahan theorem \citep{davis1970rotation,yu2015useful} to Theorem \ref{local PCA  variance analysis}. The proof of the theorem is given in Section \ref{proof of variance analysis PCA } of the Supplementary Material.

\begin{theorem}\label{local PCA  spectral behavior}
Under Assumptions \ref{manifod with noise}-\ref{assumption trans and rot}, for $\beta>1$,  suppose $\epsilon$ is small enough depending on $\beta$, $d$, $D$, $P_m$, $C^2$ norm of $P$, the second fundamental form of $\iota(M)$ and its derivative and the Ricci curvature of $M$.  Suppose $n$ is large enough so $\epsilon^{-d-2\min(\beta-1,1)} \leq \frac{n}{\log n}$. If  $\sigma \leq \min \big\{\frac{1}{\sqrt{-4(d+5)\log \epsilon}}, \frac{1}{\sqrt{12\log(2n)}} \big\} \epsilon^\beta$,  then for all $x_k$, with probability greater than $1-\frac{1}{n^2}$, we have $\|\eta_k\| _{\mathbb{R}^D}\leq \epsilon^\beta$, and
$U_{n,\epsilon}(y_k)={\scriptsize \begin{bmatrix}
X_1 & 0 \\
0& X_2  \\
\end{bmatrix}}+O\{\epsilon^{2\min(\beta-1,1)}\}$,
where $X_1 \in \mathbb{O}(d)$,  $X_2 \in \mathbb{O}(D-d)$.  $O\{\epsilon^{2\min(\beta-1,1)}\}$ is a $D$ by $D$ matrix whose entries are of order $O\{\epsilon^{2\min(\beta-1,1)}\}$, where constant factors depend on $d$, $D$, $P_m$, $C^2$ norm of $P$, the second fundamental form of $\iota(M)$ and its derivative and the Ricci curvature of $M$.
\end{theorem}

We discuss Theorem \ref{local PCA  spectral behavior} from the following two aspects. 

\textbf{Interpretation of the conclusion.} We explain the dominating and error terms in the expansion of $U_{n,\epsilon}(y_k)$. By Assumption \ref{assumption trans and rot},  ${\scriptsize\begin{bmatrix}
X_1 \\
0  \\
\end{bmatrix}}$ forms an orthonormal basis of $T_{\iota(x_k)}\iota(M)$ and  ${\scriptsize\begin{bmatrix}
0 \\
X_2 \\
\end{bmatrix}}$ forms an orthonormal basis of $T_{\iota(x)}\iota(M)^\bot$.  We  use the first $d$ column vectors of $U_{n,\epsilon}(y_k)$ to construct a map $\mathcal{P}_{y_k}(y)$ from $\mathbb{R}^D$ to $\mathbb{R}^d$ defined in \eqref{regression PYK 0}. In the proof of Theorem \ref{Construction of a chart main theorem}, we allow  the first $d$ column vectors of $U_{n,\epsilon}(y_k)$ to deviate from ${\scriptsize \begin{bmatrix}
X_1 \\
0  \\
\end{bmatrix}}$ by an order $1$ term, while $\mathcal{P}_{y_k}(y)$ on $\iota(M)$ still remains a local diffeomorphism. Therefore, it suffices that $U_{n,\epsilon}(y_k)-{\scriptsize\begin{bmatrix}
X_1 & 0 \\
0& X_2  \\
\end{bmatrix}}$   is of order $O\{\epsilon^{2\min(\beta-1,1)}\}$.

\textbf{Conditions on $n$, $\epsilon$ and $\sigma$.} For any $\beta>1$, if we expect error between  $U_{n,\epsilon}(y_k)$ and ${\scriptsize\begin{bmatrix}
X_1 & 0 \\
0& X_2  \\
\end{bmatrix}}$ to be $O\{\epsilon^{2\min(\beta-1,1)}\}$, we require $\epsilon$ to be smaller than a constant that depends on the factors in the theorem.  Dependence of $\epsilon$ on $\beta$ can be expressed as $\epsilon \leq C^{\frac{1}{\min(\beta-1,1)}}$ based on the proof of Theorem  \ref{local PCA  spectral behavior}, where $C$ is a constant. Therefore, $\epsilon$ is independent of $\beta$ if $\beta \geq 2$.  Since Theorem  \ref{local PCA  spectral behavior} depends on Theorem  \ref{local PCA  variance analysis}, a lower bound on $n$ depending on $\epsilon$ is needed. We impose $\epsilon^{-d-2\min(\beta-1,1)} \leq \frac{n}{\log n}$ to balance  bias and variance errors in Theorem  \ref{local PCA  variance analysis}.  Finally, we can achieve the desired error bound of $O\{\epsilon^{2\min(\beta-1,1)}\}$ whenever the standard deviation of the noise $\sigma$ is of order $O(\epsilon^\beta)$, up to a square root logarithmic factor. 

\subsection{Construction of a chart from noisy samples}
In the following theorem, we show that, with high probability, the map $\mathcal{P}_{y_k}(y)$ in \eqref{regression PYK 0} is a diffeomorphism from $B^{\mathbb{R}^D}_{3\delta}(y_k) \cap \iota(M)$ onto its image for some $\delta>0$. Consequently, its inverse defines a chart of $\iota(M)$.  The proof of the theorem, in Section \ref{proofs of charts}  of the Supplementary Material, relies on Theorem \ref{local PCA  spectral behavior} and an analysis of perturbations of tangent spaces of $\iota(M)$ near $y_k$. 

\begin{theorem}\label{Construction of a chart main theorem}
Under Assumption \ref{manifod with noise},  for $\beta>1$,  suppose $\epsilon$ is small enough depending on $\beta$, $d$, $D$, $P_m$, $C^2$ norm of $P$, the second fundamental form of $\iota(M)$ and its derivative and the Ricci curvature of $M$.  Suppose $n$ is large enough such that $\epsilon^{-d-2\min(\beta-1,1)} \leq \frac{n}{\log n}$. If  $\sigma \leq \min \big\{\frac{1}{\sqrt{-4(d+5)\log \epsilon}}, \frac{1}{\sqrt{12\log(2n)}} \big\} \epsilon^\beta$,  then for all $x_k$ and any $\delta$ such that $\epsilon^\beta <\delta<\frac{\tau_{\iota(M)}}{16}$,  with probability greater than $1-\frac{1}{n^2}$, we have the following facts:

(1) If $y_i \in B^{\mathbb{R}^D}_{\delta}(y_k)$, then $\iota(x_i) \in B^{\mathbb{R}^D}_{3\delta}(y_k)$.

(2) The map $\mathcal{P}_{y_k}(y)$ defined in \eqref{regression PYK 0} is a diffeomorphism from $B^{\mathbb{R}^D}_{3\delta}(y_k) \cap \iota(M)$ onto its image $O_k \subset \mathbb{R}^d$. $O_k$ is homeomorphic to $B^{\mathbb{R}^d}_{1}(0)$.

(3) There is a point $u_{k,0}$ in $O_k$ such that $0 \in B^{\mathbb{R}^d}_{R_k}(u_{k,0}) \subset O_k \subset B^{\mathbb{R}^d}_{3\delta}(0)$. Let $A= 1-\Omega \epsilon^{2\min(\beta-1,1)}$ and $B=\frac{8 \delta +2 \epsilon^\beta}{\tau_{\iota(M)}}$, where $\Omega$ is a constant depending on $d$, $D$, $P_m$, $C^2$ norm of $P$, the second fundamental form of $\iota(M)$ and its derivative and the Ricci curvature of $M$. Then $$R_k \geq (\sqrt{A^2-A^2B}- \sqrt{B-A^2B}) (3\delta - \epsilon^\beta).$$

\end{theorem}
We discuss some of the conditions in the above theorem.

\textbf{Discussion of the magnitude of $\sigma$.} {In the above theorem, we require the condition $$\sigma \leq \min \big\{\frac{1}{\sqrt{-4(d+5)\log \epsilon}}, \frac{1}{\sqrt{12\log(2n)}} \big\} \epsilon^\beta,$$ which allows for any $\beta>1$. This relation implies that $\sigma$ is not small; in fact, it can be close to the order $\epsilon$ up to a square root logarithmic factor. The reason why $\sigma$ is not allowed to be too large compared to $\epsilon$ is due to the use of the $0-1$ kernel in the construction of the local covariance matrix. A large $\sigma$ can result in too few points within the  $\epsilon$ neighborhood of $y_k$, causing the local covariance matrix at $y_k$ to have a rank less than $d$. In such a case, the map $\mathcal{P}_{y_k}(y)$ cannot be constructed. One possible improvement to the condition on $\sigma$ would be to apply a different kernel, such as a squared exponential kernel. This would require analyzing the spectral behavior of the local covariance matrix with a more general kernel, constructed from noisy samples around a manifold, analogous to the result in Theorem \ref{local PCA  variance analysis}. Moreover, the relationship between $\sigma$ and $\epsilon$, along with condition $\epsilon^\beta <\frac{\tau_{\iota(M)}}{16}$, ensures that $y_k$ is close to $\iota(x_k)$ relative to the reach, with high probability. We illustrate in Section \ref{numerical example curve} of the Supplementary Material that we may not reconstruct $\iota(M)$ accurately if $y_k$ is too far from $\iota(x_k)$ compared to the reach. }

\textbf{The relation between $\sigma$ and the shape of the set $O_k$.}  Depending on $\sigma$, the last part of the theorem describes how close to round $O_k$ is. Consider the case of no noise. Equivalently, we can take $\beta \rightarrow \infty$. By the discussion after Theorem \ref{local PCA  spectral behavior}, $\epsilon$ is independent of $\beta$ if $\beta \geq 2$.  By applying an approximation for small $\epsilon$ and $\delta$, $R_k \geq 3(1-\frac{4 \delta }{\tau_{\iota(M)}}-\Omega'\epsilon) \delta$  for some constant $\Omega'$ depending on $\tau_{\iota(M)}$ and $\Omega$. Hence, $O_k$ is close to round in the sense that $O_k$ is contained in the round ball of radius $3\delta$ and contains a round ball of radius close to $3\delta$. In general, when $\sigma$ is small,  we can choose $\beta$ to be large so that $R_k$ is closer to $3\delta$.  

The implications of Theorem \ref{Construction of a chart main theorem} are explored in the following two subsections. In Section \ref{Setup of the regression functions}, under the proposed relations between $\epsilon$, $\delta$, $\sigma$, and $n$ in the theorem, we connect the manifold reconstruction problem to a series of local regressions. Section \ref{Discussion of the main algorithm} addresses the impact of Theorem \ref{Construction of a chart main theorem} on the evolution of the geometry of $O_k$ throughout the iterations of Algorithm \ref{MrGap1}. We demonstrate that as the domain becomes more regular through iterations, interpolation improves.

\subsection{Setup of the regression functions}\label{Setup of the regression functions}
Suppose $\epsilon$, $\delta$, $\sigma$ and $n$ satisfy the conditions in Theorem \ref{Construction of a chart main theorem}. Then, all the following statements hold with probability greater than $1-\frac{1}{n^2}$.  

Statement (2) of Theorem \ref{Construction of a chart main theorem} shows that for any $y_k$, $\mathcal{P}_{y_k}(y)$ is a diffeomorphism on $B^{\mathbb{R}^D}_{3\delta}(y_k) \cap \iota(M)$. Then, the inverse of $\mathcal{P}_{y_k}(y)$ denoted as $\Phi_k: O_k \rightarrow \mathbb{R}^D$ with $O_k \subset \mathbb{R}^d$ is a chart of $\iota(M)$. As illustrated in Figure \ref{illustration 2}, since $\mathcal{P}_{y_k}(y)$ is constructed from a projection and is a diffeomorphism, 
\begin{align}\label{chart Gk}
\Phi_k(u)=y_k+U_{n,\epsilon}(y_k){\scriptsize\begin{bmatrix}
u \\
F_k(u) \\
\end{bmatrix}},
\end{align}
where $F_k(u): O_k \rightarrow \mathbb{R}^{D-d}$ is a smooth function. Since $\iota(M)$ is unknown, $F_k$ and the domain $O_k$ are unknown. Moreover, statement (3) in Theorem \ref{Construction of a chart main theorem} confirms $0 \in O_k$ which is not straightforward as $y_k$ is not necessarily on $\iota(M)$. Thus, \eqref{regression PYK 0} implies that $\mathcal{P}_{y_k}(y_k)=0$, suggesting that $\Phi_k(0)$ serves as a denoised output for $y_k$. 

The reconstruction of $B^{\mathbb{R}^D}_{3\delta}(y_k) \cap \iota(M)$ is based on the prediction of $F_k(u)$ for $u \in O_k$. We find $B^{\mathbb{R}^D}_{\delta}(y_k) \cap \{y_1, \cdots, y_n\}=\{y_{k,1}, \cdots, y_{k, N_k}\}$ with $ y_{k, i}=\iota(x_{k,i})+\eta_{k,i}$.  Statement (1) in Theorem \ref{Construction of a chart main theorem} justifies that the clean data $\{\iota(x_{k,i})\}_{i=1}^{N_k}$ are in the domain $B^{\mathbb{R}^D}_{3\delta}(y_k) \cap \iota(M)$ on which $\mathcal{P}_{y_k}$ is a diffeomorphism. By \eqref{regression PYK 0}, \eqref{regression PbotYK 0} and \eqref{chart Gk}, the following equations hold:
\begin{align*}
& \mathcal{P}_{y_k}\{\iota(x_{k,i})\}=\mathcal{P}_{y_k}\{\Phi_k(u_{k,i})\}=u_{k,i} \in O_k, \\
& \mathcal{P}^\bot_{y_k}\{\iota(x_{k,i})\}=\mathcal{P}^\bot_{y_k}\{\Phi_k(u_{k,i})\}=F_k(u_{k,i}).
\end{align*}
Therefore, we apply $\{y_{k,1}, \ldots, y_{k, N_k}\}$ to construct the predictors and responses for $F_k$. Specifically, let $\eta_{k,i,1}=J^\top U_{n,\epsilon}(y_k)^\top \eta_{k,i} \in \mathbb{R}^d$ and $\eta_{k,i,2}=\bar{J}^\top U_{n,\epsilon}(y_k)^\top \eta_{k,i}  \in \mathbb{R}^{D-d}$.  By \eqref{regression PYK 0} and \eqref{regression PbotYK 0}
\begin{align}
& \mathcal{P}_{y_k}(y_{k,i})=\mathcal{P}_{y_k}\{\iota(x_{k,i})+\eta_{k,i}\} =\mathcal{P}_{y_k}\{\iota(x_{k,i})\}+ J^\top U_{n,\epsilon}(y_k)^\top \eta_{k,i} =u_{k,i}+ \eta_{k,i,1}, \label{decomposition of noise 1}\\  
& \mathcal{P}^\bot_{y_k}(y_{k,i})=\mathcal{P}^\bot_{y_k}\{ \iota(x_{k,i})+\eta_{k,i}\}=\mathcal{P}^\bot_{y_k}\{\iota(x_{k,i})\}+ \bar{J}^\top U_{n,\epsilon}(y_k)^\top \eta_{k,i} =F_k(u_{k,i})+\eta_{k,i,2}. \label{decomposition of noise 2}
\end{align}
Since $U_{n,\epsilon}(y_k)^\top \eta_{k,i}  \sim \mathcal{N}(0, \sigma^2 I_{D \times D})$,
from properties of the normal distribution, 
$\eta_{k,i,1} \sim \mathcal{N}(0, \sigma^2 I_{d \times d})$ and $\eta_{k,i,2} \sim \mathcal{N}(0, \sigma^2 I_{(D-d) \times (D-d)})$.

We summarize the above observations through the errors-in-variables regression model.  Suppose $F=(f_1,\ldots,f_{q}) : O \subset \mathbb{R}^d \rightarrow \mathbb{R}^{q}$, with $q=D-d$, is an unknown regression function, where $O$ is an unknown open subset in  $\mathbb{R}^d$. We observe the labeled training data $(w_i, z_i)_{i=1}^N$, where
\begin{eqnarray}
w_i=u_i+\eta_{i,1},\quad 
z_i=F(u_i)+\eta_{i,2},\quad 
\eta_{i,1} \sim \mathcal{N}(0, \sigma^2 I_{d \times d}), \quad 
\eta_{i,2} \sim \mathcal{N}(0, \sigma^2 I_{q \times q}),
\label{errors in variables model 1}
\end{eqnarray}
and $\{u_i\}_{i=1}^N \subset O$. We apply Gaussian process regression to predict $F(u)$ for $u \in O$.

\subsection{Discussion of interpolation}\label{Discussion of the main algorithm}
For any $i>1$, denote the denoised outputs of the $(i-1)$th iteration of Steps 1-3 of Algorithm \ref{MrGap1} as $\{y^{(i-1)}_1 , \cdots, y^{(i-1)}_n\}$. 
We assume 
\begin{align}\label{assumption convergence}
y^{(i-1)}_k=\iota(x^{(i-1)}_k)+\eta^{(i-1)}_k\ \mbox{for $x^{(i-1)}_k \in M$ and $\eta^{(i-1)}_k \sim \mathcal{N}(0, {\sigma^{(i-1)}}^2 I_{D \times D})$},
\end{align}
where $\sigma^{(i-1)}$ decreases as $i$ increases. In the $i$th iteration, for each $y^{(i-1)}_k$, there is a chart $\Phi^{(i-1)}_k$ and regression function $F^{(i-1)}_k$ associated with unknown domain $O^{(i-1)}_k$. 
Suppose Algorithm \ref{MrGap1} stops at the $\texttt{I}$th round.  We use the outputs $\{y^{(\texttt{I}-1)}_1 , \cdots, y^{(\texttt{I}-1)}_n\}$ as inputs to interpolate the points on $\iota(M)$. In the $\texttt{I}$th iteration, if we generate $K$ points in each domain $O^{(\texttt{I}-1)}_k$, then we can interpolate $K$ points on $\iota(M)$ through reconstruction of $F^{(\texttt{I}-1)}_k$. Since $O^{(\texttt{I}-1)}_k$ is unknown, we outline a method to generate samples in $O^{(\texttt{I}-1)}_k$ with an explanation by the results in Theorem \ref{Construction of a chart main theorem}.

For each $y^{(\texttt{I}-1)}_k$, we construct $C^{(\texttt{I}-1)}_{n,\epsilon}(y^{(\texttt{I}-1)}_k)$ through $\{y^{(\texttt{I}-1)}_1 , \cdots, y^{(\texttt{I}-1)}_n\}$ as in \eqref{Gaussian local covariance matrix}.  We find an orthonormal eigenvector matrix $U^{(\texttt{I}-1)}_{n,\epsilon}(y^{(i-1)}_k)$ of $C^{(\texttt{I}-1)}_{n,\epsilon}(y^{(\texttt{I}-1)}_k)$ to construct $\mathcal{P}_{y^{(\texttt{I}-1)}_k}(y)$ and $\mathcal{P}^\bot_{y^{(\texttt{I}-1)}_k}(y)$. Suppose $\{y^{(\texttt{I}-1)}_j\}_{j=1}^n \cap B^{\mathbb{R}^D}_{\delta}(y^{(\texttt{I}-1)}_k)=\{y^{(\texttt{I}-1)}_{k,1} , \cdots, y^{(\texttt{I}-1)}_{k,N^{(\texttt{I}-1)}_k}\}$. Let $w^{(\texttt{I}-1)}_{k,j}=\mathcal{P}_{y^{(\texttt{I}-1)}_k}(y^{(\texttt{I}-1)}_{k,j})$ and $z^{(\texttt{I}-1)}_{k,j}=\mathcal{P}^\bot_{y^{(\texttt{I}-1)}_k}(y^{(\texttt{I}-1)}_{k,j})$. We estimate a small round ball contained in $O^{(\texttt{I}-1)}_k$ using $\{w^{(\texttt{I}-1)}_{k,1}, \cdots, w^{(\texttt{I}-1)}_{k,N^{(\texttt{I}-1)}_k}\}$ and sample $K$ points, $\{\tilde{u}_{k,1}, \cdots, \tilde{u}_{k, K}\}$, uniformly at random from the ball.  The center of the ball is 
$$\mathcal{U}_k=\frac{1}{N^{(\texttt{I}-1)}_k} \sum_{j=1}^{N^{(\texttt{I}-1)}_k} w^{(\texttt{I}-1)}_{k,j}= \frac{1}{N^{(\texttt{I}-1)}_k} \sum_{j=1}^{N^{(\texttt{I}-1)}_k}u^{(\texttt{I}-1)}_{k,j}+\frac{1}{N^{(\texttt{I}-1)}_k} \sum_{j=1}^{N^{(\texttt{I}-1)}_k} \eta^{(\texttt{I}-1)}_{k,j,1} \in \mathbb{R}^d,$$
where $u^{(\texttt{I}-1)}_{k,j} \in O^{(\texttt{I}-1)}_k$ and  $\eta^{(\texttt{I}-1)}_{k,j,1} \sim \mathcal{N}(0, {\sigma^{(\texttt{I}-1)}}^2 I_{d \times d})$, $j=1, \cdots, N^{(\texttt{I}-1)}_k$. We set radius $m_k-s_k$, with 
$m_k$ the mean and $s_k$ the standard deviation of  $\{\|w^{(\texttt{I}-1)}_{k,1}-\mathcal{U}_k\|_{\mathbb{R}^d}, \cdots, \|w^{(\texttt{I}-1)}_{k,N^{(\texttt{I}-1)}_k}-\mathcal{U}_k\|_{\mathbb{R}^d}\}$. 

The method above to construct a small round ball in $O^{(\texttt{I}-1)}_k$ relates to the evolution of the shapes of the sets $\{O_k, O^{(1)}_k, \cdots, O^{(\texttt{I}-1)}_k\}$ through the iteration process. If $O^{(\texttt{I}-1)}_k \subset \mathbb{R}^d$ is close to spherical, our algorithm produces a ball contained in $O^{(\texttt{I}-1)}_k$.
By \eqref{assumption convergence} and statement (3) in Theorem \ref{Construction of a chart main theorem}, all sets $\{O_k, O^{(1)}_k, \cdots, O^{(\texttt{I}-1)}_k\}$ are contained in $B^{\mathbb{R}^d}_{3\delta}(0)$. Since $\sigma^{(i-1)}$ decreases as $i$ increases, compared to $\{O_k, O^{(1)}_k, \cdots, O^{(\texttt{I}-2)}_k\}$, $O^{(\texttt{I}-1)}_k$ is close to spherical, as it contains a larger radius ball closer to $3\delta$. In Figure \ref{Comparison Ok O'k},  we illustrate sets $O_k$ and $O^{(\texttt{I}-1)}_k$.   Hence, we perform the interpolation after sufficient iterations of Steps 1-3 in Algorithm \ref{MrGap1} to ensure that the generated $K$ samples $\{\tilde{u}_{k,i}\}_{i=1}^K$ are within the domain. Refer to Section \ref{Interpolation phase} of the Supplementary Material for explicit expressions of interpolations constructed through $\{\tilde{u}_{k,i}\}_{i=1}^K$ and estimated parameters.

\begin{figure}[h!]
\centering
{
\includegraphics[width=8cm,height=3.6cm]{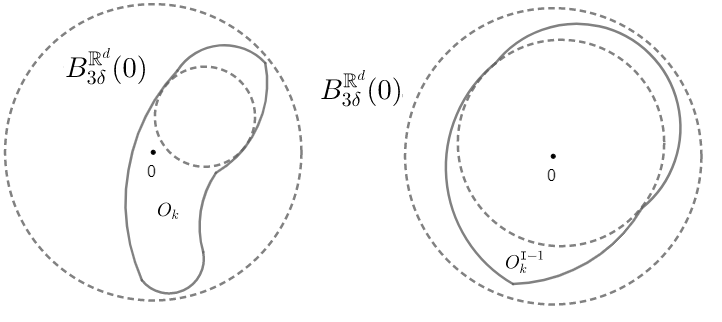}
}
\caption{Both $O_k$ and $O^{(\texttt{I}-1)}_k$ are contained in $B^{\mathbb{R}^d}_{3\delta}(0)$. Left: An illustration of  $O_k \subset \mathbb{R}^d$ where $O_k$ contains a small round ball.  Right: An illustration of  $O^{(\texttt{I}-1)}_k \subset \mathbb{R}^d$ where $O^{(\texttt{I}-1)}_k$ contains a larger round ball of radius close to $3\delta$.  }\label{Comparison Ok O'k}
\end{figure}

\subsection{Performance evaluation}
We evaluate the performance of \texttt{MrGap} in manifold reconstruction, assessing deviation of the interpolated points from the embedded submanifold $\iota(M)$.  For any $y \in \mathbb{R}^D$, the distance from $y$ to $S \subset \mathbb{R}^D$ is $\texttt{dist}(y,S)=\inf_{y' \in S} \|y-y'\|_{\mathbb{R}^D}$.
Given $\mathcal{Y}=\{y_1, \cdots, y_n\} \subset \mathbb{R}^D$,  the geometric root mean square error from $\mathcal{Y}$ to $S$ is $G(\mathcal{Y}, S)=\sqrt{\frac{1}{n} \sum_{i=1}^n \texttt{dist}(y_i,S)^2},$
with $G(\mathcal{Y}, S)=0$ if and only if $\mathcal{Y}$ is in the closure of $S$. When $S$ is compact, $G(\mathcal{Y}, S)=0$ if and only if $\mathcal{Y} \subset S$. In the following theorem, we show that when $\mathcal{Y}_{true}$ is sampled in $\iota(M)$ based on a density function with a positive lower bound, as the sample size of $\mathcal{Y}_{true}$ goes to infinity, $G(\mathcal{Y}, \mathcal{Y}_{true}) \to G\{\mathcal{Y}, \iota(M)\}$ almost surely. We also provide the convergence rate. The proof is given in Section \ref{Appendix: proof of GRMSE estimation}  of the Supplementary Material.

\begin{theorem}\label{GRMSE estimation}
Suppose $\mathcal{Y}=\{y_1, \cdots, y_n\} \subset \mathbb{R}^D$. $\{x_1, \cdots, x_m\}$ are samples on $M$ based on a $C^{1}$ p.d.f $\mathrm{q}$ on $M$ such that $\mathrm{q}>\mathrm{q}_{\min}>0$. Let $\mathcal{Y}_{true}=\{\iota(x_1), \cdots, \iota(x_m)\}$. Suppose $C$ is a constant depending on $d$, $D$, $C^{1}$ norm of $\mathrm{q}$, the curvature of $M$ and the second fundamental form of $\iota(M)$, and $r = \big[\max\{ \frac{2}{\mathrm{q}_{\min}}, (\frac{2C}{\mathrm{q}_{\min}})^2\}\frac{\log m}{m}\big]^{1/d}$, then with probability greater than $1-\frac{1}{m^2}$,

(1) $0 \leq \texttt{dist}(y_i, \mathcal{Y}_{true}) -  \texttt{dist}\{y_i, \iota(M)\}  \leq r$, 

(2) $G(\mathcal{Y}, \mathcal{Y}_{true}) ^2-2r G\{\mathcal{Y}, \iota(M)\}-r^2 \leq G\{\mathcal{Y}, \iota(M)\}^2 \leq G(\mathcal{Y}, \mathcal{Y}_{true})^2$.
\end{theorem}
Convergence in the above theorem follows from the fact that $r \rightarrow 0$ as the sample size $m \rightarrow 0$. Moreover, statement (2) implies that if $m$ is large enough so that $r<\frac{1}{10}G\{\mathcal{Y}, \iota(M)\}$, then
\begin{align}\label{error GRMSE on r}
G\{\mathcal{Y}, \iota(M)\} \leq  G(\mathcal{Y}, \mathcal{Y}_{true}) \leq G\{\mathcal{Y}, \iota(M)\} +1.3r.
\end{align}
Computing $\texttt{dist}\{y,\iota(M)\}$ is generally not computationally feasible. Therefore, we use samples $\mathcal{Y}_{true}$ generated in $\iota(M)$ to estimate $G\{\mathcal{Y}, \iota(M)\}$. The limits given in \eqref{error GRMSE on r} offer guidance on the sample sizes required for an accurate estimation. Furthermore, $\mathrm{q}_{\min} \leq \frac{1}{Vol(M)}$, with equality holding when $\mathrm{q}$ is uniform. Thus, for a fixed sample size $m$, $r$ is minimized when $\mathrm{q}$ is uniform. This suggests that the most efficient way to estimate
$G\{\mathcal{Y}, \iota(M)\}$ is to sample $ \mathcal{Y}_{true}$ uniformly.

\section{Numerical Simulation}\label{Numerics}
\subsection{Cassini Oval  in $\mathbb{R}^3$}\label{Cassini Oval main article}
Let $r(\theta)=\big[\cos(2\theta)+\big\{\cos(2\theta)^2+0.2\big\}^{1/2}\big]^{1/2}$. A Cassini Oval $\iota(M)$ in $\mathbb{R}^3$ is parametrized by 
\begin{align}\label{para co}
X(\theta)=r(\theta)\cos(\theta), \quad\quad  Y(\theta)= r(\theta)\sin(\theta),  \quad\quad Z(\theta)=0.3\sin(\theta+\pi), 
\end{align}
where $\theta \in [0, 2\pi)$. We obtain non-uniform samples $\{X(\theta_i), Y(\theta_i), Z(\theta_i)\}_{i=1}^{102}$ on $\iota(M)$ from samples $\{\theta_i\}_{i=1}^{102}$ on $[0, 2\pi)$. Suppose $\eta_{i} \sim \mathcal{N}(0, {\sigma}^2 I_{3 \times 3}),$ with $\sigma=0.04$, $y_i=\{X(\theta_i), Y(\theta_i), Z(\theta_i)\}^\top+\eta_i$, $i=1, \cdots, 102$, and $\mathcal{Y}=\{y_i\}_{i=1}^{102}$. 
We uniformly sample $\{\phi_i\}_{i=1}^{10^5}$ on $ [0, 2\pi)$ letting $\mathcal{Y}_{true}=\{X(\phi_i), Y(\phi_i), Z(\phi_i)\}_{i=1}^{10^5}$.  For any sample points $\mathcal{S}$,  $G\{\mathcal{S}, \iota(M)\}$ is approximated by $G(\mathcal{S},\mathcal{Y}_{true})$. Without denoising, we obtain $G(\mathcal{Y},\mathcal{Y}_{true})=0.059$.  

Using the methods in Sections \ref{determine the dimension} and \ref{bandwidth selection example} of the Supplementary Material, we estimate dimension $d=1$ and the scale parameters $\epsilon=0.3$ and $\delta=0.6$  in \texttt{MrGap}. We iterate Steps 1-3 of Algorithm \ref{MrGap1} twice. The estimated covariance parameters  in the first round are  $A^{(0)}=0.014$, $\rho^{(0)}=0.2$ and $\sigma^{(0)}=\sqrt{0.002}$, while the values in the last round are  $A^{(1)}=0.048$, $\rho^{(1)}=0.3$ and $\sigma^{(1)}=\sqrt{2 \times 10^{-5}}$.  The denoised outputs are $\mathcal{X}_1=\{\hat{y}_i\}_{i=1}^{102}$, with $G(\mathcal{X}_1,\mathcal{Y}_{true})=0.0224$. In the interpolation phase, we construct $102$ charts and interpolate $K=20$ points in each chart. 
The outputs are $\mathcal{X}_2=\{\tilde{y}_i\}_{i=1}^{2040}$ with  $G(\mathcal{X}_2,\mathcal{Y}_{true})=0.0216$.   A comparison of $\mathcal{Y}$, $\mathcal{X}_1$, and $\mathcal{X}_2$ with $\mathcal{Y}_{true}$ is provided in Figure \ref{cassini}. 

We compare the performance of \texttt{MrGap} with principal graphs \citep{mao2016principal}, which are a generalization of principal curves \citep{hastie1989principal}. Principal graphs involve a number of nearest neighbors $nn$ and tuning parameters $\sigma_p$, $\gamma_p$ and $\lambda_p$.  We choose $nn=5$, $\sigma_p=0.01$, $\gamma_p=0.5$ and $\lambda_p=1$ based on suggestions in their codes. We iterate their algorithm for $20$ times to acquire $102$ denoised data points $\mathcal{X}_p$, obtaining $G(\mathcal{X}_p,\mathcal{Y}_{true})=0.0322$. The principal graphs algorithm also generates a curve which fits $\mathcal{Y}$.  One of the primary challenges in manifold reconstruction lies in addressing the non-smoothness problem during interpolations. In the principal graph method, non-smoothness is evident, while \texttt{MrGap}  effectively eliminates this issue. Refer to Figure \ref{cassini}.
\begin{figure}
\begin{subfigure}
\centering
\includegraphics[width=15cm,height=3cm]{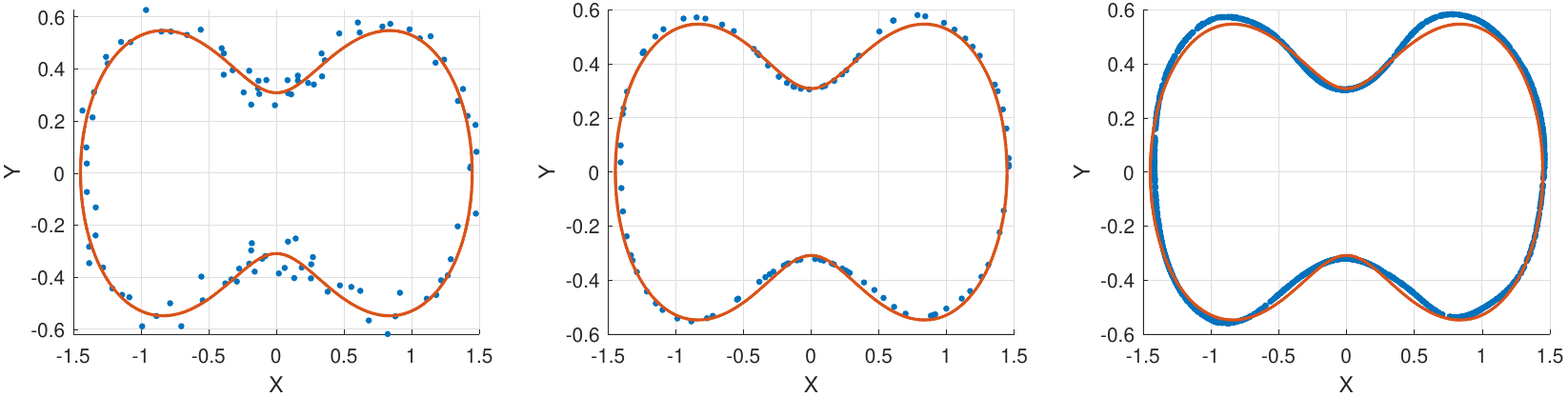}
\end{subfigure}
\begin{subfigure}
\centering
\includegraphics[width=15cm,height=3cm]{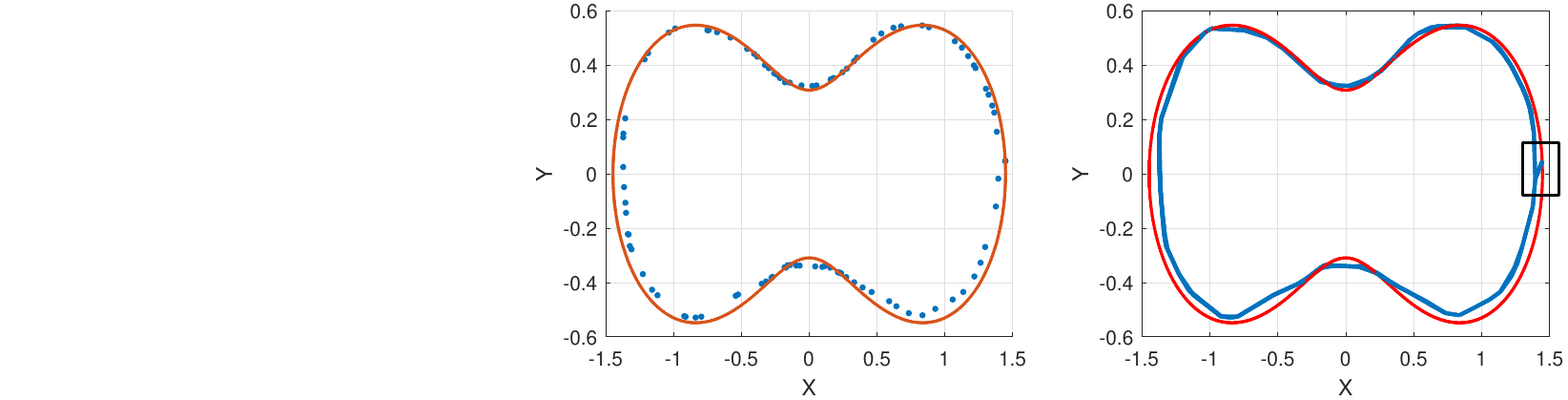}
\end{subfigure}
\caption{Red curves represent $\mathcal{Y}_{true}$ consisting of $10^5$ points on $\iota(M)$. Top row: left panel shows data $\mathcal{Y}$ (blue) and $\mathcal{Y}_{true}$ ($G(\mathcal{Y},\mathcal{Y}_{true})=0.059$). Middle panel shows denoised samples $\mathcal{X}_1$  (blue) by \texttt{MrGap}  and $\mathcal{Y}_{true}$ ($G(\mathcal{X}_1,\mathcal{Y}_{true})=0.0224$).  Right panel shows interpolations $\mathcal{X}_2$  (blue) by \texttt{MrGap} and $\mathcal{Y}_{true}$ ($G(\mathcal{X}_2,\mathcal{Y}_{true})=0.0216$). Bottom row: middle panel shows denoised samples $\mathcal{X}_p$ by principal graphs (blue) and $\mathcal{Y}_{true}$ ($G(\mathcal{X}_p,\mathcal{Y}_{true})=0.0322$). Right panel shows principal graphs fit (blue) and $\mathcal{Y}_{true}$; fit is not smooth at the indicated position. }\label{cassini}
\end{figure}
\subsection{Torus in $\mathbb{R}^3$}
We consider a torus $\iota(M)$ in $\mathbb{R}^3$ parametrized by $u, v \in [0, 2\pi)$,
$$X(u,v) =\{2+0.8\cos(u)\}\cos(v), \hspace{3mm} Y(u,v) =\{2+0.8\cos(u)\}\sin(v),\hspace{3mm}
Z(u,v)=0.8\sin(u).$$
%
We randomly sample $558$ points $\{\iota(x_i)\}_{i=1}^{558}$ from the uniform density on $\iota(M)$.  Let $\eta_{i} \sim \mathcal{N}(0, {\sigma}^2 I_{3 \times 3}),$ with $\sigma=0.12$, $y_i=\iota(x_i)+\eta_i$ for $i=1, \ldots, 558$, and $\mathcal{Y}=\{y_i\}_{i=1}^{558}$. We randomly sample $3.2 \times 10^5$ points from the uniform density on $ \iota(M)$ to form $\mathcal{Y}_{true}$.  
We calculate the error from $\mathcal{Y}$ to $\mathcal{Y}_{true}$ which is $G(\mathcal{Y},\mathcal{Y}_{true})=0.1238$. 
We plot $\mathcal{Y}$ and $\mathcal{Y}_{true}$ in Figure \ref{toruscomparison}.  

Using the methods in Sections \ref{determine the dimension} and \ref{bandwidth selection example} of the Supplementary Material, we estimate dimension $d=2$ and the scale parameters $\epsilon=0.8$ and $\delta=1$  . We iterate Steps 1-3 of Algorithm \ref{MrGap1} twice. The estimated covariance parameters are $A^{(0)}=0.06$, $\rho^{(0)}=0.2$ and $\sigma^{(0)}=\sqrt{0.03}$ in the first round, and $A^{(1)}=0.2$, $\rho^{(1)}=1.1$ and $\sigma^{(1)}=\sqrt{0.007}$ in the last round.  The denoised outputs are $\mathcal{X}_1=\{\hat{y}_i\}_{i=1}^{558}$ with $G(\mathcal{X}_1,\mathcal{Y}_{true})=0.0544$.  We choose $K=20$ in the interpolation phase, yielding the output $\mathcal{X}_2=\{\tilde{y}_i\}_{i=1}^{11160}$ with $G(\mathcal{X}_2,\mathcal{Y}_{true})=0.0605$.  We plot $\mathcal{X}_1$, $\mathcal{X}_2$  in Figure \ref{toruscomparison}.  

We compare with the method of \citep{aizenbud2021non}, which assumes data are uniformly distributed in a $\sigma_{a}$ neighborhood of $\iota(M)$ in $\mathbb{R}^D$. We estimate $\sigma_{a}=1.19$ and denoised data $\mathcal{X}_{a}$, which have $G(\mathcal{X}_{a},\mathcal{Y}_{true})=0.2196$. Although they do not provide an algorithm for interpolation, we modify their approach to produce interpolated points by uniformly generating $20$ points in the $\frac{\sigma_{a}}{2}$ neighborhood of each denoised point in $\mathcal{X}_{a}$ to obtain $\mathcal{Y}_{a}$, and then re-apply their algorithm with the same parameters to obtain new denoised data $\mathcal{X}'_{a}$ that have $G(\mathcal{X}_{a}',\mathcal{Y}_{true})=0.2424$. A comparison of $\mathcal{X}_{a}$ and $\mathcal{X}_{a}'$ with $\mathcal{Y}_{true}$ in Figure \ref{toruscomparison} indicates the denoised points and interpolations lying on a thinner torus than the true one.
\begin{figure}[htb!]
\centering
\includegraphics[width=15cm,height=7cm]{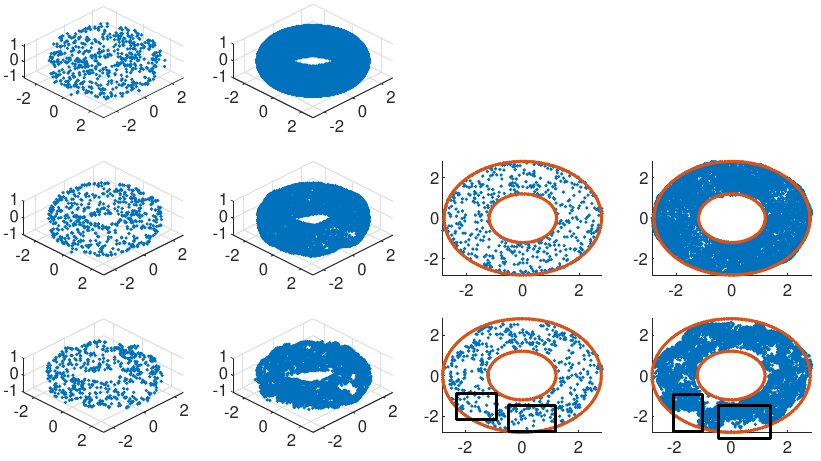}
\caption{Top row: the first panel shows $\mathcal{Y}$ containing $558$ points around torus $\iota(M)$. The second panel shows $\mathcal{Y}_{true}$ containing $3.2 \times 10^5$ points on $\iota(M)$ with 
$G(\mathcal{Y},\mathcal{Y}_{true})=0.1238$. Middle row: the first panel shows denoised outputs $\mathcal{X}_1$ from \texttt{MrGap} with $G(\mathcal{X}_1,\mathcal{Y}_{true})=0.0544$. The second panel shows 11260 interpolated fits $\mathcal{X}_2$ from \texttt{MrGap} with $G(\mathcal{X}_2,\mathcal{Y}_{true})=0.0605$. In the last two panels, blue points are XY projections of $\mathcal{X}_1$ and $\mathcal{X}_2$. Red circles are the intersection of the XY plane and the torus. 
Bottom row: first panel shows denoised outputs $\mathcal{X}_{a}$ via \citep{aizenbud2021non} with $G(\mathcal{X}_{a},\mathcal{Y}_{true})=0.2196$. The second panel shows 11260 interpolated fits $\mathcal{X}'_{a}$ by their method with $G(\mathcal{X}'_{a},\mathcal{Y}_{true})=0.2424$. In the last two panels, blue points are XY projections of $\mathcal{X}_{a}$ and $\mathcal{X}'_{a}$. Since the clean data are uniformly distributed on $\iota(M)$, if the reconstructed torus matches the size of $\iota(M)$,  the XY projections of denoised outputs and interpolation should lie within the annulus formed by the intersection of the XY plane and $\iota(M)$, with no significant gaps. The holes in the XY projections of $\mathcal{X}_{a}$ and $\mathcal{X}'_{a}$ near the annulus boundary shown in the figure suggest that the reconstructed torus is thinner than $\iota(M)$. }\label{toruscomparison}
\end{figure}

\subsection{Additional simulations}
 In Section \ref{boundary and noise}, we illustrate the performance of \texttt{MrGap} on a real projective space, a $3$ dimensional closed manifold embedded in $\mathbb{R}^{10}$. This section also shows the performances of \texttt{MrGap} under different assumptions about the manifold and data, specifically when the manifold has a boundary or the noise is not Gaussian.

\section{Modeling of bird vocalization data}\label{bird data}

Recall the bird vocalization data introduced in Section \ref{introduction}.  Data consist of $83$ matrices of dimension $75 \times 197$ corresponding to a type of call of Anthus trivialis; these matrices are submatrices of the spectrograms. Since we want to remove the impact of volume of the vocalizations, we normalize the data after vectorizing to obtain the data in $x_i \in \mathbb{R}^{14775}$, for $i=1,\ldots,83$. Following standard practice for extremely high-dimensional data, we project noisy data onto the first $n-1=82$ principal components to obtain the data $\mathcal{Y}=\{y_i\}_{i=1}^{83} \subset \mathbb{R}^{82}$.

{These are representative of data that are routinely collected in many biodiversity studies. Standard practice uses the spectrogram data to infer which species is vocalizing in each audio segment based on deep neural network classifiers \citep{OvaskainenDunson2025}. Our collaborators are also interested in studying variation in vocalizations across individuals of the same species, both randomly and in relationship to spatial location and environmental covariates. Our proposed local Gaussian process regression framework provides a natural starting point for such analyses. By interpolating between different bird spectrograms of the same species we obtain insight into variation across individuals. By using local GP regressions to characterize a low dimensional bird vocalization manifold, we obtain a statistically efficient modeling approach that can naturally include sample covariates and/or spatial location. Here, we avoid the complexities of including such information and focus on showing that \texttt{MrGap} performs well in capturing the bird vocalization manifold from a small number of noisy samples.}

To increase the difficulty, we add Gaussian noise $\eta_{i} \sim \mathcal{N}(0, {\sigma}^2 I_{14775 \times 14775}),$ with $\sigma=0.005$, to $x_i$. We evaluated the eigenvalues of the local covariance matrix constructed through $\mathcal{Y}$ as in Section \ref{determine the dimension} of the Supplementary Material, producing an estimate of $d=1$. Using the method in Section \ref{bandwidth selection example}, we estimate the scale parameters $\epsilon=1.27$ and $\delta=1.27$. The estimated covariance parameters in the first round of Algorithm \ref{MrGap1} are $A^{(0)}=0.005$, $\rho^{(0)}=0.2$ and $\sigma^{(0)}=0.009$. We apply steps 1-3 of Algorithm \ref{MrGap1} for only one round, as subsequent iterations do not significantly affect the estimate of $\sigma$. We choose $K=10$ in the interpolation phase and generate $830$ samples $\{\tilde{y}_{i}\}_{i=1}^{830}$ in the underlying manifold, representing new synthetic Anthus trivialis vocalizations in a space of 82-dimensional principal components. We map these samples back to the $\mathbb{R}^{14775}$ space of the vectorized submatrices of the spectrogram, reconstructing the spectrograms from these. 

In Figure \ref{spectrograms2}, we present $10$ such synthetic bird vocalization spectrograms generated by our methodology.
To make the problem more challenging, we added noise to the original data. Nonetheless, we obtain excellent performance, with the spectrograms less noisy compared to the training data, successfully incorporating audio information from nearby denoised spectrograms that we are interpolating around. Our denoised spectrograms closely resemble those obtained in applying \texttt{MrGap} to the original data without noise added; some examples are shown in
Figure \ref{spectrograms1}. 
 \begin{figure}[h!]
\centering
\includegraphics[width=14cm,height=5.5cm]{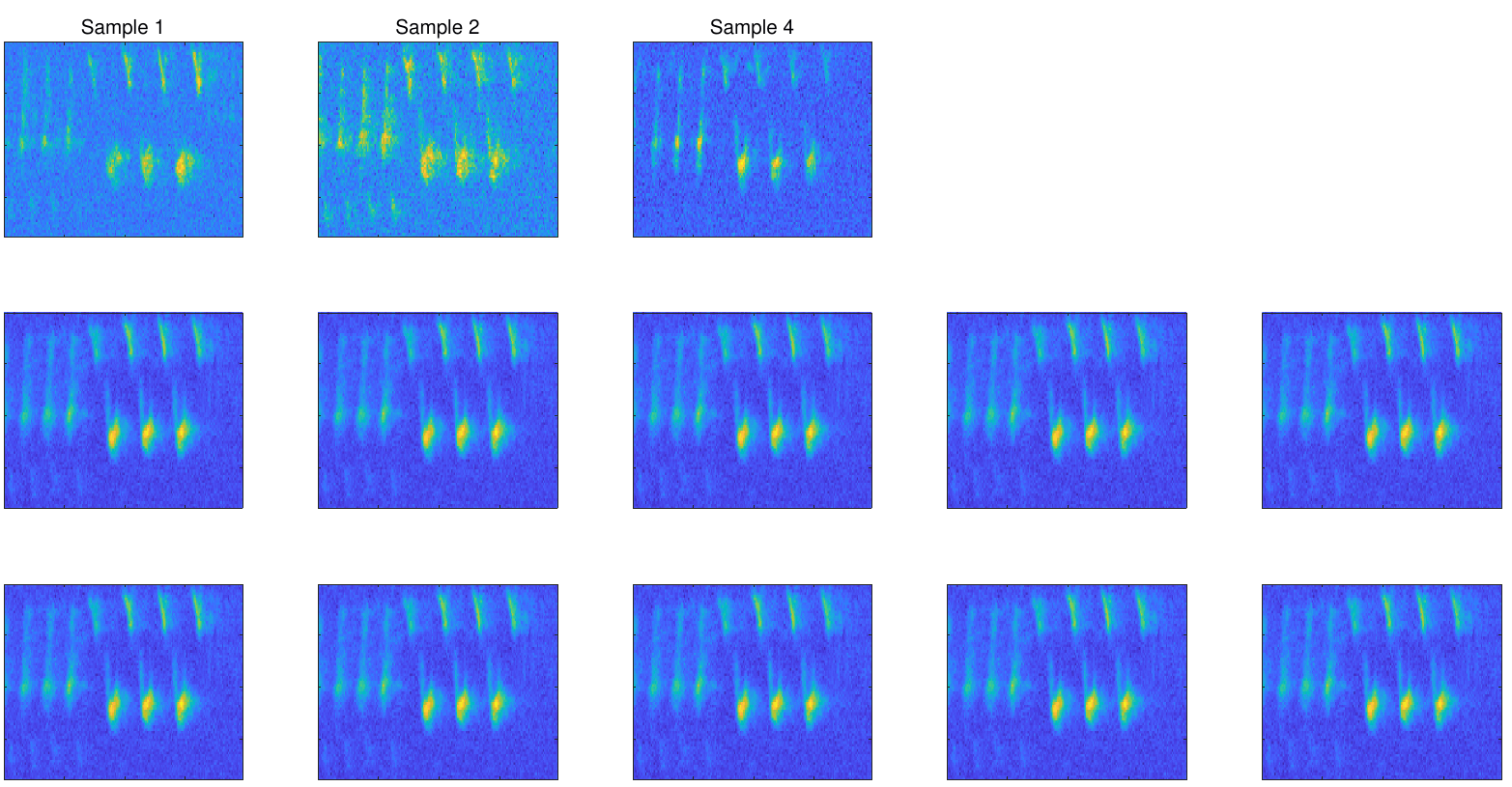}
\caption{Top row: The data spectrograms with noise added. They correspond to Sample 1, 2 and 4, respectively, in Figure \ref{spectrograms1}.  Second and third rows: $10$ synthetic bird vocalizations generated close to Sample 1 using \texttt{MrGap}.}
\label{spectrograms2}
\end{figure}

We apply the method in \citep{aizenbud2021non} to the noisy data $\mathcal{Y}=\{\textbf{y}_i\}_{i=1}^{83} \subset \mathbb{R}^{82}$,  following the process described in Section \ref{Cassini Oval main article}, which estimates $\sigma_{a}=1.427$ and produces $83$ denoised samples. We then generate $10$  samples uniformly in the $\frac{\sigma_{a}}{2}$ neighborhood of each denoised point. We obtain $830$ clean new samples $\{\tilde{\textbf{y}}^{as}_i\}_{i=1}^{830}$ on the 1 dimensional manifold by reapplying the algorithm with the same parameters. We map these samples back to the $\mathbb{R}^{14775}$ space of the vectorized submatrices of the spectrogram, reconstructing $830$ new spectrograms from these.  In Figure \ref{spectrograms3}, we present $10$ generated spectrograms around Sample 1. Comparing these to the given noisy Sample 1  and the nearby Samples 2 and 4 in Figure \ref{spectrograms2} (corresponding to clean Samples 1, 2 and 4 in Figure \ref{spectrograms1}),  it is evident that important audio information is missing in the indicated regions. 

\begin{figure}[h!]
\centering
\includegraphics[width=14cm,height=3.8cm]{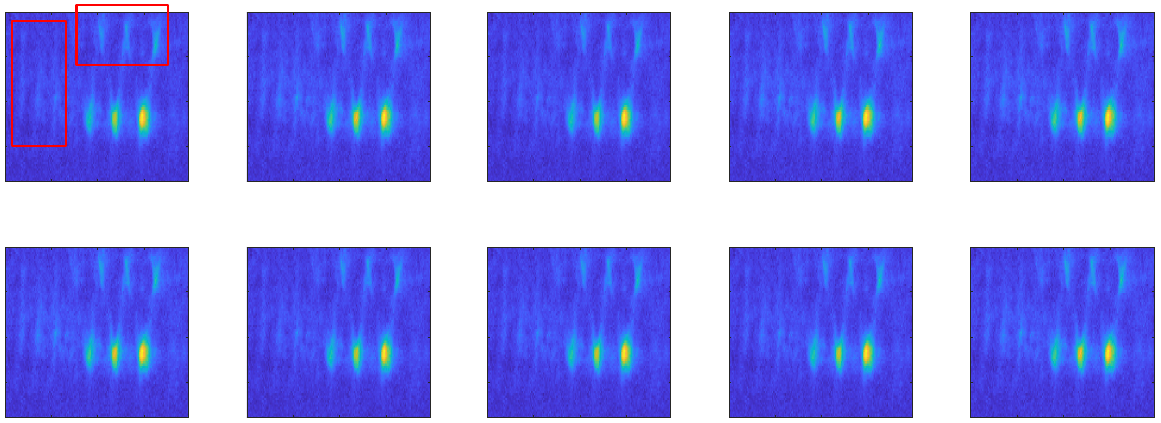}
\caption{First and second rows: $10$ generated samples around Sample 1 by the method in \citep{aizenbud2021non}. Audio information is missing at the indicated regions in the spectrograms. }\label{spectrograms3}
\end{figure}

\section{Discussion}

{This article proposed a local regression modeling framework for inferring manifold structure underlying noisy data, leveraging on novel theory on properties of local covariance matrices. While adding to the literature on manifold reconstruction, the fact that we obtain a statistical modeling framework is highly appealing in providing a starting point for many interesting extensions. For example, motivated by ecology applications we are currently pursuing extensions that include sample covariates and spatio-temporal dependence within the local Gaussian process regressions. Such models can allow intricate features of bird vocalizations to change depending on context.}

There are many other interesting future directions. First, we can consider broader classes of covariance functions in the local Gaussian process regressions, including locally adaptive approaches accommodating complex manifolds that have varying smoothness. For example, we may consider a Mat{\'e}rn covariance instead of the squared exponential and can allow the local bandwidth parameters in the covariance functions to vary. Potentially, this can be done in a spatially smooth manner so that local neighborhoods that are close together in the ambient space have bandwidth parameters that tend to be closer together {\em a priori}.  Another key direction is to obtain a theoretical justification in terms of uncertainty quantification. In this article, we have focused on using the local  regressions for point estimation, but the posterior distributions clearly contain substantial information about uncertainty. 

\section{Supplementary Material and Codes}
The Supplementary Material include the proofs of the theorems and additional examples. The MATLAB implementation of the method and the data sets in Sections 4 and 5 are available on \url{https://github.com/wunan3/Manifold-reconstruction-via-Gaussian-processes}.

\section{Acknowledgement}
The authors acknowledge the support of the European Research Council (ERC) under the European Union’s Horizon 2020 research and innovation program (grant agreement No 856506) and the National Science Foundation (IIS-2426762).

\appendix
\section{Review of Gaussian process regression }\label{review GPR}
Suppose $F:\mathbb{R}^d \rightarrow \mathbb{R}^q$ is an unknown regression function with $F=(f_1, \cdots, f_q)^\top$. Letting ${z}_i\in \mathbb{R}^q$ denote the response vector and ${u}_i \in \mathbb{R}^d$ denote the predictor vector, for $i=1,\ldots, N$, we have
\begin{equation}\label{standard GP model appendix}
{z}_i=F ({u}_i) +\eta_i,\qquad 
\eta_i \sim \mathcal{N}(0, \sigma^2 I_{q \times q}).
\end{equation}
We assign independent Gaussian process priors to each $f_j$ with mean $0$ and covariance function 
\begin{equation}\label{GP cov appendix}
\mathsf{C}({u}, {u}')=A \exp\Big(-\frac{\|{u}-{u}'\|^2_{\mathbb{R}^d}}{\rho}\Big).
\end{equation}
Denote ${f}'_j \in \mathbb{R}^N$ as the discretization of $f_j$ over $\{{u}_i\}_{i=1}^N$ so that ${f}'_j=\{f_j({u}_1), \ldots f_j({u}_N)\}^\top$ for $j=1,\ldots,q$.  A Gaussian process prior for $f_j$ implies $p({f}'_j|{u}_1,{u}_2, \cdots, {u}_N)=\mathcal{N}(0, \Sigma_1)$, where $\Sigma_1 \in \mathbb{R}^{N \times N}$ is the covariance matrix induced from $\mathsf{C}$, with element $(j,k)$ of $\Sigma_1$ corresponding to $\mathsf{C}({u}_j,{u}_k)$, for $1 \leq j,k \leq N$.  Prior distribution $\mathcal{N}(0, \Sigma_1)$ can be combined with information in the likelihood function under model \eqref{standard GP model appendix} to obtain the posterior distribution, which will be used as a basis for inference.

Suppose that we want to predict $F$ at $\{{u}_i\}_{i=N+1}^{N+m}$. Define ${f}^*_j \in\mathbb{R}^m$ as ${f}^*_j=\{f_j({u}_{N+1}),\ldots, f_j({u}_{N+m})\}^\top$ for $j=1, \ldots, q$.  Denote ${F} \in\mathbb{R}^{m \times q}$ with $j$th column ${f}^*_j$.  Under a Gaussian process prior for $f_j$, the joint distribution of ${f}'_j$ and ${f}^*_j$ is  $p({f}'_j,{f}^*_j)=\mathcal{N}(0, \Sigma),$ 
where 
$
\Sigma= {\scriptsize\begin{bmatrix}
\Sigma_1 & \Sigma_2  \\
\Sigma_3 & \Sigma_4
\end{bmatrix}}\,, \nonumber 
$
with $\Sigma_2 \in \mathbb{R}^{N \times m}$, $\Sigma_3 \in \mathbb{R}^{m \times N}$, and $\Sigma_4 \in \mathbb{R}^{m \times m}$ induced from covariance function $\mathsf{C}$. 

Denote ${Z} \in \mathbb{R}^{N \times q}$ with $i$th row ${z}^\top_i$ for $i=1,\ldots,N$. Denote $\tilde{{z}}_j$ for $j=1,\ldots,q$ as the $j$th column of ${Z}$ consisting of the responses of $f_j$ over $\{{u}_i\}_{i=1}^N$.  Under model \eqref{standard GP model appendix} and a Gaussian process prior, we have 
$p(\tilde{{z}}_j, {f}^*_j)=\mathcal{N}(0, \tilde{\Sigma}), $
where
\begin{align*}
\tilde{\Sigma}=\Sigma+{\scriptsize\begin{bmatrix}
\sigma^2 I_{N \times N} & 0  \\
0 & 0 \\
\end{bmatrix}}= {\scriptsize\begin{bmatrix}
\Sigma_1 +\sigma^2 I_{N \times N} & \Sigma_2  \\
\Sigma_3 & \Sigma_4 \\
\end{bmatrix}}. \nonumber 
\end{align*}
For $j=1, \ldots, q$, the predictive distribution of $ {f}^*_j $ is 
\begin{align}\label{GP posterior distribution 0}
p( {f}^*_j |\tilde{{z}}_j)=\mathcal{N}\{\Sigma_3(\Sigma_1+\sigma^2 I_{N \times N})^{-1}\tilde{{z}}_j, \Sigma_4-\Sigma_3 (\Sigma_1+\sigma^2 I_{N \times N})^{-1}\Sigma_2\}.
\end{align}
Then, prediction for ${F}$ can be expressed using the means of the above distributions:
\begin{align*}
\Sigma_3(\Sigma_1+\sigma^2 I_{N \times N})^{-1}{Z}.
\end{align*}

In the special case $m=1$, so we want to predict $F$ at only one point ${u}_{N+1}$, we have ${F}^\top= F({u}_{N+1}) \in \mathbb{R}^{q}$. By \eqref {GP posterior distribution 0}, the predictive distribution of $ {F}^\top$ is
\begin{align*}
p( {F}^\top |{Z})=\mathcal{N}[\big\{\Sigma_3(\Sigma_1+\sigma^2 I_{N \times N})^{-1}{Z}\big\}^\top, \big\{\Sigma_4-\Sigma_3 (\Sigma_1+\sigma^2 I_{N \times N})^{-1}\Sigma_2\big\} I_{q \times q}].
\end{align*}

For each $\tilde{{z}}_j$, the covariance parameters $A$, $\rho$ and $\sigma$ can be estimated by maximizing the natural log of the marginal likelihood, obtained by marginalizing over the Gaussian process prior, 
\begin{align}
\log p_j(\tilde{{z}}_j| A, \rho, \sigma) = -\tilde{{z}}_j^\top (\Sigma_1+\sigma^2 I_{N \times N})^{-1}\tilde{{z}}_j- \log\{\det(\Sigma_1+\sigma^2 I_{N \times N})\}-\frac{N}{2}\log(2\pi). \nonumber
\end{align}
To estimate the parameters $A$, $\rho$ and $\sigma$, we maximize $\log \big\{\prod_{j=1}^q  p_j(\tilde{{z}}_j| A, \rho, \sigma)\big\}$, 
\begin{align*}
\log p({Z}| A, \rho, \sigma) = & \log \big\{\Pi_{j=1}^q  p_j(\tilde{{z}}_j| A, \rho, \sigma)\big\}= \sum_{j=1}^q\log   p_j(\tilde{{z}}_j| A, \rho, \sigma) \nonumber \\
= &  -\sum_{j=1}^q\tilde{{z}}_j^\top (\Sigma_1+\sigma^2 I_{N \times N})^{-1}\tilde{{z}}_j -q \log\{\det(\Sigma_1+\sigma^2 I_{N \times N})\}-\frac{q N}{2}\log(2\pi) \nonumber \\
= & -\mathrm{tr}\{{Z}^\top (\Sigma_1+\sigma^2 I_{N \times N})^{-1}{Z}\}-q \log\{\det(\Sigma_1+\sigma^2 I_{N \times N})\}-\frac{q N}{2}\log(2\pi).
\end{align*}
This empirical Bayes approach for parameter estimation protects against over-fitting. 

\section{Some basic concepts about smooth embedded manifolds}\label{Basic concepts}
\subsection{Diffeomorphism and chart}
Suppose $U \subset \mathbb{R}^p$, $V \subset \mathbb{R}^q$ where $U$ and $V$ are not necessarily open. Suppose $\Phi: U \rightarrow V$ is continuous. Then $\Phi$ is a smooth map whenever it can be extended locally to a smooth map from an open set in $\mathbb{R}^p$ to $\mathbb{R}^q$. More precisely, we have the following definition.
\begin{definition}
Suppose $U \subset \mathbb{R}^p$, $V \subset \mathbb{R}^q$. $\Phi: U \rightarrow V $ is a smooth map if for any $u \in U$ there is an open subset $O_{{u}}$ of $\mathbb{R}^p$  containing ${u}$ and  a smooth map $\tilde{\Phi}:O_{{u}} \rightarrow  \mathbb{R}^q$ such that $\Phi=\tilde{\Phi}$ on $O_{{u}}$.
\end{definition} 
Suppose $\Phi: U \rightarrow V $ is bijective. Then $\Phi^{-1}$ exists. We say that $\Phi$ is a diffeomorphism if both $\Phi$ and $\Phi^{-1}$ are smooth.  

Suppose $M$ is a $d$-dimensional smooth, closed and connected Riemannian manifold isometrically embedded in $\mathbb{R}^D$ through $\iota$. We define a chart of the embedded submanifold $\iota(M)$. 
\begin{definition}
A map $\Phi$ is a chart (local parametrization) of $\iota(M)$, if there is an open subset $U$ of $\mathbb{R}^d$ and $V$ of $\mathbb{R}^D$ such that $\Phi: U \rightarrow V \cap \iota(M)$ is a diffeomorphism.
\end{definition} 

We refer the reader to \citep{munkres2018analysis} for the above definitions. The definitions of diffeomorphisms and charts can be generalized to the abstract manifold $M$. Since our proofs of the main results focus on $\iota(M)$, we discuss only the relevant definitions for $\iota(M)$. We refer the reader to \citep{lee2012smooth} for more general definitions.

\subsection{The reach of an embedded submanifold}
We recall the definition of the reach of an embedded submanifold in Euclidean space \citep{federer1959curvature}, which is an important concept in manifold reconstruction. 
\begin{definition}
For any point ${y} \in \mathbb{R}^D$, the distance between ${y}$ and $\iota(M)$ is $dist\{{y}, \iota(M)\}=\inf_{{y}' \in \iota(M)} \|{y}-{y}'\|_{\mathbb{R}^D}$. The reach of $\iota(M)$ is the supremum of all $\ell$ such that if ${y} \in \mathbb{R}^D$ and $dist\{{y},\iota(M)\} <\ell$ then there is a unique ${y}' \in \iota(M)$ with $dist\{{y},\iota(M)\}=\|{y}-{y}'\|_{\mathbb{R}^D}$ We denote the reach of $\iota(M)$ as $\tau_{\iota(M)}$.
\end{definition}
The reach $\tau_{\iota(M)}$ of $\iota(M)$ is the largest number such that any point in $\mathbb{R}^D$ at distance less than the reach of $\iota(M)$ has a unique nearest point in $\iota(M)$. The reaches of some special embedded submanifolds can be explicitly calculated. For example,  the reach of a $d$-dimensional sphere in $\mathbb{R}^D$ is the radius of the sphere. 

Suppose $M$ is equipped with a Riemannian metric $g$ and isometrically embedded in $\mathbb{R}^D$ via $\iota$. The geometric properties of $M$, and hence those of $\iota(M)$, which depend only on the metric $g$ are called intrinsic properties. Examples include the volume, diameter, and curvature tensor of $M$. In contrast, the geometric properties of $\iota(M)$ that also depend on the embedding $\iota$ are known as the extrinsic properties of $\iota(M)$. These include the second fundamental form and the reach of $\iota(M)$. We illustrate the dependence of the reach on $\iota$ through the following example: Consider $S^1$ as the unit circle in $\mathbb{R}^2$ and let $a \in S^1$.  Define $\iota_1(M)=S^1\setminus a \subset \mathbb{R}^2 $ and $\iota_2(M)= (0, 2\pi) \times \{0\}$, an interval on the x-axis in $\mathbb{R}^2$. Both $\iota_1(M)$ and $\iota_2(M)$ are different isometric embeddings of the same manifold $M$, but they have different reaches. 

A relevant topological result on the reach of $\iota(M)$ suggests that if there exists a point ${y} \in \mathbb{R}^D$ and a radius $\xi$ such that $B^{\mathbb{R}^D}_{\xi}({y}) \cap \iota(M)$ has more than one connected component, then $\tau_{\iota(M)}<\xi$. Figure \ref{reachillustration} illustrates this concept. We summarize this result in the following proposition.
\begin{figure}[h!]
\centering
{
\includegraphics[width=0.5 \columnwidth]{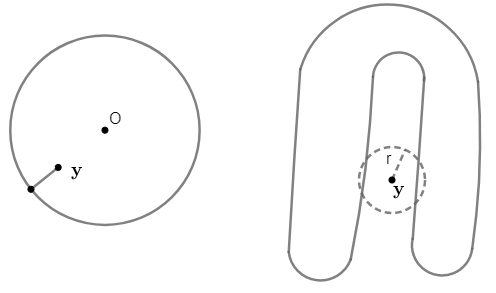}
}
\caption{Left panel: The reach of a circle in the plane is the radius. For the center of the circle, the nearest point on the circle is the whole circle. For any point ${y}$ whose distance to the circle is less than the radius, there is a unique nearest point on the circle.  Right panel: A closed curve in the plane with small reach. We can find a point ${y}$ and a small ball of radius $r$ centered at ${y}$. The intersection between the ball and the curve has two connected components. Hence, by the consequence of Proposition \ref{reach topology}, the reach of the curve is smaller than $r$. }\label{reachillustration}
\end{figure}

\begin{proposition}[Proposition 1 \citep{boissonnat2019reach}]\label{reach topology}
Suppose $0<\xi<\tau_{\iota(M)}$. For any  ${y} \in \mathbb{R}^D$, if $B^{\mathbb{R}^D}_{\xi}({y}) \cap \iota(M) \not= \emptyset$, then  $B^{\mathbb{R}^D}_{\xi}({y}) \cap \iota(M)$ is an open subset of $\iota(M)$ that is homeomorphic to $B^{\mathbb{R}^d}_{1}(0)$.
\end{proposition}

\subsection{Uniform construction of charts}
As we introduced in the definition, for any $x \in M$, there is a chart of $\iota(M)$ over a neighborhood $U_x \cap \iota(M) \subset \mathbb{R}^D$ around $\iota(x)$. However, the size of $U_x$ may not be uniform for all $x$. In the next proposition, by applying Proposition \ref{reach topology}, we show that we can choose the neighborhoods $\{U_x\}$ to be uniform for all $x$. In particular, we can construct a chart over $B^{\mathbb{R}^D}_{\xi}\{\iota(x)\} \cap \iota(M)$ for any $\xi<\frac{\tau_{\iota(M)}}{2}$. The proposition will be used later in the proof of the main theorem. 
\begin{proposition}\label{uniform coordinates}
Suppose $0<\xi<\frac{\tau_{\iota(M)}}{2}$. For any $x \in M$, there is an open set $V_x \subset B^{\mathbb{R}^d}_{\xi}(0) \subset  \mathbb{R}^d$ containing $0$ such that for any ${y} \in B^{\mathbb{R}^D}_{\xi}\{\iota(x)\} \cap \iota(M)$, we have $$ {y}= \iota(x)+X(x) \begin{bmatrix}
{u} \\
G_{x}({u})
\end{bmatrix},$$ where ${u} \in V_x$.  Moreover,
\begin{enumerate}
\item
$X(x) \in \mathbb{O}(D)$ is a $D \times D$ orthogonal matrix  and $\{Xe_i\}_{i=1}^d$ form a basis of $T_{\iota(x)}\iota(M)$.
\item
$V_x$ is homeomorphic to $B^{\mathbb{R}^d}_{1}(0)$.
\item
$G_x({u})=\{g_{x,1}({u}) ,\ldots, g_{x,D-d}({u}) \}^\top \in \mathbb{R}^{D-d}$ and each $g_{x,i}$ is a smooth function on $V_x$.
\item
$g_{x,i}(0)=0$ and $\mathcal{D}g_{x,i}(0)=0$ for $i=1,\ldots, D-d$, where $\mathcal{D}g_{x,i}$ is the derivative of $g_{x,i}$.
\end{enumerate}
\end{proposition}
\begin{proof}
By Proposition \ref{reach topology}, $B^{\mathbb{R}^D}_{\xi}\{\iota(x)\} \cap \iota(M)$ is an open subset of $\iota(M)$ that is homeomorphic to $B^{\mathbb{R}^d}_{1}(0)$. Let $f$ be the projection from $\mathbb{R}^D$ to $T_{\iota(x)}\iota(M)$. For any ${y} \in B^{\mathbb{R}^D}_{\xi}\{\iota(x)\} \cap \iota(M)$, we express $g({y})=f\{{y}-\iota(x)\}$ in a chart around ${y}$, then by Lemma 5.4 in \citep{niyogi2008finding}, $Dg({y})$ is not singular. By the inverse function theorem, $g({y})$ is a diffeomorphism. Hence, the image of $B^{\mathbb{R}^D}_{\xi}\{\iota(x)\} \cap \iota(M)$ under $g$ is the open set $V_x$ that is homeomorphic to $B^{\mathbb{R}^d}_{1}(0)$. Then $g^{-1}$ is the chart with the described properties. 
\end{proof}

\section{Interpolation phase of \texttt{MrGap}}\label{Interpolation phase}
The steps of the interpolation phase of \texttt{MrGap} are listed as Algorithm \ref{MrGap2}.
Suppose that the iterations in Algorithm \ref{MrGap1} stop at the $\texttt{I}$th round.  Then the final denoised outputs are $\{\hat{{y}}_k={y}^{(\texttt{I})}_k\}_{k=1}^n$.  We use the outputs $\{{y}^{(\texttt{I}-1)}_1 , \cdots, {y}^{(\texttt{I}-1)}_n\}$ from the $\texttt{I}-1$ th iteration and the covariance parameters $A^{(\texttt{I}-1)}$, $\rho^{(\texttt{I}-1)}$, and $\sigma^{(\texttt{I}-1)}$ from the $\texttt{I}$th iteration as the inputs to interpolate points on $\iota(M)$. 

\setcounter{algocf}{1}
\begin{algorithm}[ht!]
\SetAlgoLined
\For{$k=1,\cdots,n$}
{
 Let  $\tilde{M}_{k-1}$ be the union of the $K(k-1)$ points already interpolated around $\{ \hat{{y}}_1, \cdots, \hat{{y}}_{k-1} \}$. Use $\{{y}^{(\texttt{I}-1)}_1 , \cdots, {y}^{(\texttt{I}-1)}_n\}$ as input to repeat Step 1 of Algorithm \ref{MrGap1}. Construct $\mathcal{P}_{{y}^{(\texttt{I}-1)}_k}({y}^{(\texttt{I}-1)}_{k,j})$ and $\mathcal{P}^\bot_{{y}^{(\texttt{I}-1)}_k}$. Find $\{{y}^{(\texttt{I}-1)}_j\}_{j=1}^n \cap B^{\mathbb{R}^D}_{\delta}({y}^{(\texttt{I}-1)}_k)=\{{y}^{(\texttt{I}-1)}_{k,1} , \cdots, {y}^{(\texttt{I}-1)}_{k,N^{(\texttt{I}-1)}_k}\}$ and let ${w}^{(\texttt{I}-1)}_{k,j}=\mathcal{P}_{{y}^{(\texttt{I}-1)}_k}({y}^{(\texttt{I}-1)}_{k,j})$ and ${z}^{(\texttt{I}-1)}_{k,j}=\mathcal{P}^\bot_{{y}^{(\texttt{I}-1)}_k}({y}^{(\texttt{I}-1)}_{k,j})$.

Let $\mathcal{U}_k=\frac{1}{N^{(\texttt{I}-1)}_k} \sum_{j=1}^{N^{(\texttt{I}-1)}_k} {w}^{(\texttt{I}-1)}_{k,j}$,
and $(m_k,s_k)$ denote the mean and standard deviation of 
$\{\|{w}^{(\texttt{I}-1)}_{k,1}-\mathcal{U}_k\|_{\mathbb{R}^d}, \cdots, \|{w}^{(\texttt{I}-1)}_{k,N^{(\texttt{I}-1)}_k}-\mathcal{U}_k\|_{\mathbb{R}^d}\}$.
Let 
$\{\tilde{{u}}_{k,1}, \cdots, \tilde{{u}}_{k, K}\}$ be $K$ samples generated uniformly in $B^{\mathbb{R}^d}_{m_k-s_k}(\mathcal{U}_k)$. Suppose $B^{\mathbb{R}^D}_{\delta}({y}^{(\texttt{I}-1)}_k) \cap \tilde{M}_{k-1}=\{\tilde{{y}}_{k,1} , \cdots, \tilde{{y}}_{k, L_k}\}$. Let $\tilde{{w}}_{k,j}=\mathcal{P}_{{y}^{(\texttt{I}-1)}_k}(\tilde{{y}}_{k,j})$ and $\tilde{{z}}_{k,j}=\mathcal{P}^\bot _{{y}^{(\texttt{I}-1)}_k}(\tilde{{y}}_{k,j})$  for  $j=1,\cdots,L_k$. 

Construct the covariance matrix $\tilde{\Sigma}_k= \begin{bmatrix}
\tilde{\Sigma}_{k,1} & \tilde{\Sigma}_{k,2}  \\
\tilde{\Sigma}_{k,3} & \tilde{\Sigma}_{k,4}
\end{bmatrix}$ over $\{{w}^{(\texttt{I}-1)}_{k,1}, \cdots, {w}^{(\texttt{I}-1)}_{k,N^{(\texttt{I}-1)}_k}, \tilde{{w}}_{k,1}, \cdots, \tilde{{w}}_{k,L_k}, \tilde{{u}}_{k,1}, \cdots, \tilde{{u}}_{k, K}\}$ using covariance 
$\mathsf{C}$ in \eqref{GP cov appendix} and the covariance parameters $A^{(\texttt{I}-1)}$ and $\rho^{(\texttt{I}-1)}$, with $\tilde{\Sigma}_{k,1}$ corresponding to 
$\{{w}^{(\texttt{I}-1)}_{k,1}, \cdots, {w}^{(\texttt{I}-1)}_{k,N^{(\texttt{I}-1)}_k}, \tilde{{w}}_{k,1}, \cdots, \tilde{{w}}_{k,L_k}\}$

Denote ${Z}^{(\texttt{I}-1)}_k  \in \mathbb{R}^{N_k\times (D-d)}$ with $j$th row  ${{z}^{(\texttt{I}-1)}_{k,j}
}^\top$, and $\tilde{{Z}}_k \in \mathbb{R}^{L_k \times (D-d)}$ with $j$th row $\tilde{{z}}_{k,j}^\top$. Let ${1}_K$ denote a $K$-vector of ones. Then, the column vectors of 
\begin{align*}
{y}^{(\texttt{I}-1)}_k {1}^\top_K +U^{(\texttt{I}-1)}_{n,\epsilon}({y}^{(\texttt{I}-1)}_k){\scriptsize\begin{bmatrix}
\tilde{{u}}_{k,1}, \ldots, \tilde{{u}}_{k, K} \\
\left[\tilde{\Sigma}_{k,3}\Big\{\tilde{\Sigma}_{k,1}+{\sigma^{(\texttt{I}-1)}}^2 I_{(N^{(\texttt{I}-1)}_k+L_k) \times (N^{(\texttt{I}-1)}_k+L_k) }\Big\}^{-1}\begin{bmatrix}{Z}^{(\texttt{I}-1)}_k \\  \tilde{{Z}}_k \end{bmatrix}\right]^\top \\
\end{bmatrix}} \in \mathbb{R}^{D \times K}.
\end{align*}
give the Euclidean coordinates of the points that we interpolate around $\hat{{y}}_k$. 
}
\caption{\texttt{MrGap} steps to iteratively interpolate $K$ points around each $\hat{{y}}_k$.}\label{MrGap2}
\end{algorithm}

\section{Bias analysis of the local covariance matrix}\label{proof of bias analysis PCA }
Recall Assumption \ref{manifod with noise}, and let $G_{\sigma}(\eta)$ denote the $\mathcal{N}(0, \sigma^2 I_{D \times D})$ probability density function of $\eta$. Since samples $\{\eta_1, \ldots, \eta_n\}$ are independent of $\{x_1,\ldots, x_n\}$, pairs $\{(x_i , \eta_i)\}_{i=1}^n$ can be regarded as $n$ independent and identically distributed samples based on the probability density function $P(x) G_{\sigma}(\eta)$ in $M \times \mathbb{R}^D$. For simplicity of notation, we provide the bias analysis under Assumption \ref{assumption trans and rot}.

Suppose $\mathbb{X}$ is a random variable associated with the probability density function $P(x)$ on $M$.  Suppose $\mathbb{H}$  is a random variable associated with the probability density function $G_{\sigma}(\eta)$ on $\mathbb{R}^D$.  Since $\iota(x_k)=0$ by the Assumption \ref{assumption trans and rot}, we have ${y}_k=\iota(x_k)+\eta_k=\eta_k$. The expectation of the local covariance matrix at ${y}_k$ is defined as follows:
\begin{align}\label{expection local PCA}
C_{\epsilon}({y}_k)= &\mathbb{E}\Big[\{\iota(\mathbb{X})+\mathbb{H}-\eta_k\}\{\iota(\mathbb{X})+\mathbb{H}-\eta_k\}^\top \chi\big(\frac{\|\iota(\mathbb{X})+\mathbb{H}-\eta_k\|_{\mathbb{R}^D}}{\epsilon} \big)\Big] \\
=& \int_{\mathbb{R}^D} \int_M \{\iota(x)+\eta-\eta_k\}\{\iota(x)+\eta-\eta_k\}^\top \chi\big(\frac{\|\iota(x)+\eta-\eta_k\|_{\mathbb{R}^D}}{\epsilon}\big) P(x) G_{\sigma}(\eta) dV d\eta. \nonumber
\end{align}

\begin{remark}
Suppose $f: \mathbb{R}^D \rightarrow \mathbb{R}$ is a measurable function. Then, by a change of variables, 
\begin{align}
\int_{\mathbb{R}^D} \int_M f\{\iota(x)+\eta\} P(x) G_{\sigma}(\eta) dV d\eta = & \int_{\mathbb{R}^D}\int_M f (\eta) P(x) G_{\sigma}\{\eta-\iota(x)\}  dV d\eta\nonumber \\
=& \int_{\mathbb{R}^D} f (\eta) P \star G_{\sigma}(\eta) d\eta, \nonumber 
\end{align}
where the convolution $ P \star G_{\sigma}(\eta)=\int_MP(x) G_{\sigma}\{\eta-\iota(x)\}dV$ is the probability density function for the random variable $\iota(\mathbb{X})+\mathbb{H}$. Hence, the definition in \eqref{expection local PCA} is equivalent to 
\begin{align}
& \mathbb{E}\Big[\{\iota(\mathbb{X})+\mathbb{H}-\eta_k\}\{\iota(\mathbb{X})+\mathbb{H}-\eta_k\}^\top \chi(\frac{\|\iota(\mathbb{X})+\mathbb{H}-\eta_k\|_{\mathbb{R}^D}}{\epsilon} )\Big] \\
=& \int_{\mathbb{R}^D}  (\eta-\eta_k)(\eta-\eta_k)^\top \chi(\frac{\|\eta-\eta_k\|_{\mathbb{R}^D}}{\epsilon}) P \star G_{\sigma}(\eta) d\eta. \nonumber
\end{align}
\end{remark}

In the following proposition, we provide a bias analysis for the local covariance matrix. The proposition is proved later in this section.
\begin{proposition}\label{local PCA bias analysis}
Under Assumptions \ref{manifod with noise} and \ref{assumption trans and rot}, suppose $\epsilon$ is small enough depending on $d$, $D$, the second fundamental form of $\iota(M)$ and the scalar curvature of $M$.   Suppose $\beta>1$, $\sigma \leq \frac{\epsilon^\beta}{\sqrt{-4(d+5)\log \epsilon}}$,  and  $\|\eta_k\|_{\mathbb{R}^D} \leq \epsilon^{\beta}$, then
\begin{align}
C_{\epsilon}({y}_k) =\epsilon^{d+2} \Bigg( \frac{|S^{d-1}|{P}(x_k)}{d(d+2)} 
\begin{bmatrix}
I_{d \times d} & 0 \\
0& 0  \\
\end{bmatrix}+
\begin{bmatrix}
O(\epsilon^{\beta-1}+\epsilon^2) & O(\epsilon^{2\beta-2}+\epsilon^2) \\
O(\epsilon^{2\beta-2}+\epsilon^2) & O(\epsilon^{2\beta-2}+\epsilon^2) 
\end{bmatrix}\Bigg), \nonumber 
\end{align}
where the constant factors in all blocks depend on $d$, $C^2$ norm of $P$,  the second fundamental form of $\iota(M)$ and its derivative and the Ricci curvature of $M$. 
\end{proposition}

Suppose $\iota(x_k)=0$ and there is no noise, then the expectation of the local covariance matrix at ${y}_k$ can be expressed as follows:
\begin{align}
\tilde{C}_{\epsilon}({y}_k)= \mathbb{E}\Big\{\iota(\mathbb{X})\iota(\mathbb{X})^\top \chi(\frac{\|\iota(\mathbb{X})\|_{\mathbb{R}^D}}{\epsilon} )\Big\}. \nonumber
\end{align}
By Proposition 3.1 in \citep{wu2018think}, 
\begin{align}\label{expectation no noise}
\tilde{C}_{\epsilon}({y}_k)=\epsilon^{d+2} \Bigg( \frac{|S^{d-1}|{P}(x_k)}{d(d+2)} 
\begin{bmatrix}
I_{d \times d} & 0 \\
0& 0  \\
\end{bmatrix}+O(\epsilon^2)
\Bigg),
\end{align}
where $O(\epsilon^2)$ represents a $D$ by $D$ matrix whose entries are of order $O(\epsilon^2)$. The constant factors in $O(\epsilon^2)$ depend on $d$, $C^2$ norm of $P$,  the second fundamental form of $\iota(M)$ and its derivative and the Ricci curvature of $M$. If the standard deviation of the Gaussian noise is small enough, so that $\sigma \leq \frac{\epsilon^3}{\sqrt{-4(d+5)\log \epsilon}}$  then 
\begin{align}
C_{\epsilon}({y}_k)=\epsilon^{d+2} \Bigg( \frac{|S^{d-1}|{P}(x_k)}{d(d+2)} 
\begin{bmatrix}
I_{d \times d} & 0 \\
0& 0  \\
\end{bmatrix}+O(\epsilon^2)
\Bigg). \nonumber 
\end{align}
Comparing to Proposition 3.1 in \citep{wu2018think}, the impact of the noise on the local covariance matrix is negligible.

\begin{remark}
Since $\tilde{C}_{\epsilon}({y}_k)$ is invariant under translation and its eigenvalues are invariant under rotation, the leading term $\epsilon^{d+2}\frac{|S^{d-1}|{P}(x_k)}{d(d+2)} 
\begin{bmatrix}
I_{d \times d} & 0 \\
0& 0  \\
\end{bmatrix}$ in \eqref{expectation no noise} implies that $\tilde{C}_{\epsilon}({y}_k)$ has $d$ large eigenvalues of order $\epsilon^{d+2}$, while the rest of the eigenvalues are of order $O(\epsilon^{d+4})$. Refer to Section \ref{determine the dimension} for an illustration. 
\end{remark}

\subsection{Preliminary lemmas}
We introduce some notation. Suppose $S$ is a subset of $M$. Then $\chi_S(x)$ is the characteristic function on $S$.  For $S_1, S_2 \subset M$,  let $S_1 \Delta S_2= (S_1 \setminus S_2) \cup (S_2 \setminus S_1)$ be the symmetric difference.  Let $d_M(x,x')$ be the geodesic distance between $x$ and $x'$ in $M$. Let $B_r(x) \subset M$ be the open geodesic ball of radius $r$ centered at $x \in M$.

We have the following lemma about the geodesic distance between $x$ and $x'$ on $M$ and the Euclidean distance between $\iota(x)$ and $\iota(x')$ in $\mathbb{R}^D$. The proof of the lemma is available in \citep{wu2018think}.
\begin{lemma}\label{geodesic lemma}
Fix $x\in M$. Let ${u} \in T_xM \approx \mathbb{R}^d$ with $\|{u}\|_{\mathbb{R}^d}$ sufficiently small.  Suppose $x'=exp_x({u})$.  Let $\iota_{*}{u}=\frac{d}{dt} \iota (\exp_x(ut))|_{t=0} \in T_{\iota(x)}\iota(M)$.  Then
\begin{align}
\iota(x')-\iota(x)
=&\,\iota_{*}{u}+\frac{1}{2}\Second_x({u},{u})+O(\|{u}\|^3_{\mathbb{R}^d}),\nonumber
\end{align}
where $\Second_x$ is the second fundamental form of $\iota(M)$ at $\iota(x)$. Moreover, 
\begin{align}
\|\iota(x')-\iota(x)\|_{\mathbb{R}^D}=\|{u}\|_{\mathbb{R}^d}-\frac{1}{24} \Big\|\Second_x\Big(\frac{{u}}{\|{u}\|_{\mathbb{R}^d}},\frac{{u}}{\|{u}\|_{\mathbb{R}^d}}\Big)\Big\|^2_{\mathbb{R}^D}  \|{u}\|^3_{\mathbb{R}^d}+O( \|{u}\|^4_{\mathbb{R}^d}).\nonumber
\end{align}
This implies that  
\begin{align}
\|{u}\|_{\mathbb{R}^d}=\|\iota(x')-\iota(x)\|_{\mathbb{R}^D}+\frac{1}{24} \Big\|\Second_x\Big(\frac{{u}}{\|{u}\|_{\mathbb{R}^d}},\frac{{u}}{\|{u}\|_{\mathbb{R}^d}}\Big)\Big\|^2_{\mathbb{R}^D} \|\iota(x')-\iota(x)\|^3_{\mathbb{R}^D}+O( \|\iota(x')-\iota(x)\|^4_{\mathbb{R}^D}). \nonumber
\end{align}
\end{lemma}

Next, we introduce the following lemma about the volume form on $M$. The proof of the lemma is available in \citep{do2013riemannian}.
\begin{lemma} \label{volume form lemma}
Fix $x\in M$. For ${u} \in T_xM \approx \mathbb{R}^d$ with $\|{u}\|_{\mathbb{R}^d}$ sufficiently small, the volume form has the following expansion
\begin{align}
dV=\Big\{1-\sum_{i,j=1}^d\frac{1}{6}\texttt{Ric}_{x}(e'_i,e'_j) u_iu_j+O(\|{u}\|^3_{\mathbb{R}^d})\Big\} d{u},\nonumber
\end{align}
where ${u}=\sum_{i=1}^d u_i e'_i$ and $\texttt{Ric}_{x}$ is the Ricci curvature tensor of $M$ at $x$.  Consequently, 
\begin{align}
\texttt{Vol}\{B_r(x)\}=\texttt{Vol}(B^{\mathbb{R}^d}_1) r^d \Big\{1-\frac{\texttt{S}_{x}}{6(d+2)}r^2+O(r^4)\Big\}, \nonumber
\end{align} 
where $\texttt{S}_{x}$ is the scalar curvature of $M$ at $x$.
\end{lemma}

Recalling that $G_{\sigma}(\eta)$ is the $\mathcal{N}(0, \sigma^2 I_{D \times D})$ probability density function of $\eta$,
we have the following Lemma describing bounds on $G_{\sigma}(\eta)$.

\begin{lemma}\label{Gaussian integral bound}
(1) If $\frac{t}{\sigma} \geq \sqrt{2}(D-2)$, then
$\int_{\|\eta\|_{\mathbb{R}^D}\geq t}G_{\sigma}(\eta)d\eta \leq \frac{2}{\Gamma(\frac{D}{2})}e^{-\frac{t^2}{4\sigma^2}}.$

(2)$\int_{\mathbb{R}^D}\eta(i)^2G_{\sigma}(\eta)d\eta \leq \sigma^{2},$ for $i=1, \ldots, D$. 
\end{lemma}
\begin{proof}
(1) By the change of variables in the spherical coordinates, 
\begin{align} 
\int_{\|\eta\|_{\mathbb{R}^D}\geq t}G_{\sigma}(\eta)d\eta=\frac{1}{(2\pi \sigma^2)^{\frac{D}{2}}} (2\sigma^2)^{\frac{D}{2}} |S^{D-1}| \int_{\frac{t}{\sqrt{2}\sigma}}^\infty r^{D-1} e^{-r^2} dr=\frac{ |S^{D-1}|}{\pi ^{\frac{D}{2}}} \int_{\frac{t}{\sqrt{2}\sigma}}^\infty r^{D-1} e^{-r^2} dr. \nonumber
\end{align}
Since $\frac{r}{2} \geq \log(r)$ for all $r>0$, when $r \geq D-2$, we have $D-2 \leq \frac{r^2}{2 \log(r)}$. Hence, $r^{D-2}e^{-r^2} \leq e^{-\frac{r^2}{2}}$. If $\frac{t}{\sigma} \geq \sqrt{2}(D-2)$, then
\begin{align}
\int_{\|\eta\|_{\mathbb{R}^D}\geq t}G_{\sigma}(\eta)d\eta \leq  \frac{ |S^{D-1}|}{\pi ^{\frac{D}{2}}}\int_{\frac{t}{\sqrt{2}\sigma}}^\infty  e^{-\frac{r^2}{2}}r dr=\frac{2}{\Gamma(\frac{D}{2})}e^{-\frac{t^2}{4\sigma^2}}. \nonumber
\end{align}

(2) follows from the change of variables in the spherical coordinates and the symmetry of $\mathbb{R}^D$.
\end{proof}

Recall that we have $\iota(x_k)=0 \in \mathbb{R}^D$. We define the following terms for our proofs:
\begin{align*}
& F_1\{\iota(x), \eta, \eta_k\}= \{\iota(x)+\eta-\eta_k)\{\iota(x)+\eta-\eta_k\}^\top  \chi(\frac{\|\iota(x)+\eta-\eta_k\|_{\mathbb{R}^D}}{\epsilon}), \\
&E_1=\int_{\|\eta\|_{\mathbb{R}^D}<\epsilon^{\beta}} \int_{M} F_1\{\iota(x), \eta, \eta_k\}P(x) G_{\sigma}(\eta) dV d\eta, \\
&E_2=\int_{\|\eta\|_{\mathbb{R}^D}\geq \epsilon^{\beta}} \int_M  F_1\{\iota(x), \eta, \eta_k\} P(x) G_{\sigma}(\eta) dV d\eta, \\
&E_3=\int_{\|\eta\|_{\mathbb{R}^D}<\epsilon^{\beta}} \int_{M}\{\iota(x)+\eta-\eta_k\}\{\iota(x)+\eta-\eta_k\}^\top \chi_{B_\epsilon(x_k)}(x) P(x) G_{\sigma}(\eta) dV d\eta.
\end{align*}

\begin{lemma}\label{bias analysis lemma 1}
Suppose $\epsilon$ is small enough depending on $d$ and $D$.  For $\beta>1$, if $\sigma \leq \frac{\epsilon^\beta}{\sqrt{-4(d+5)\log \epsilon}}$, then $|e_i^\top E_2 e_j| \leq \epsilon^{d+5} $ for all $1 \leq i,j \leq D$.
\end{lemma}
\begin{proof}
When $\|\iota(x)+\eta-\eta_k\|_{\mathbb{R}^D} \leq \epsilon$, $\chi(\frac{\|\iota(x)+\eta-\eta_k\|_{\mathbb{R}^D}}{\epsilon}) \leq 1$. When $\|\iota(x)+\eta-\eta_k\|_{\mathbb{R}^D} > \epsilon$, $\chi(\frac{\|\iota(x)+\eta-\eta_k\|_{\mathbb{R}^D}}{\epsilon}) =0$. Hence, for all $1 \leq i,j \leq D$ and all $x$, we have 
$$|e_i^\top\{\iota(x)+\eta-\eta_k\}\{\iota(x)+\eta-\eta_k\}^\top e_j | \chi(\frac{\|\iota(x)+\eta-\eta_k\|_{\mathbb{R}^D}}{\epsilon}) \leq \epsilon^2.$$
Next, we bound $|e_i^\top E_2 e_j|$,
\begin{align}
|e_i^\top E_2 e_j| \leq & \int_{\|\eta\|_{\mathbb{R}^D}\geq \epsilon^{\beta}} \int_M  |e_i^\top\{\iota(x)+\eta-\eta_k\}\{\iota(x)+\eta-\eta_k\}^\top e_j | \nonumber \\
& \hspace{40mm} \times \chi(\frac{\|\iota(x)+\eta-\eta_k\|_{\mathbb{R}^D}}{\epsilon})P(x) G_{\sigma}(\eta) dV d\eta \nonumber \\
\leq & \int_{\|\eta\|_{\mathbb{R}^D}\geq \epsilon^{\beta}} \int_M  \epsilon^2 P(x) G_{\sigma}(\eta) dV  d\eta = \int_{\|\eta\|_{\mathbb{R}^D}\geq \epsilon^{\beta}}  \epsilon^2G_{\sigma}(\eta)  d\eta \int_M  P(x)  dV  \nonumber \\
\leq & \int_{\|\eta\|_{\mathbb{R}^D}\geq \epsilon^{\beta}}\epsilon^2 G_{\sigma}(\eta) d\eta. \nonumber 
\end{align}
We apply Fubini's theorem in the second to last step and the fact that $P(x)$ is a probability density function on $M$ in the last step.  Based on the assumption of the lemma, when $\epsilon$ is small enough based on $d$ and $D$, $\frac{\epsilon^\beta}{\sigma} \geq \sqrt{-4(d+5)\log \epsilon} \geq 2\sqrt{2}(D-2)$. Then by Lemma \ref{Gaussian integral bound}, $|e_i^\top E_2 e_j|\leq  \frac{2\epsilon^2}{\Gamma(\frac{D}{2})} e^{-\frac{\epsilon^{2\beta}}{4\sigma^2}}$.
Since $\Gamma(\frac{D}{2})>1$, when $\epsilon \leq \frac{1}{2}$, a straightforward calculation shows that  $\sigma \leq \frac{\epsilon^\beta}{\sqrt{-4(d+5)\log \epsilon}}$ implies $|e_i^\top E_2 e_j|\leq  \frac{2\epsilon^2}{\Gamma(\frac{D}{2})} e^{-\frac{\epsilon^{2\beta}}{4\sigma^2}}\leq \epsilon^{d+5}$.
\end{proof}

\begin{lemma}\label{bias analysis lemma 2}
Suppose $\epsilon$ is small enough depending on the second fundamental form of $\iota(M)$.   For $x \in M$, define $W_{ij}(x)=|e_i^\top \{\iota(x)+\eta-\eta_k\}\{\iota(x)+\eta-\eta_k\}^\top e_j|$.  Suppose $1<\beta \leq 3$. If  $\|\eta\|_{\mathbb{R}^D} \leq \epsilon^{\beta}$ and $\|\eta_k\|_{\mathbb{R}^D} \leq \epsilon^{\beta}$, then  for $x \in B_{ \epsilon+3 \epsilon^{\beta}}(x_k)$
\begin{align*}
&W_{ij}(x) \leq  49\epsilon^2,  \quad  (1 \leq i,j \leq d), \\
&W_{ij}(x) \leq  C_3 (\epsilon^3+\epsilon^{1+\beta}), \quad  (1 \leq i\leq d ; d+1 \leq j \leq D),\\
&W_{ij}(x) \leq C_3 (\epsilon^3+\epsilon^{1+\beta}),  \quad  (d+1 \leq i\leq D; 1 \leq j \leq d), \\
&W_{ij}(x) \leq  C_3 (\epsilon^4+\epsilon^{2\beta}), \quad  (d+1 \leq i,j \leq D), 
\end{align*}
where $C_3$ is a constant depending on the second fundamental form of $\iota(M)$. 

Suppose $\beta > 3$. If  $\|\eta\|_{\mathbb{R}^D} \leq \epsilon^{\beta}$ $\|\eta_k\|_{\mathbb{R}^D} \leq \epsilon^{\beta}$, then  for $x \in B_{ \epsilon+3 \epsilon^{3}}(x_k)$
\begin{align*}
&W_{ij}(x) \leq  49\epsilon^2,  \quad (1 \leq i,j \leq d), \\
&W_{ij}(x) \leq  2C_3 \epsilon^3, \quad (1 \leq i\leq d; d+1 \leq j \leq D),\\
&W_{ij}(x) \leq  2C_3 \epsilon^3,  \quad (d+1 \leq i\leq D; 1 \leq j \leq d), \\
&W_{ij}(x) \leq  2 C_3 \epsilon^4, \quad (d+1 \leq i,j \leq D).
\end{align*}
\end{lemma}

\begin{proof}
Since $\|\eta\|_{\mathbb{R}^D}<\epsilon^{\beta}$ and $\|\eta_k\|_{\mathbb{R}^D} \leq \epsilon^{\beta}$, we have
\begin{align} \label{error bounds in bias 1}
& |e_i^\top \{\iota(x)+\eta-\eta_k\}\{\iota(x)+\eta-\eta_k\}^\top e_j| \\
\leq & (|e_i^\top \iota(x)|+|e_i^\top\eta|+|e_i^\top\eta_k|)(|\iota(x)^\top e_j|+|\eta^\top e_j|+ |\eta_k^\top e_j|) \nonumber \\
\leq & (|e_i^\top \iota(x)|+2\epsilon^\beta)(|\iota(x)^\top e_j|+2\epsilon^\beta).  \nonumber 
\end{align}
For any $x \in B_{ \epsilon+3 \epsilon^{\beta}}(x_k)$, suppose $x=\exp_{x_k}({u})$ for ${u} \in T_{x_k}M \approx \mathbb{R}^d$. Then, $\|{u}\|_{\mathbb{R}^d} \leq  \epsilon+3 \epsilon^{\beta}$. Since $\iota(x_k)=0$, by Lemma \ref{geodesic lemma}, we have 
$$\iota(x)=\iota_{*}{u}+\frac{1}{2}\Second_{x_k}({u},{u})+O(\|{u}\|^3_{\mathbb{R}^d}),$$
where $\iota_{*}{u}=\frac{d}{dt} \iota (\exp_{x_k}(ut))|_{t=0} \in T_{\iota(x_k)}\iota(M)$.
Since $\{e_1,\cdots, e_d\}$ form a basis of $T_{\iota(x_k)}\iota(M)$, they are perpendicular to $\Second_{x_k}({u},{u})$.  Since $1< \beta$, when $\epsilon$ is small enough depending on the second fundamental form of $\iota(M)$, we have $|e_i^\top \iota(x)| \leq 5 \epsilon$ for $i=1 ,\ldots, d$. Similarly, since $\{e_{d+1},\ldots, e_D\}$ form a basis of $T_{\iota(x_k)}\iota(M)^{\bot}$, they are perpendicular to $\iota_{*}{u}$. When $\epsilon$ is small enough depending on the second fundamental form of $\iota(M)$, we have $|e_i^\top \iota(x)| \leq C_2 \epsilon^2$ for $i=d+1 ,\ldots, D$, where $C_2$ is a constant depending on the second fundamental form of $\iota(M)$. If we substitute the bounds of $|e_i^\top \iota(x)|$ into \eqref{error bounds in bias 1}, then after the simplification, for any $\|\eta\|_{\mathbb{R}^D}<\epsilon^{\beta}$, $\|\eta_k\|_{\mathbb{R}^D} \leq \epsilon^{\beta}$ and $x \in B_{ \epsilon+3 \epsilon^{\beta}}(x_k)$, we have
\begin{align}
&|e_i^\top \{\iota(x)+\eta-\eta_k\}\{\iota(x)+\eta-\eta_k\}^\top e_j| \leq  49\epsilon^2,  \quad (1 \leq i,j \leq d), \nonumber \\
&|e_i^\top \{\iota(x)+\eta-\eta_k\}\{\iota(x)+\eta-\eta_k\}^\top e_j| \leq  C_3 (\epsilon^3+\epsilon^{1+\beta}), \quad (1 \leq i\leq d; d+1 \leq j \leq D), \nonumber \\
&|e_i^\top \{\iota(x)+\eta-\eta_k\}\{\iota(x)+\eta-\eta_k\}^\top e_j| \leq C_3 (\epsilon^3+\epsilon^{1+\beta}),  \quad (d+1 \leq i\leq D; 1 \leq j \leq d) \nonumber \\
&|e_i^\top \{\iota(x)+\eta-\eta_k\}\{\iota(x)+\eta-\eta_k\}^\top e_j| \leq  C_3 (\epsilon^4+\epsilon^{2\beta}), \quad (d+1 \leq i,j \leq D), \nonumber  
\end{align}
where $C_3$ is a constant depending on the second fundamental form of $\iota(M)$.  The proof for the case when $\beta>3$ is similar.
\end{proof}

\begin{lemma}\label{bias analysis lemma 3}
Suppose $\epsilon$ is small enough depending on the scalar curvature of $M$ and the second fundamental form of $\iota(M)$. Suppose $1<\beta $. If $\|\eta_k\|_{\mathbb{R}^D} \leq \epsilon^{\beta}$  then
\begin{align*}
E_1=E_3+\begin{bmatrix}
O(\epsilon^{d+1+\beta}+\epsilon^{d+4}) & O(\epsilon^{d+2+\beta}+\epsilon^{d+2\beta}+\epsilon^{d+5}) \\
 O(\epsilon^{d+2+\beta}+\epsilon^{d+2\beta}+\epsilon^{d+5}) &  O(\epsilon^{d+3+\beta}+\epsilon^{d-1+3\beta}+\epsilon^{d+6}) \\
\end{bmatrix},
\end{align*}
where the top left block is a $d$ by $d$ matrix and the constant factors in the four blocks depend on $d$, $P_M$ and the second fundamental form of $\iota(M)$.
\end{lemma}
\begin{proof}
 Let $A(\eta)=\{x \in M  \hspace{1mm} | \hspace{1mm}  \|\iota(x)+\eta-\eta_k\|_{\mathbb{R}^D} \leq \epsilon \mbox{ for a fixed } \|\eta\|_{\mathbb{R}^D}<\epsilon^{\beta} \}$. Since $ \|\eta\|_{\mathbb{R}^D}<\epsilon^{\beta}$ and $\|\eta_k\|_{\mathbb{R}^D} \leq \epsilon^{\beta}$, by the triangle inequality, if $x \in A(\eta)$, then $\epsilon- 2\epsilon^{\beta}  \leq \|\iota(x)\|_{\mathbb{R}^D}\leq \epsilon+ 2\epsilon^{\beta}$. If $1< \beta \leq 3$, when $\epsilon$ is small enough depending on the second fundamental form of $\iota(M)$, by Lemma \ref{geodesic lemma}, we have $\epsilon- 2\epsilon^{\beta} \leq d(x,x_k) <\epsilon+3 \epsilon^{\beta}$. Hence, $B_{\epsilon- 2\epsilon^{\beta}}(x_k) \subset A(\eta) \subset B_{ \epsilon+3 \epsilon^{\beta}}(x_k)$ which implies $A(\eta)  \Delta B_{\epsilon} \subset B_{ \epsilon+3 \epsilon^{\beta}}(x_k) \setminus B_{\epsilon- 2\epsilon^{\beta}}(x_k)$. 
Since $1< \beta \leq 3$, when $\epsilon$ is small enough depending on the scalar curvature of $M$, by Lemma \ref{volume form lemma}, for all $\|\eta\|_{\mathbb{R}^D}<\epsilon^{\beta}$ we have  
\begin{align*}
\texttt{Vol}\{A(\eta)  \Delta B_{\epsilon}(x_k)\} \leq \texttt{Vol}\{B_{ \epsilon+3 \epsilon^{\beta}}(x_k)\}-\texttt{Vol}\{B_{ \epsilon-2 \epsilon^{\beta}}(x_k)\} \leq 6 \texttt{Vol}(B^{\mathbb{R}^d}_1)  d \epsilon^{d-1+\beta}= C_1(d) \epsilon^{d-1+\beta},
\end{align*}
where $C_1(d)=6 d \texttt{Vol}(B^{\mathbb{R}^d}_1) $ is a constant only depending on $d$. Similarly, if $\beta>3$, then for all $\|\eta\|_{\mathbb{R}^D}<\epsilon^{\beta}$ we have $A(\eta)  \Delta B_{\epsilon} \subset B_{ \epsilon+3 \epsilon^{3}}(x_k) \setminus B_{\epsilon- 2\epsilon^{3}}(x_k)$ and
$$\texttt{Vol}\{A(\eta)  \Delta B_{\epsilon}(x_k)\} \leq  C_1(d) \epsilon^{d+2}.$$ 

Observe that for a fixed $ \|\eta\|_{\mathbb{R}^D}<\epsilon^{\beta}$, $|\chi_{A(\eta)}(x)-\chi_{B_\epsilon(x_k)}(x)| \leq \chi_{A(\eta)  \Delta B_{\epsilon}(x_k)}(x)$. Moreover, we have $\chi_{A(\eta)}(x)=\chi(\frac{\|\iota(x)+\eta-\eta_k\|_{\mathbb{R}^D}}{\epsilon})$. Hence, 
\begin{align}\label{bias lemma E1-E3}
& |e_i^\top (E_1-E_3) e_j  | \\
\leq & \int_{\|\eta\|_{\mathbb{R}^D}<\epsilon^{\beta}} \int_{M}|e_i^\top \{\iota(x)+\eta-\eta_k\}\{\iota(x)+\eta-\eta_k\}^\top e_j|  |\chi_{A(\eta) }(x)-\chi_{B_\epsilon(x_k)}(x)|  P(x) G_{\sigma}(\eta) dV d\eta \nonumber \\
\leq & \int_{\|\eta\|_{\mathbb{R}^D}<\epsilon^{\beta}} \int_{M}|e_i^\top \{\iota(x)+\eta-\eta_k\}\{\iota(x)+\eta-\eta_k\}^\top e_j|  \chi_{A(\eta)  \Delta B_{\epsilon}(x_k)}(x) P(x) G_{\sigma}(\eta) dV d\eta \nonumber \\
\leq & \Big( \sup_{\|\eta\|_{\mathbb{R}^D}<\epsilon^{\beta},  x \in A(\eta) \Delta B_{\epsilon}(x_k)}|e_i^\top \{\iota(x)+\eta-\eta_k\}\{\iota(x)+\eta-\eta_k\}^\top e_j| \Big)  \nonumber \\
& \hspace{40mm} \times \int_{\|\eta\|_{\mathbb{R}^D}<\epsilon^{\beta}} \int_{M}\chi_{A(\eta) \Delta B_{\epsilon}(x_k)}(x) P(x) G_{\sigma}(\eta) dV d\eta  \nonumber \\
\leq  & \Big( \sup_{\|\eta\|_{\mathbb{R}^D}<\epsilon^{\beta},  x \in A(\eta) \Delta B_{\epsilon}(x_k)}|e_i^\top \{\iota(x)+\eta-\eta_k\}\{\iota(x)+\eta-\eta_k\}^\top e_j| \Big) P_M \Big(\sup_{\|\eta\|_{\mathbb{R}^D}<\epsilon^{\beta}}\texttt{Vol}\{A(\eta) \Delta B_{\epsilon}(x_k)\}\Big) \nonumber 
\end{align}

Recall that if $1 < \beta \leq 3$, for any $\|\eta\|_{\mathbb{R}^D}<\epsilon^{\beta}$, $A(\eta)  \Delta B_{\epsilon}(x_k) \subset B_{ \epsilon+3 \epsilon^{\beta}}(x_k) \setminus B_{\epsilon- 2\epsilon^{\beta}}(x_k) \subset  B_{ \epsilon+3 \epsilon^{\beta}}(x_k)$. Hence, if we substitute the bounds in Lemma \ref{bias analysis lemma 2} into \eqref{bias lemma E1-E3}, then we conclude that when $1 < \beta \leq 3$ and $\epsilon$ is small enough depending on the scalar curvature of $M$ and the second fundamental form of $\iota(M)$, 
\begin{align*}
E_1=E_3+\begin{bmatrix}
O(\epsilon^{d+1+\beta}) & O(\epsilon^{d+2+\beta}+\epsilon^{d+2\beta}) \\
 O(\epsilon^{d+2+\beta}+\epsilon^{d+2\beta}) &  O(\epsilon^{d+3+\beta}+\epsilon^{d-1+3\beta}) \\
\end{bmatrix},
\end{align*}
where the top left block is a $d$ by $d$ matrix and the constant factors depend on $d$, $P_M$ and the second fundamental form of $\iota(M)$. 

When $\beta > 3$ and $\epsilon$ is small enough depending on the scalar curvature of $M$ and the second fundamental form of $\iota(M)$,    we have $A(\eta)  \Delta B_{\epsilon} \subset B_{ \epsilon+3 \epsilon^{3}}(x_k) \setminus B_{\epsilon- 2\epsilon^{3}}(x_k)$. Hence, if we substitute the bounds in Lemma \ref{bias analysis lemma 2} into \eqref{bias lemma E1-E3}, then
\begin{align*}
E_1=E_3+\begin{bmatrix}
O(\epsilon^{d+4}) & O(\epsilon^{d+5}) \\
 O(\epsilon^{d+5}) &  O(\epsilon^{d+6}) \\
\end{bmatrix},
\end{align*}
where the top left block is a $d$ by $d$ matrix and the constant factors depend on $d$, $P_M$ and the second fundamental form of $\iota(M)$. The statement of the lemma follows from combining the cases when $1 < \beta \leq 3$  and $\beta > 3$.
\end{proof}

\begin{lemma} \label{bias analysis lemma 4}
Suppose $\epsilon$ is small enough depending on $d$, $D$, and the scalar curvature of $M$.  Suppose $\beta>1$, $\sigma \leq \frac{\epsilon^\beta}{\sqrt{-4(d+5)\log \epsilon}}$, and  $\|\eta_k\|_{\mathbb{R}^D} \leq \epsilon^{\beta}$. Then
\begin{align*}
E_3=& \frac{|S^{d-1}|{P}(x_k)}{d(d+2)} \epsilon^{d+2}
\begin{bmatrix}
I_{d \times d} & 0 \\
0& 0  \\
\end{bmatrix}+O(\epsilon^{d+4}+\epsilon^{d+2\beta}),
\end{align*}
where $O(\epsilon^{d+4}+\epsilon^{d+2\beta})$ represents a $D$ by $D$ matrix whose entries are of order $O(\epsilon^{d+4}+\epsilon^{d+2\beta})$. The constants in $O(\epsilon^{d+4}+\epsilon^{d+2\beta})$ depend on $d$, $C^2$ norm of $P$, the second fundamental form of $\iota(M)$ and its derivative and the Ricci curvature of $M$. 
\end{lemma}
\begin{proof}
\begin{align}\label{error bounds in bias 2}
E_3=& \int_{\|\eta\|_{\mathbb{R}^D}<\epsilon^{\beta}} \int_{M}\{\iota(x)+\eta-\eta_k\}\{\iota(x)+\eta-\eta_k\}^\top \chi_{B_\epsilon(x_k)}(x) P(x) G_{\sigma}(\eta) dV d\eta \\
= &  \int_{\|\eta\|_{\mathbb{R}^D}<\epsilon^{\beta}} \int_{M}\big\{\iota(x)\iota(x)^\top+\eta \iota(x)^\top-\eta_k\iota(x)^\top + \iota(x)\eta^\top+\eta\eta^\top \nonumber \\
& -\eta_k\eta^\top - \iota(x)\eta_k^\top-\eta\eta_k^\top+\eta_k\eta_k^\top \big\} \chi_{B_\epsilon(x_k)}(x) P(x) G_{\sigma}(\eta) dV d\eta \nonumber \\
= &  \int_{\|\eta\|_{\mathbb{R}^D}<\epsilon^{\beta}} \int_{M}\big\{\iota(x)\iota(x)^\top-\eta_k\iota(x)^\top +\eta\eta^\top - \iota(x)\eta_k^\top+\eta_k\eta_k^\top \big\} \chi_{B_\epsilon(x_k)}(x) P(x) G_{\sigma}(\eta) dV d\eta, \nonumber 
\end{align}
where we use the symmetry of the region $\{\|\eta\|_{\mathbb{R}^D}<\epsilon^{\beta}\}$ and the symmetry of the function $G_{\sigma}(\eta)$ in the last step.

When $\epsilon$ is small enough depending on the scalar curvature of $M$, by Lemma \ref{volume form lemma}, we have $\texttt{Vol}\{B_{\epsilon}(x_k)\} \leq C_4(d) \epsilon^d$, where $C_4(d)=2\texttt{Vol}(B^{\mathbb{R}^d}_1)$ is a constant only depending on $d$.
We have 
\begin{align*}
& \Big| \int_{\|\eta\|_{\mathbb{R}^D}<\epsilon^{\beta}} \int_{M}e_i^\top \eta\eta^\top e_j \chi_{B_\epsilon(x_k)}(x) P(x) G_{\sigma}(\eta) dV d\eta\Big| \\
\leq &  \Big| \int_{\|\eta\|_{\mathbb{R}^D}<\epsilon^{\beta}} e_i^\top \eta\eta^\top e_j G_{\sigma}(\eta)d\eta\Big|\hspace{1mm} \Big|\int_{M}\chi_{B_\epsilon(x_k)}(x) dV  P(x)\Big| \nonumber \\
\leq & \Big| \int_{\|\eta\|_{\mathbb{R}^D}<\epsilon^{\beta}} e_i^\top \eta\eta^\top e_j G_{\sigma}(\eta)d\eta\Big| C_4(d) \epsilon^d. \nonumber 
\end{align*}

By the symmetry of the region $\{\|\eta\|_{\mathbb{R}^D}<\epsilon^{\beta}\}$ and the symmetry of the function $G_{\sigma}(\eta)$, 
$$\int_{\|\eta\|_{\mathbb{R}^D}<\epsilon^{\beta}} e_i^\top \eta\eta^\top e_j G_{\sigma}(\eta)d\eta=0, \quad (i \not=j).$$   
By Lemma \ref{Gaussian integral bound}, 
$$\Big| \int_{\|\eta\|_{\mathbb{R}^D}<\epsilon^{\beta}} e_i^\top \eta\eta^\top e_j G_{\sigma}(\eta)d\eta\Big| \leq \sigma^2 \leq \epsilon^{2\beta}. \quad (i=j).$$   
In conclusion
\begin{align}\label{bias lemma 3 eq 1}
\int_{\|\eta\|_{\mathbb{R}^D}<\epsilon^{\beta}} \int_{M} \eta\eta^\top \chi_{B_\epsilon(x_k)}(x) P(x) G_{\sigma}(\eta) dV d\eta =O(\epsilon^{d+2\beta}) I_{D \times D},
\end{align}
where the constant in $O(\epsilon^{d+2\beta})$ depends on $d$. 
 
 Similarly,  for $1 \leq i,j \leq D$,
 \begin{align}\label{bias lemma 3 eq 2}
 \int_{\|\eta\|_{\mathbb{R}^D}<\epsilon^{\beta}} \int_{M} e_i^\top \eta_k\eta_k^\top e_j \chi_{B_\epsilon(x_k)}(x) P(x) G_{\sigma}(\eta) dV d\eta=O(\epsilon^{d+2\beta}),
 \end{align} 
 where the constant in $O(\epsilon^{d+2\beta})$ depends on $d$. 
 
 Next, we have
\begin{align*}
 & \Big|\int_{\|\eta\|_{\mathbb{R}^D}<\epsilon^{\beta}} \int_{M} e_i^\top \iota(x)\eta_k^\top e_j \chi_{B_\epsilon(x_k)}(x) P(x) G_{\sigma}(\eta) dV d\eta \Big| \\
\leq &  \Big|\int_{M} e_i^\top\iota(x)\chi_{B_\epsilon(x_k)}(x) P(x) dV \Big|\hspace{1mm}\Big|\int_{\|\eta\|_{\mathbb{R}^D}<\epsilon^{\beta}} G_{\sigma}(\eta) d\eta\Big|\hspace{1mm} \Big|\eta_k^\top e_j\Big| \nonumber \\ 
\leq & \epsilon^{\beta}\Big|\int_{M} e_i^\top\iota(x)\chi_{B_\epsilon(x_k)}(x) P(x) dV \Big|. \nonumber
\end{align*}

Based on the proof of Lemma SI.5 in \citep{wu2018think}, $|\int_{M} e_i^\top\iota(x)\chi_{B_\epsilon(x_k)}(x) P(x) dV |=O(\epsilon^{d+2})$ for $i=1, \cdots, d$, where the constant in $O(\epsilon^{d+2})$ depends on $d$ and the $C^{(1)}$ norm of $P$. $|\int_{M} e_i^\top\iota(x)\chi_{B_\epsilon(x_k)}(x) P(x) dV |=O(\epsilon^{d+2})$ for $i=d+1, \cdots, D$, where the constant in $O(\epsilon^{d+2})$ depends on $d$, $P_M$ and the second fundamental form of $\iota(M)$. In conclusion, for $1 \leq i,j \leq D$,
 \begin{align}\label{bias lemma 3 eq 3}
\int_{\|\eta\|_{\mathbb{R}^D}<\epsilon^{\beta}} \int_{M} e_i^\top \iota(x)\eta_k^\top e_j \chi_{B_\epsilon(x_k)}(x) P(x) G_{\sigma}(\eta) dV d\eta =O(\epsilon^{d+2+\beta}),
\end{align}
where  the constant in $O(\epsilon^{d+2})$ depends on $d$, the $C^{(1)}$ norm of $P$ and the second fundamental form of $\iota(M)$.
Next, observe that 
\begin{align}\label{bias lemma 3 eq 4}
& \int_{\|\eta\|_{\mathbb{R}^D}<\epsilon^{\beta}} \int_{M} \eta_k  \iota(x)^\top \chi_{B_\epsilon(x_k)}(x) P(x) G_{\sigma}(\eta) dV d\eta \\
=& \Big\{\int_{\|\eta\|_{\mathbb{R}^D}<\epsilon^{\beta}} \int_{M}  \iota(x)\eta_k^\top \chi_{B_\epsilon(x_k)}(x) P(x) G_{\sigma}(\eta) dV d\eta\Big\}^\top. \nonumber
\end{align}

At last, 
\begin{align*}
& \int_{\|\eta\|_{\mathbb{R}^D}<\epsilon^{\beta}} \int_{M} \iota(x)\iota(x)^\top \chi_{B_\epsilon(x_k)}(x) P(x) G_{\sigma}(\eta) dV d\eta \\
=&\int_{M} \iota(x)\iota(x)^\top \chi_{B_\epsilon(x_k)}(x) P(x) dV  \int_{\|\eta\|_{\mathbb{R}^D}<\epsilon^{\beta}} G_{\sigma}(\eta) d\eta  \nonumber \\
=&\int_{M} \iota(x)\iota(x)^\top \chi_{B_\epsilon(x_k)}(x) P(x) dV \big(1- \int_{\|\eta\|_{\mathbb{R}^D}\geq \epsilon^{\beta}} G_{\sigma}(\eta) d\eta \big). \nonumber 
\end{align*}
Based on the proof of Proposition 3.1 in \citep{wu2018think},
\begin{align*}
\int_{M} \iota(x)\iota(x)^\top \chi_{B_\epsilon(x_k)}(x) P(x) dV=\frac{|S^{d-1}|{P}(x_k)}{d(d+2)} \epsilon^{d+2}
\begin{bmatrix}
I_{d \times d} & 0 \\
0& 0  \\
\end{bmatrix}+
\begin{bmatrix}
O(\epsilon^{d+4}) & O(\epsilon^{d+4}) \\
O(\epsilon^{d+4}) & O(\epsilon^{d+4}) 
\end{bmatrix},
\end{align*}
where the constant in $O(\epsilon^{d+4}) $ depends on $d$, $C^2$ norm of $P$, the second fundamental form of $\iota(M)$ and its derivative and the Ricci curvature of $M$. 

Based on the assumption of the lemma, when $\epsilon$ is small enough based on $d$ and $D$, $\frac{\epsilon^\beta}{\sigma} \geq \sqrt{-4(d+5)\log \epsilon} \geq 2\sqrt{2}(D-2)$. Then by Lemma \ref{Gaussian integral bound},  $0<\int_{\|\eta\|_{\mathbb{R}^D}\geq \epsilon^{\beta}} G_{\sigma}(\eta) d\eta \leq \frac{2}{\Gamma(\frac{D}{2})} e^{-\frac{\epsilon^{2\beta}}{4\sigma^2}}$.
Since $\Gamma(\frac{D}{2})>1$, when $\epsilon \leq \frac{1}{2}$, a straightforward calculation shows that  $\sigma \leq \frac{\epsilon^\beta}{\sqrt{-4(d+5)\log \epsilon}}$ implies $0<\int_{\|\eta\|_{\mathbb{R}^D}\geq \epsilon^{\beta}} G_{\sigma}(\eta) d\eta \leq \epsilon^{2}$.

In conclusion,
\begin{align} \label{bias lemma 3 eq 5}
& \int_{\|\eta\|_{\mathbb{R}^D}<\epsilon^{\beta}} \int_{M} \iota(x)\iota(x)^\top \chi_{B_\epsilon(x_k)}(x) P(x) G_{\sigma}(\eta) dV d\eta \\
=& \frac{|S^{d-1}|{P}(x_k)}{d(d+2)} \epsilon^{d+2}
\begin{bmatrix}
I_{d \times d} & 0 \\
0& 0  \\
\end{bmatrix}+
\begin{bmatrix}
O(\epsilon^{d+4}) & O(\epsilon^{d+4}) \\
O(\epsilon^{d+4}) & O(\epsilon^{d+4}) 
\end{bmatrix}, \nonumber 
\end{align}
where the constant in $O(\epsilon^{d+4}) $ depends on $d$, $C^2$ norm of $P$, the second fundamental form of $\iota(M)$ and its derivative and the Ricci curvature of $M$. 
If we substitute \eqref{bias lemma 3 eq 1}, \eqref{bias lemma 3 eq 2}, \eqref{bias lemma 3 eq 3}, \eqref{bias lemma 3 eq 4} and \eqref{bias lemma 3 eq 5} into \eqref{error bounds in bias 2}, then we have
\begin{align*}
E_3=& \frac{|S^{d-1}|{P}(x_k)}{d(d+2)} \epsilon^{d+2}
\begin{bmatrix}
I_{d \times d} & 0 \\
0& 0  \\
\end{bmatrix}+O(\epsilon^{d+4}+\epsilon^{d+2+\beta}+\epsilon^{d+2\beta}),
\end{align*}
where $O(\epsilon^{d+4}+\epsilon^{d+2+\beta}+\epsilon^{d+2\beta})$ represents a $D$ by $D$ matrix whose entries are of order $O(\epsilon^{d+4}+\epsilon^{d+2+\beta}+\epsilon^{d+2\beta})$. The constants in $O(\epsilon^{d+4}+\epsilon^{d+2+\beta}+\epsilon^{d+2\beta})$ depend on $d$, $C^2$ norm of $P$, the second fundamental form of $\iota(M)$ and its derivative and the Ricci curvature of $M$.
Observe that if $\beta > 2$, then $\epsilon^{d+4}+\epsilon^{d+2+\beta}+\epsilon^{d+2\beta}$ is bounded by $3\epsilon^{d+4}$. If $\beta \leq 2$, then $\epsilon^{d+4}+\epsilon^{d+2+\beta}+\epsilon^{d+2\beta}$ is bounded by $3\epsilon^{d+2\beta}$. The statement of the lemma follows.
\end{proof}

\subsection{Proof of Proposition \ref{local PCA bias analysis}}

Observe that 
$$ \mathbb{E}[\{\iota(\mathbb{X})+\mathbb{H}-\eta_k\}\{\iota(\mathbb{X})+\mathbb{H}-\eta_k\}^\top \chi(\frac{\|\iota(\mathbb{X})+\mathbb{H}-\eta_k\|_{\mathbb{R}^D}}{\epsilon} )]=E_1+E_2=E_3+(E_1-E_3)+E_2.$$
By combining Lemmas \ref{bias analysis lemma 1}, \ref{bias analysis lemma 3} and \ref{bias analysis lemma 4}, we have 
\begin{align}
C_{\epsilon}({y}_k) =& \epsilon^{d+2}\frac{|S^{d-1}|{P}(x_k)}{d(d+2)} 
\begin{bmatrix}
I_{d \times d} & 0 \\
0& 0  \\
\end{bmatrix} \nonumber \\
& + \begin{bmatrix}
O(\epsilon^{d+1+\beta}+\epsilon^{d+2\beta}+\epsilon^{d+4}) & O(\epsilon^{d+2+\beta}+\epsilon^{d+2\beta}+\epsilon^{d+4}) \\
 O(\epsilon^{d+2+\beta}+\epsilon^{d+2\beta}+\epsilon^{d+4}) &  O(\epsilon^{d+3+\beta}+\epsilon^{d-1+3\beta}+\epsilon^{d+2\beta}+\epsilon^{d+4}) \\
\end{bmatrix}, \nonumber 
\end{align} 
where the constant factors in all the blocks depend on $d$, $C^2$ norm of $P$,  the second fundamental form of $\iota(M)$ and its derivative and the Ricci curvature of $M$. Since $1<\beta$, we have $\epsilon^{d+2\beta} <\epsilon^{d+1+\beta}$, $\epsilon^{d+3+\beta}<\epsilon^{d+4}$ and $\epsilon^{d-1+3\beta}<\epsilon^{d+2\beta}$. If $\beta > 2$, then $\epsilon^{d+4}+\epsilon^{d+2+\beta}+\epsilon^{d+2\beta}$ is bounded by $3\epsilon^{d+4}$. If $\beta \leq 2$, then $\epsilon^{d+4}+\epsilon^{d+2+\beta}+\epsilon^{d+2\beta}$ is bounded by $3\epsilon^{d+2\beta}$. 
Hence
\begin{align}
C_{\epsilon}({y}_k) =& \epsilon^{d+2}\frac{|S^{d-1}|{P}(x_k)}{d(d+2)} 
\begin{bmatrix}
I_{d \times d} & 0 \\
0& 0  \\
\end{bmatrix}+ \begin{bmatrix}
O(\epsilon^{d+1+\beta}+\epsilon^{d+4}) & O(\epsilon^{d+2\beta}+\epsilon^{d+4}) \\
 O(\epsilon^{d+2\beta}+\epsilon^{d+4}) &  O(\epsilon^{d+2\beta}+\epsilon^{d+4}) \\
\end{bmatrix}. \nonumber 
\end{align} 

\section{Proof of Theorem \ref{local PCA  variance analysis} and Theorem \ref{local PCA  spectral behavior} }\label{proof of variance analysis PCA }
\subsection{Preliminary lemmas}
Suppose $U, V \in \mathbb{R}^{m \times n}$. Recall that the Hadamard product of $U$ and $V$ is defined as $(U \circ V)_{ij}=U_{ij}V_{ij}$ for $1\leq i \leq m$ and $1 \leq j \leq n$. Then $U^{(2)}=U \circ U$ is the Hadamard square of $U$. We start by considering the second moment of the local covariance matrix: 
$$\mathbb{E}\Big(\big[\{\iota(\mathbb{X})+\mathbb{H}-\eta_k\}\{\iota(\mathbb{X})+\mathbb{H}-\eta_k\}^\top\big]^{(2)} \chi\Big\{\frac{\|\iota(\mathbb{X})+\mathbb{H}-\eta_k\|_{\mathbb{R}^D}}{\epsilon} \Big\}\Big)$$

Next, we define the following matrices $E_4$ and $E_5$ to simplify our discussion:
\begin{align*}
& F_2 \{\iota(x), \eta, \eta_k\}= \big[\{\iota(x)+\eta-\eta_k\}\{\iota(x)+\eta-\eta_k\}^\top\big]^{(2)}  \chi\Big\{\frac{\|\iota(x)+\eta-\eta_k\|_{\mathbb{R}^D}}{\epsilon}\Big\}, \\
&E_4=\int_{\|\eta\|_{\mathbb{R}^D}<\epsilon^{\beta}} \int_{M} F_2\{\iota(x), \eta, \eta_k\}P(x) G_{\sigma}(\eta) dV d\eta, \\
&E_5=\int_{\|\eta\|_{\mathbb{R}^D}\geq \epsilon^{\beta}} \int_M  F_2\{\iota(x), \eta, \eta_k\} P(x) G_{\sigma}(\eta) dV d\eta. 
\end{align*}

Note that for $1 \leq i,j \leq D$,
\begin{align*}
e_i^\top E_5 e_j=\int_{\|\eta\|_{\mathbb{R}^D} \geq \epsilon^{\beta}} \int_{M}& \big[e_i^\top \{\iota(x)+\eta-\eta_k\}\big]^2\big[\{\iota(x)+\eta-\eta_k\}^\top e_j\big]^2  \\
 & \times \chi\Big\{\frac{\|\iota(x)+\eta-\eta_k\|_{\mathbb{R}^D}}{\epsilon}\Big\}P(x) G_{\sigma}(\eta) dV d\eta. \nonumber
\end{align*}
When $\|\iota(x)+\eta-\eta_k\|_{\mathbb{R}^D} \leq \epsilon$, $\chi\{\frac{\|\iota(x)+\eta-\eta_k\|_{\mathbb{R}^D}}{\epsilon}\} \leq 1$. When $\|\iota(x)+\eta-\eta_k\|_{\mathbb{R}^D} > \epsilon$, $\chi\{\frac{\|\iota(x)+\eta-\eta_k\|_{\mathbb{R}^D}}{\epsilon}\} =0$. Hence, for all $1 \leq i,j \leq D$ and all $x$, we have 
$$\big[e_i^\top \{\iota(x)+\eta-\eta_k\}\big]^2\big[\{\iota(x)+\eta-\eta_k\}^\top e_j\big]^2 \chi\Big\{\frac{\|\iota(x)+\eta-\eta_k\|_{\mathbb{R}^D}}{\epsilon}\Big\} \leq \epsilon^4.$$
By applying the same argument as Lemma \ref{bias analysis lemma 1}, we have the following lemma about $E_5$.

\begin{lemma}\label{Variance analysis lemma 1}
Suppose $\epsilon$ is small enough depending on $d$ and $D$.  For $\beta>1$, if $\sigma \leq \frac{\epsilon^\beta}{\sqrt{-4(d+5)\log \epsilon}}$, then $|e_i^\top E_5 e_j| \leq \epsilon^{d+8}$ for al $1 \leq i,j \leq D$.
\end{lemma}

Now we are ready to bound the second moment of the local covariance matrix.

\begin{lemma}\label{Variance analysis lemma 2}
Suppose $\epsilon$ is small enough depending on $d$, $D$, the scalar curvature of $M$ and the second fundamental form of $\iota(M)$. Suppose $\beta>1$, $\sigma \leq \frac{\epsilon^\beta}{\sqrt{-4(d+5)\log \epsilon}}$, and $\|\eta_k\|_{\mathbb{R}^D} \leq \epsilon^{\beta}$,  then
\begin{align*}
& \mathbb{E}\Big(\big[\{\iota(\mathbb{X})+\mathbb{H}-\eta_k\}\{\iota(\mathbb{X})+\mathbb{H}-\eta_k\}^\top\big]^{(2)} \chi\Big\{\frac{\|\iota(\mathbb{X})+\mathbb{H}-\eta_k\|_{\mathbb{R}^D}}{\epsilon} \Big\}\Big] \\
=&\begin{bmatrix}
O(\epsilon^{d+4}) & O(\epsilon^{d+6}+\epsilon^{d+2+2\beta}) \\
 O(\epsilon^{d+6}+\epsilon^{d+2+2\beta}) &  O(\epsilon^{d+8}+\epsilon^{d+4\beta}) \nonumber \\
\end{bmatrix},
\end{align*}
where the top left block is a $d$ by $d$ matrix and the constant factors in all blocks depend on $d$, $P_M$ and the second fundamental form of $\iota(M)$.
\end{lemma}
\begin{proof} Recall that $B_r(x)$ is a geodesic ball of radius $r$ centered at $x \in M$. Let $A(\eta)=\{x \in M  \hspace{1mm} | \hspace{1mm}  \|\iota(x)+\eta-\eta_k\|_{\mathbb{R}^D} \leq \epsilon \mbox{ for a fixed } \|\eta\|_{\mathbb{R}^D}<\epsilon^{\beta} \}$. Since $ \|\eta\|_{\mathbb{R}^D}<\epsilon^{\beta}$ and $\|\eta_k\|_{\mathbb{R}^D} \leq \epsilon^{\beta}$, by the triangle inequality, if $x \in A(\eta)$,  then $\epsilon- 2\epsilon^{\beta}  \leq \|\iota(x)\|_{\mathbb{R}^D}\leq \epsilon+ 2\epsilon^{\beta}$. If $1< \beta \leq 3$, when $\epsilon$ is small enough depending on the second fundamental form of $\iota(M)$, by Lemma \ref{geodesic lemma}, we have $d(x,x_k) <\epsilon+3 \epsilon^{\beta}$. Hence, $ A(\eta) \subset B_{ \epsilon+3 \epsilon^{\beta}}(x_k)$. Since $1< \beta \leq 3$, when $\epsilon$ is small enough depending on the scalar curvature of $M$, by Lemma \ref{volume form lemma}, for all $ \|\eta\|_{\mathbb{R}^D}<\epsilon^{\beta}$, we have 
\begin{align*}
\texttt{Vol}\{A(\eta)\} \leq \texttt{Vol}\{B_{ \epsilon+3 \epsilon^{\beta}}(x_k)\} \leq 4^d \texttt{Vol}(B^{\mathbb{R}^d}_1)   \epsilon^{d}=C_5(d) \epsilon^d,
\end{align*}
where $C_5(d)=4^d \texttt{Vol}(B^{\mathbb{R}^d}_1)$ is a constant depending on $d$.  Similarly, if  $ \beta > 3$ and $\epsilon$ is small enough depending on the second fundamental form of $\iota(M)$ and the scalar curvature of $M$, then we have $ A(\eta) \subset B_{ \epsilon+3 \epsilon^{3}}(x_k)$ and $\texttt{Vol}\{A(\eta)\} \leq C_5(d) \epsilon^d$.

Observe that for a fixed $\eta$,  we have $\chi_{A(\eta)}(x)=\chi\{\frac{\|\iota(x)+\eta-\eta_k\|_{\mathbb{R}^D}}{\epsilon}\}$. 

Let $Q_{ij}(x, \eta, \eta_k)=\big[e_i^\top \{\iota(x)+\eta-\eta_k\}\big]^2\big[\{\iota(x)+\eta-\eta_k\}^\top e_j\big]^2$, then
\begin{align}\label{Variance Lemma E4}
|e_i^\top E_4 e_j|=& \Big|\int_{\|\eta\|_{\mathbb{R}^D}<\epsilon^{\beta}} \int_{M} e_i^\top F_2\{\iota(x), \eta, \eta_k\} e_j P(x) G_{\sigma}(\eta) dV d\eta\Big| \\
= & \int_{\|\eta\|_{\mathbb{R}^D}<\epsilon^{\beta}} \int_{M}Q_{ij}(x, \eta, \eta_k)  \chi_{A(\eta)}(x) P(x) G_{\sigma}(\eta) dV d\eta \nonumber \\
\leq & \sup_{\|\eta\|_{\mathbb{R}^D}<\epsilon^{\beta}, x \in A(\eta)}Q_{ij}(x, \eta, \eta_k) \int_{\|\eta\|_{\mathbb{R}^D}<\epsilon^{\beta}} \int_{M}\chi_{A(\eta)}(x) P(x) G_{\sigma}(\eta) dV d\eta  \nonumber \\
\leq  & \sup_{\|\eta\|_{\mathbb{R}^D}<\epsilon^{\beta}, x \in A(\eta)}Q_{ij}(x, \eta, \eta_k) P_M \sup_{\|\eta\|_{\mathbb{R}^D}<\epsilon^{\beta}}\texttt{Vol}\{A(\eta)\} \nonumber \\
\leq & \sup_{\|\eta\|_{\mathbb{R}^D}<\epsilon^{\beta}, x \in A(\eta)}Q_{ij}(x, \eta, \eta_k) P_M  C_5(d) \epsilon^{d}. \nonumber 
\end{align}

If $1< \beta \leq 3$ and $\|\eta\|_{\mathbb{R}^D}<\epsilon^{\beta}$, then $x \in A(\eta) \subset B_{ \epsilon+3 \epsilon^{\beta}}(x_k)$. By Lemma \ref{bias analysis lemma 2}, when $\epsilon$ is small enough depending on the second fundamental form of $\iota(M)$, for any $\|\eta\|_{\mathbb{R}^D}<\epsilon^{\beta}$, $\|\eta_k\|_{\mathbb{R}^D} \leq \epsilon^{\beta}$ and $x \in B_{ \epsilon+3 \epsilon^{\beta}}(x_k)$, we have
\begin{align*}
&Q_{ij}(x, \eta, \eta_k) \leq  (49)^2\epsilon^4,  \quad (1 \leq i,j \leq d), \\
&Q_{ij}(x, \eta, \eta_k) \leq  C_3^2 (\epsilon^3+\epsilon^{1+\beta})^2 \leq 2 C_3^2 (\epsilon^6+\epsilon^{2+2\beta}), \quad (1 \leq i\leq d; d+1 \leq j \leq D),  \\
&Q_{ij}(x, \eta, \eta_k) \leq C_3^2 (\epsilon^3+\epsilon^{1+\beta})^2 \leq 2 C_3^2 (\epsilon^6+\epsilon^{2+2\beta}),  \quad (d+1 \leq i\leq D; 1 \leq j \leq d), \\
&Q_{ij}(x, \eta, \eta_k) \leq  C_3^2 (\epsilon^4+\epsilon^{2\beta})^2 \leq  2 C_3^2 (\epsilon^8+\epsilon^{4\beta}), \quad (d+1 \leq i,j \leq D),
\end{align*}
where $C_3$ is a constant depending on the second fundamental form of $\iota(M)$.  Therefore, we can substitute the above bounds into \eqref{Variance Lemma E4}.  We conclude that when $\epsilon$ is small enough depending on the scalar curvature of $M$ and the second fundamental form of $\iota(M)$, 
\begin{align*}
E_4=\begin{bmatrix}
O(\epsilon^{d+4}) & O(\epsilon^{d+6}+\epsilon^{d+2+2\beta}) \\
 O(\epsilon^{d+6}+\epsilon^{d+2+2\beta}) &  O(\epsilon^{d+8}+\epsilon^{d+4\beta}) \\
\end{bmatrix},
\end{align*}
where the top left block is a $d$ by $d$ matrix and the constant factors depend on $d$, $P_M$ and the second fundamental form of $\iota(M)$.

If $\beta > 3$ and $\|\eta\|_{\mathbb{R}^D}<\epsilon^{\beta}$, then $x \in A(\eta) \subset B_{ \epsilon+3 \epsilon^{3}}(x_k)$. By Lemma \ref{bias analysis lemma 2}, when $\epsilon$ is small enough depending on the second fundamental form of $\iota(M)$, for any $\|\eta\|_{\mathbb{R}^D}<\epsilon^{\beta}$, $\|\eta_k\|_{\mathbb{R}^D} \leq \epsilon^{\beta}$ and $x \in B_{ \epsilon+3 \epsilon^{\beta}}(x_k)$, we have
\begin{align*}
&Q_{ij}(x, \eta, \eta_k) \leq  (49)^2\epsilon^4,  \quad (1 \leq i,j \leq d), \\
&Q_{ij}(x, \eta, \eta_k) \leq  4 C_3^2 \epsilon^6, \quad (1 \leq i\leq d; d+1 \leq j \leq D),  \\
&Q_{ij}(x, \eta, \eta_k) \leq 4 C_3^2 \epsilon^6 ,  \quad (d+1 \leq i\leq D; 1 \leq j \leq d), \\
&Q_{ij}(x, \eta, \eta_k) \leq  4C_3^2 \epsilon^8, \quad (d+1 \leq i,j \leq D),
\end{align*}
where $C_3$ is a constant depending on the second fundamental form of $\iota(M)$.  Therefore, we can substitute the above bounds into \eqref{Variance Lemma E4}.  We conclude that when $\epsilon$ is small enough depending on the scalar curvature of $M$ and the second fundamental form of $\iota(M)$, 
\begin{align*}
E_4=\begin{bmatrix}
O(\epsilon^{d+4}) & O(\epsilon^{d+6}) \\
 O(\epsilon^{d+6}) &  O(\epsilon^{d+8}) \\
\end{bmatrix},
\end{align*}
where the top left block is a $d$ by $d$ matrix and the constant factors depend on $d$, $P_M$ and the second fundamental form of $\iota(M)$.
Since 
$$\mathbb{E}\Big(\big[\{\iota(\mathbb{X})+\mathbb{H}-\eta_k\}\{\iota(\mathbb{X})+\mathbb{H}-\eta_k\}^\top\big]^{(2)} \chi\Big\{\frac{\|\iota(\mathbb{X})+\mathbb{H}-\eta_k\|_{\mathbb{R}^D}}{\epsilon} \Big\}\Big)=E_4+E_5,$$
this lemma follows from combining the cases when  $1< \beta \leq 3$ and $\beta > 3$ with Lemma \ref{Variance analysis lemma 1}.
\end{proof}

\subsection{Proof of Theorem \ref{local PCA  variance analysis}}
When $\epsilon$ is small enough based on $d$ and $D$, $\frac{\epsilon^\beta}{\sigma} \geq \sqrt{-4(d+5)\log \epsilon} \geq 2\sqrt{2}(D-2)$. Thus, we are allowed to apply Lemma \ref{Gaussian integral bound} to bound $\text{pr}(\|\eta_i\|_{\mathbb{R}^D} \geq \epsilon^{\beta})$, For a fixed $i$, $\text{pr}(\|\eta_i\|_{\mathbb{R}^D} \geq \epsilon^{\beta})  \leq  \frac{2}{\Gamma(\frac{D}{2})} e^{-\frac{\epsilon^{2\beta}}{4\sigma^2}} \leq 2e^{-\frac{\epsilon^{2\beta}}{4\sigma^2}}$, where we use $ \frac{2}{\Gamma(\frac{D}{2})} \leq 2$ in the last step.

If $\sigma \leq \frac{\epsilon^\beta}{\sqrt{12\log(2n)}}$,  then by a straightforward calculation, we have 
\begin{align*}
\text{pr}(\|\eta_i\|_{\mathbb{R}^D} \geq \epsilon^{\beta}) \leq 2e^{-\frac{\epsilon^{2\beta}}{4\sigma^2}} \leq \frac{1}{4n^3}.
\end{align*}
Since $\{\eta_i\}_{i=1}^n$ are i.i.d. samples, by a straightforward union bound, we have 
\begin{align*}
\text{pr}(\|\eta_i\|_{\mathbb{R}^D} \leq \epsilon^{\beta} | i=1, \cdots, n)  \geq  1-\frac{1}{4n^2}.
\end{align*}

Suppose $\|\eta_k\|_{\mathbb{R}^D} \leq \epsilon^{\beta}$.  Since ${y}_k=\iota(x_k)+\eta_k$ with $\iota(x_k)=0$,  for each $a,b=1,\cdots,D$, we denote
\begin{align}
F_{a,b,i,k}=& e_a^\top ({y}_i-{y}_k)({y}_i-{y}_k)^\top e_b \chi\Big(\frac{\|{y}_i-{y}_k\|_{\mathbb{R}^D}}{\epsilon}\Big)\nonumber \\
=& e_a^\top \{\iota(x_i)+\eta_i-\eta_k\}\{\iota(x_i)+\eta_i-\eta_k\}^\top e_b \chi\Big\{\frac{\|\iota(x_i)+\eta_i-\eta_k\|_{\mathbb{R}^D}}{\epsilon}\Big\}\nonumber
\end{align}
$\{F_{a,b,i,k}\}$ for $i=1, \ldots, n$ and $i \not= k$ are independent and identically distributed realizations of 
$$F_{a,b,k}=e_a^\top \{\iota(\mathbb{X})+\mathbb{H}-\eta_k\}\{\iota(\mathbb{X})+\mathbb{H}-\eta_k\}^\top e_b \chi\Big\{\frac{\|\iota(\mathbb{X})+\mathbb{H}-\eta_k\|_{\mathbb{R}^D}}{\epsilon}\Big\}.$$

If $\|\iota(x)+\eta-\eta_k\|_{\mathbb{R}^D} \leq \epsilon$, then $\chi\{\frac{\|\iota(x)+\eta-\eta_k\|_{\mathbb{R}^D}}{\epsilon}\}=1$, $|e_a^\top \{\iota(x)+\eta-\eta_k\}| \leq \epsilon$  and $|\{\iota(x)+\eta-\eta_k\}^\top e_b| \leq \epsilon$. Hence, $|F_{a,b,k}|$ is uniformly bounded by 
$c=\epsilon^2,$ 
for all $1 \leq a,b \leq D$.

By Proposition \ref{local PCA bias analysis}, when $\epsilon$ is small enough depending on $d$, $D$, the second fundamental form of $\iota(M)$ and the scalar curvature of $M$,
\begin{align*}
& \mathbb{E}\Big[\{\iota(\mathbb{X})+\mathbb{H}-\eta_k\}\{\iota(\mathbb{X})+\mathbb{H}-\eta_k\}^\top \chi\Big\{\frac{\|\iota(\mathbb{X})+\mathbb{H}-\eta_k\|_{\mathbb{R}^D}}{\epsilon} \Big\}\Big] \\
=&\epsilon^{d+2} \Bigg( \frac{|S^{d-1}|{P}(x_k)}{d(d+2)} 
\begin{bmatrix}
I_{d \times d} & 0 \\
0& 0  \\
\end{bmatrix}+
\begin{bmatrix}
O(\epsilon^{\beta-1}+\epsilon^2) & O(\epsilon^{2\beta-2}+\epsilon^2) \\
O(\epsilon^{2\beta-2}+\epsilon^2) & O(\epsilon^{2\beta-2}+\epsilon^2) 
\end{bmatrix}\Bigg), \nonumber 
\end{align*}
where the constant factors in the four blocks depend on $d$,  $C^2$ norm of $P$,  the second fundamental form of $\iota(M)$ and its derivative and the Ricci curvature of $M$. Let 
\begin{align*}
\Omega=& \mathbb{E}\Big(\big[\{\iota(\mathbb{X})+\mathbb{H}-\eta_k\}\{\iota(\mathbb{X})+\mathbb{H}-\eta_k\}^\top\big]^{(2)} \chi\Big\{\frac{\|\iota(\mathbb{X})+\mathbb{H}-\eta_k\|_{\mathbb{R}^D}}{\epsilon} \Big\}\Big) \\
& - \Big( \mathbb{E}\Big[\{\iota(\mathbb{X})+\mathbb{H}-\eta_k\}\{\iota(\mathbb{X})+\mathbb{H}-\eta_k\}^\top \chi\Big\{\frac{\|\iota(\mathbb{X})+\mathbb{H}-\eta_k\|_{\mathbb{R}^D}}{\epsilon} \Big\}\Big] \Big)^{(2)}. \nonumber
\end{align*}
The variance of the random variable $F_{a,b,k}$ is $e_a^\top \Omega e_b$. By combining Proposition \ref{local PCA bias analysis} and Lemma \ref{Variance analysis lemma 2}, 
\begin{align}
& e_a^\top \Omega e_b \leq C_6 \epsilon^{d+4}  \quad (1 \leq a, b \leq d). \nonumber
\end{align}
If $d<a \leq D$ or $d<b \leq D$,
\begin{align}
& e_a^\top \Omega e_b \leq C_6 (\epsilon^{d+6}+\epsilon^{d+2+2\beta}).\nonumber
\end{align}
$C_6$ is a constant depending on $d$,  $C^2$ norm of $P$, the second fundamental form of $\iota(M)$ and its derivative and the Ricci curvature of $M$. 

Observe that
\begin{equation}
e_a^\top C_{n,\epsilon}({y}_k) e_b =\frac{1}{n}\sum_{i=1}^n F_{a,b,i,k}=\frac{n-1}{n} \Big(\frac{1}{n-1}\sum_{ i \not=k, i=1}^n F_{a,b,i,k}\Big).\nonumber
\end{equation}

Since $\frac{n-1}{n}\to 1$ as $n\to\infty$, the error incurred by replacing $\frac{1}{n}$ by $\frac{1}{n-1}$ is of order $\frac{1}{n}$, which is negligible asymptotically. Thus, we can simply focus on analyzing $\frac{1}{n-1}\sum_{ i \not=k, i=1}^n F_{a,b,i,k}$. By Bernstein's inequality,
\begin{align}
\text{pr}\Big\{\Big|\frac{1}{n-1}\sum_{ i \not=k, i=1}^n F_{a,b,i,k}-\mathbb{E}(F_{a,b,k})\Big|>\gamma\Big\} \leq  \exp \Big\{-\frac{(n-1) \gamma^2}{2e_a^\top \Omega e_b+\frac{2}{3}c\gamma}\Big\} \leq \exp \Big(-\frac{n\gamma^2}{4e_a^\top \Omega e_b+\frac{4}{3}c\gamma}\Big). \nonumber
\end{align}

For fixed $a,b$ with $1 \leq a,b \leq d$, we choose $\gamma_1$ so that  $\frac{\gamma_1}{\epsilon^{d+2}} \rightarrow 0$ as $\epsilon \rightarrow 0$, then we  have $4e_a^\top \Omega e_b+\frac{4}{3}c\gamma_1  \leq (4C_6+2) \epsilon^{d+4}$. Hence, 
$$\exp \Big(-\frac{n\gamma_1^2}{4e_a^\top \Omega e_b+\frac{4}{3}c\gamma_1}\Big) \leq \exp \big(-C_7 n \gamma_1^2 \epsilon^{-d-4}\big),$$
where  $C_7=\frac{1}{4C_6+2}$.

For fixed $a,b$ with $d+1 \leq a \leq D$ or $d+1 \leq b \leq D$, we choose $\gamma_2$ so that $\frac{\gamma_2}{\epsilon^{d+2\beta}+\epsilon^{d+4}} \rightarrow 0$ as $\epsilon \rightarrow 0$, then we  have $4e_a^\top \Omega e_b+\frac{4}{3}c\gamma_2 \leq (4C_6+2) (\epsilon^{d+2+2\beta}+\epsilon^{d+6}) \leq  (8C_6+4) \epsilon^{d+2+\min(4,2\beta)}$. Hence,
$$\exp \Big(-\frac{n\gamma_2^2}{4e_a^\top \Omega e_b+\frac{4}{3}c\gamma_2}\Big) \leq \exp \big(-\frac{C_7}{2} n \gamma_2^2 \epsilon^{-d-2-\min(4,2\beta)}\big).
$$
If $\gamma_1=C_8 \sqrt{\frac{\log(n) \epsilon^{d+4}}{n}}$ and $\gamma_2=C_9 \sqrt{\frac{\log(n)\epsilon^{d+2+\min(4,2\beta)}}{n}}$ for some constants $C_8$ and $C_9$ depending on $C_7$ and $D$, then $\exp \big(-C_7 n \gamma_1^2 \epsilon^{-d-4}\big) = \frac{1}{4D^2n^3}$ and $\exp \big(-\frac{C_7}{2} n \gamma_2^2 \epsilon^{-d-2-\min(4,2\beta)}\big) = \frac{1}{4D^2n^3}$. Hence, for fixed $a,b$, we have
\begin{align}
&\text{pr} \Big\{\Big|\frac{1}{n-1}\sum_{ i \not=k, i=1}^n F_{a,b,i,k}-\mathbb{E}(F_{a,b,k})\Big|>C_8 \sqrt{\frac{\log(n)\epsilon^{d+4}}{n}}\Big\} \leq  \frac{1}{4D^2n^3} \quad (1 \leq a,b \leq d). \nonumber
\end{align}
If $d<a \leq D$ or $d<b \leq D$,
\begin{align}
&\text{pr}\Big\{\Big|\frac{1}{n-1}\sum_{ i \not=k, i=1}^n F_{a,b,i,k}-\mathbb{E}(F_{a,b,k})\Big|>C_9 \sqrt{\frac{\log(n)\epsilon^{d+2+\min(4,2\beta)} }{n}}\Big\} \leq  \frac{1}{4D^2n^3}. \nonumber 
\end{align}
By considering the conditional probability such that $\|\eta_k\|_{\mathbb{R}^D} \leq \epsilon^{\beta}$  for $k=1, \cdots, n$ and taking a trivial union bound for all $1 \leq a, b \leq D$ and $1 \leq k \leq n$, we conclude that for all $x_k$ with probability greater than $1- \frac{1}{n^2}$
\begin{align}
&|e_a^\top C_{n,\epsilon}({y}_k) e_b-\mathbb{E}(F_{a,b,k})|<C_8 \sqrt{\frac{\log(n)\epsilon^{d+4}}{n}}  \quad (1 \leq a,b \leq d). \nonumber 
\end{align}
If $d<a \leq D$ or $d<b \leq D$,
\begin{align}
&|e_a^\top C_{n, \epsilon}({y}_k) e_b-\mathbb{E}(F_{a,b,k})|<C_9 \sqrt{\frac{\log(n) \epsilon^{d+2+\min(4,2\beta)}}{n}}. \nonumber 
\end{align}
The conclusion of the proposition follows.

\subsection{Proof of Theorem \ref{local PCA  spectral behavior}}
For any $\beta>1$, suppose $\epsilon$ is small enough depending on $d$, $D$, the second fundamental form of $\iota(M)$ and the scalar curvature of $M$. We apply Theorem \ref{local PCA  variance analysis}.  In the conclusion of Theorem \ref{local PCA  variance analysis}, note that $\sqrt{\frac{\log n}{n \epsilon^d}} \leq \epsilon^{\min(\beta-1,1)}$ is equivalent to $\frac{\log n}{n} \leq \epsilon^{d+2\min(\beta-1,1)}$ and $\sqrt{\frac{\log n}{n \epsilon^{d-2\min(\beta-1,1)}}} \leq \epsilon^{2\min(\beta-1,1)}$ is equivalent to $\frac{\log n}{n} \leq \epsilon^{d+2\min(\beta-1, 1)}$. Hence,  if $\epsilon^{-d-2\min(\beta-1,1)} \leq \frac{n}{\log n}$ and $\sigma \leq \min \big(\frac{1}{\sqrt{-4(d+5)\log \epsilon}}, \frac{1}{\sqrt{12\log(2n)}} \big) \epsilon^\beta$, then for all $x_k$, with probability greater than $1-\frac{1}{n^2}$,
\begin{align*}
C_{n,\epsilon}({y}_k)
=\epsilon^{d+2} \Bigg( \frac{|S^{d-1}|{P}(x_k)}{d(d+2)} 
\begin{bmatrix}
I_{d \times d} & 0 \\
0& 0  \\
\end{bmatrix}+\begin{bmatrix}
O(\epsilon^{\min(\beta-1,1)}) & O(\epsilon^{2\min(\beta-1,1)}) \\
O(\epsilon^{2\min(\beta-1,1)}) & O(\epsilon^{2\min(\beta-1,1)}) 
\end{bmatrix}\Bigg),
\end{align*}
where the constant factors depend on $d$,  $C^2$ norm of $P$, the second fundamental form of $\iota(M)$ and its derivative and the Ricci curvature of $M$.  Hence, 
\begin{align*}
C_{n,\epsilon}({y}_k)
=\epsilon^{d+2} \frac{|S^{d-1}|{P}(x_k)}{d(d+2)} 
\Bigg( \begin{bmatrix}
I_{d \times d}+E_1 & 0 \\
0& 0  \\
\end{bmatrix}+\begin{bmatrix}
E_2 & E_3 \\
E_4 & E_5 \\
\end{bmatrix}\Bigg),
\end{align*}
where the entries of $E_1$ are bounded by $C_{10} \epsilon^{\min(\beta-1,1)}$ and the entries of $E_2$, $E_3$, $E_4$ and $E_5$ are bounded by  $C_{10} \epsilon^{2\min(\beta-1,1)}$. $C_{10}$ depends on $d$,  $P_m$, $C^2$ norm of $P$, the second fundamental form of $\iota(M)$ and its derivative and the Ricci curvature of $M$. 

The scalar curvature is the trace of the Ricci curvature tensor. We will simplify any dependence on the scalar curvature by the dependence on $d$ and the Ricci curvature. Note that the operator norms of $E_1$ and $\begin{bmatrix}
E_2 & E_3 \\
E_4 & E_5 \\
\end{bmatrix}$ can be bounded by $d C_{10} \epsilon^{\min(\beta-1,1)}$ and $D C_{10} \epsilon^{2\min(\beta-1,1)}$ respectively.  Therefore, if we further require that $\epsilon$ is small enough depending on $\beta$, $d$, $D$, $P_m$, $C^2$ norm of $P$, the second fundamental form of $\iota(M)$ and its derivative and the Ricci curvature of $M$, then the operator norms of $E_1$ and $\begin{bmatrix}
E_2 & E_3 \\
E_4 & E_5 \\
\end{bmatrix}$ are bounded by $\frac{1}{3}$. Since
$\begin{bmatrix}
I_{d \times d}+E_1 & 0 \\
0& 0  \\
\end{bmatrix}$ and $\begin{bmatrix}
E_2 & E_3 \\
E_4 & E_5 \\
\end{bmatrix}$ are  symmetric matrices, based on Weyl's theorem, all the eigenvalues of $I_{d \times d}+E_1$ and the first $d$ largest eigenvalues of $\begin{bmatrix}
I_{d \times d}+E_1 & 0 \\
0& 0  \\
\end{bmatrix}+\begin{bmatrix}
E_2 & E_3 \\
E_4 & E_5 \\
\end{bmatrix}$ are bounded below by $\frac{2}{3}$ and the remaining $D-d$ eigenvalues of $\begin{bmatrix}
I_{d \times d}+E_1 & 0 \\
0& 0  \\
\end{bmatrix}+\begin{bmatrix}
E_2 & E_3 \\
E_4 & E_5 \\
\end{bmatrix}$ are bounded above by $\frac{1}{3}$. Consider the eigen decomposition of $C_{n,\epsilon}({y}_k)$ as $C_{n,\epsilon}({y}_k)  =U_{n,\epsilon}({y}_k)\Lambda_{n,\epsilon}({y}_k) U_{n,\epsilon}({y}_k)^\top$ with $U_{n,\epsilon}({y}_k) \in \mathbb{O}(D)$. Then, by Davis-Kahan theorem
\begin{align*}
U_{n,\epsilon}({y}_k)=\begin{bmatrix}
X_1 & 0 \\
0& X_2  \\
\end{bmatrix}+O(\epsilon^{2\min(\beta-1,1)}),
\end{align*}
where $X_1 \in \mathbb{O}(d)$,  $X_2 \in \mathbb{O}(D-d)$.  $O(\epsilon^{2\min(\beta-1,1)})$ represent a $D$ by $D$ matrix whose entries are of order $O\{\epsilon^{2\min(\beta-1,1)}\}$, where the constant factors depend on $d$, $D$, $P_m$, $C^2$ norm of $P$, the second fundamental form of $\iota(M)$ and its derivative and the Ricci curvature of $M$.

\section{Proof of Theorem \ref{Construction of a chart main theorem}}\label{proofs of charts}
We first define a map from $\mathbb{R}^D$ to $\mathbb{R}^d$ generalizing the map $\mathcal{P}_{{y}_k}({y})$. 

\begin{definition}\label{projection PXA}
Fix $x \in M$. Suppose $U \in \mathbb{O}(D)$. For any vector ${a} \in \mathbb{R}^D$, we define a projection map for ${y} \in \mathbb{R}^D$ associated with the first $d$ column vectors of $U$:
$$\mathcal{P}_{U,\iota(x)+{a}}({y})=J^\top U^\top \{{y}-\iota(x)-{a}\}.$$
\end{definition}

Based on the above definition, $\mathcal{P}_{{y}_k}({y})=\mathcal{P}_{U_{n,\epsilon},\iota(x_k)+\eta_k}({y})$. We show that, under suitable conditions on ${a}$  and $U$, the map $\mathcal{P}_{U,\iota(x)+{a}}({y})$ is a local diffeomorphism when it is restricted on $\iota(M)$. Intuitively, for any $x \in M$, if we have a $d$ dimension subspace of $ \mathbb{R}^D$ which does not deviate too far away from the tangent space  $T_{\iota(x)}\iota(M)$, then for any point $\iota(x)+{a} \in \mathbb{R}^D$ that remains close to $\iota(x)$ and any ${y}$ around $\iota(x)+{a}$ on $\iota(M)$,  the projection of ${y}-\iota(x)-{a}$ onto the subspace is a diffeomorphism. Hence,  the map $\mathcal{P}_{U,\iota(x)+{a}}({y})$ can be used to construct a chart.  Our next proposition formalizes this argument and serves as a key component in proving Theorem \ref{Construction of a chart main theorem}.

Before we state the proposition. We make the following observation. For any ${y} \in \iota(M)$, the map $\mathcal{P}_{U,\iota(x)+{a}}({y})$ is invariant under translations of the manifold in $\mathbb{R}^D$: for any ${b}$, we have 
$$\mathcal{P}_{U,\iota(x)+{a}}({y})=J^\top U^\top \{{y}-\iota(x)-{a}\}=J^\top U^\top [{y}+{b}-\{\iota(x)+{b}\}-{a}]^\top.$$ 
Moreover, $\mathcal{P}_{U,\iota(x)+{a}}({y})$ is  invariant under orthogonal transformation in $\mathbb{R}^D$: for any $V \in \mathbb{O}(D)$,
$$\mathcal{P}_{U,\iota(x)+{a}}({y})=J^\top U^\top \{{y}-\iota(x)-{a}\}=J^\top U^\top V^\top V  \{{y}-\iota(x)-{a}\}.$$
Hence, when we discuss the map $\mathcal{P}_{U,\iota(x)+{a}}({y})$, we can always translate $\iota(M)$ and apply an orthogonal transformation in $\mathbb{R}^D$ so that $\iota(x)=0 \in \mathbb{R}^D$ and $T_{\iota(x)}\iota(M)$ is generated by $\{e_i\}_{i=1}^d$.

\begin{proposition}\label{projection forms coordinates}
For $x \in M$, suppose we  translate $\iota(M)$ and apply an orthogonal transformation in  $\mathbb{R}^D$ so that  $\iota(x)=0 \in \mathbb{R}^D$ and  $\{Xe_i\}_{i=1}^d$ form a basis of $T_{\iota(x)}\iota(M)$ where $X=\begin{bmatrix}
X_1 & 0 \\
0 & X_2 \\
\end{bmatrix} \in \mathbb{O}(D)$ with $X_1 \in \mathbb{O}(d)$ and $X_2 \in \mathbb{O}(D-d)$.
Suppose $U \in \mathbb{O}(D)$ and $E=U-X$ with $|E_{ij}|<r$. For any $0<\xi \leq \frac{3\tau_{\iota(M)}}{16}$ and $r \leq \frac{1}{30d^{\frac{3}{2}}}$, if $\|{a}\|_{\mathbb{R}^D} < \frac{\xi}{3}$, then we have the following facts:
\begin{enumerate}
\item
$\mathcal{P}_{U,\iota(x)+{a}}$ is a diffeomorphism from $B^{\mathbb{R}^D}_{\xi}\{\iota(x)+{a}\} \cap \iota(M)$ onto its image $O \subset \mathbb{R}^d$. $O$ is homeomorphic to $B^{\mathbb{R}^d}_{1}(0)$.
\item  There is a point ${u}_0$ in $O$ such that $0 \in B^{\mathbb{R}^d}_{R}({u}_0) \subset O \subset B^{\mathbb{R}^d}_{\xi}(0)$. Let $A= 1-d^{\frac{3}{2}}r$ and $B=\frac{8\xi/3+2\|{a}\|_{\mathbb{R}^D}}{\tau_{\iota(M)}}$. Then $R \geq (\sqrt{A^2-A^2B}- \sqrt{B-A^2B}) (\xi-\|{a}\|_{\mathbb{R}^D})$.
\end{enumerate}
\end{proposition}

\subsection{Preliminary lemmas for the proof of Proposition \ref{projection forms coordinates}}
In order to prove Proposition \ref{projection forms coordinates},  we need to define the angle between two subspaces of the same dimension in $\mathbb{R}^D$.
\begin{definition}
The angle $\phi_{V,W}$ between two subspaces of the same dimension in $\mathbb{R}^D$, $V$ and $W$, is 
$$\phi_{V,W}=\max_{{v} \in V} \min_{{w} \in W} \arccos\Big(\frac{|{v}^\top {w}|}{\|{v}\|_{\mathbb{R}^D}\|{w}\|_{\mathbb{R}^D}}\Big).$$
Based on the definition, we have  $0 \leq \phi_{V,W} \leq \frac{\pi}{2}$.
\end{definition}

We summarize the basic properties of the angle between two subspaces of the same dimension in the following lemma. 

\begin{lemma}\label{properties of angle between subspaces}
Suppose $V$, $W$, $Z$ are three subspaces of the same dimension in $\mathbb{R}^D$. Then, we have the following properties:
\begin{enumerate}[(1)]
\item $\cos(\phi_{V,W})=\min_{{v} \in V} \max_{{w} \in W} \frac{|{v}^\top {w}|}{\|{v}\|_{\mathbb{R}^D}\|{w}\|_{\mathbb{R}^D}}.$
\item If there is a vector ${v} \in V$ which is perpendicular to $W$, then $\phi_{V,W}=\frac{\pi}{2}$.
\item The angle is invariant under orthogonal transformation of $\mathbb{R}^D$. In other words,  if $X \in \mathbb{O}(D)$, then $\phi_{XV,XW}=\phi_{V,W}$.
\item The angle is a metric on the set of all subspaces of the same dimension in $\mathbb{R}^D$. Specifically, we have
\begin{enumerate}
\item (Non-negativeness) $\phi_{V,W} \geq 0$.
\item (Identification) $\phi_{V,W}=0$ implies $V=W$.
\item (Symmetry) $\phi_{V,W}=\phi_{W,V}$.
\item (Triangle inequality) $\phi_{V,Z} \leq \phi_{V,W}+\phi_{W,Z}$ .
\end{enumerate}
\end{enumerate}
\end{lemma}

\begin{proof}
(1) follows from the definition and the fact that cosine is decreasing on $[0, \frac{\pi}{2}]$.  (2) and (3) follow directly from the definition.  A proof of  (4) can be found in \citep{wedin1983angles} 
\end{proof}

We refer the reader to \citep{wedin1983angles} for more discussion about the angle between two subspaces of different dimensions. Next, we prove a lemma which bounds the angle for two sufficiently close subspaces.
\begin{lemma}\label{angle lemma 1}
Suppose $X \in \mathbb{O}(D)$ such that $X=\begin{bmatrix}
X_1 & 0 \\
0 & X_2 \\
\end{bmatrix}$ with $X_1 \in \mathbb{O}(d)$ and $X_2 \in \mathbb{O}(D-d)$. Suppose $U \in \mathbb{O}(D)$ such that $E=U-X$ with $|E_{ij}|<r$.  Let $\{Xe_i\}_{i=1}^d$ be a basis of $V$ and $\{Ue_i\}_{i=1}^d$ be a basis of $W$. Suppose $\phi_{V,W}$ is the angle between $V$ and $W$, then $\cos(\phi_{V,W}) \geq 1-d^{\frac{3}{2}}r$ .  In particular, if $r\leq \frac{1}{30d^{\frac{3}{2}}}$, then $\cos(\phi_{V,W}) \geq \frac{29}{30}$. 
\end{lemma}

\begin{proof}
We write  $E=\begin{bmatrix}
E_1 & E_2 \\
E_3 & E_4 \\
\end{bmatrix}$ with $E_1 \in \mathbb{R}^{d \times d}$. Let ${u} \in \mathbb{R}^d$ be an arbitrary unit vector. Then $\begin{bmatrix}
X_1 \\
0  \\
\end{bmatrix}{u}$ is an arbitrary vector unit in $V$ and $\begin{bmatrix}
X_1+E_1 \\
E_3  \\
\end{bmatrix}{u}$ is a unit vector in $W$. By (1) in Lemma \ref{properties of angle between subspaces}, we have 
\begin{align*}
\cos(\phi_{V,W}) \geq \Big|{u}^\top [X_1^\top, 0] \begin{bmatrix}
X_1+E_1 \\
E_3  \\
\end{bmatrix}{u}\Big|
=|{u}^\top X_1^\top (X_1+E_1 ){u}|=|1+{u}^\top X_1^\top E_1 {u}|
\end{align*}
By the Cauchy-Schwarz inequality, each entry of $X_1^\top E_1$ is bounded by $\sqrt{d}r$ and each entry of $X_1^\top E_1 {u}$ is bounded by $dr$. Hence, $|{u}^\top X_1^\top E_1 {u}|\leq d^{\frac{3}{2}}r$. The conclusion follows.
\end{proof}

Suppose $x,x' \in M$, then the next lemma describes the angle between the tangent spaces of two points $\iota(x)$ and $\iota(x')$ when they are close in the Euclidean distance. The proof is a combination of Proposition 6.2 and Proposition 6.3 in \citep{niyogi2008finding}.
\begin{lemma}\label{angle lemma 2}
Suppose $x,x' \in M$. Let $\phi_{V, W}$ be the angle between $V=T_{\iota(x)}\iota(M)$ and  $W=T_{\iota(x')}\iota(M)$. If $\|\iota(x)-\iota(x')\| \leq \frac{\tau_{\iota(M)}}{2}$, then $\cos(\phi_{V,W}) \geq \sqrt{1-\frac{2\|\iota(x)-\iota(x')\|}{\tau_{\iota(M)}}}$. In particular, if $\|\iota(x)-\iota(x')\| \leq \frac{\tau_{\iota(M)}}{4}$, then $\cos(\phi_{V,W}) \geq \frac{1}{\sqrt{2}}$.
\end{lemma} 

\subsection{Proof of Proposition \ref{projection forms coordinates}}

{Proof of part 1 of Proposition \ref{projection forms coordinates}}

Since we assume $\iota(x)=0$, we express $\mathcal{P}_{U,\iota(x)+{a}}({y})$ as $\mathcal{P}_{U, {a}}({y})$ for notation simplicity. By proposition \ref{uniform coordinates}, suppose we choose $0<\xi_1 \leq \frac{\tau_{\iota(M)}}{4}$, then any ${y} \in B^{\mathbb{R}^D}_{\xi_1}(0) \cap \iota(M)$,  $\mathcal{P}_{U, {a}}({y})$ can be represented in the chart as follows,
\begin{align}
\mathcal{P}_{U, {a}}({y})=& J^\top U^\top \begin{bmatrix}
{u} \nonumber \\
G_x({u})
\end{bmatrix}- J^\top U^\top  {a}, \nonumber 
\end{align}
where ${u} \in V_x \subset B^{\mathbb{R}^d}_{\xi_1}(0) \subset  \mathbb{R}^d$. $V_x$ contains $0$ and is homeomorphic to $B^{\mathbb{R}^d}_{1}(0)$.

Since $\iota(M)$ is smooth, $\mathcal{P}_{U, {a}}({y})$ is smooth.  By Proposition \ref{reach topology}, $B^{\mathbb{R}^D}_{\xi_1}(0) \cap \iota(M)$ is homeomorphic to $B^{\mathbb{R}^d}_{1}(0)$. 
Hence, we can show that $\mathcal{P}_{U, {a}}({y})$ is a diffeomorphism from $B^{\mathbb{R}^D}_{\xi_1}(0) \cap \iota(M)$ onto its image by applying the inverse function theorem.  Note that
$$D\mathcal{P}_{U, {a}}({y})= J^\top U^\top \begin{bmatrix}
I_{d \times d} \\
DG_x({u})
\end{bmatrix}.$$
The column vectors of $\begin{bmatrix}
I_{d \times d} \\
DG_x({u})
\end{bmatrix}$ are a basis of the tangent space $V=T_y \iota(M)$, while the column vectors of $UJ$ are a basis of a subspace $W$. Hence,  $D\mathcal{P}_{U,{a}}({y})$ is not singular if and only if there is no vector in $V$ that is perpendicular to $W$. By (2) in Lemma \ref{properties of angle between subspaces}, it is sufficient to show the angle  $\phi_{V, W}$ between $V$ and $W$ is less than $\frac{\pi}{2}$. By Lemma \ref{angle lemma 1}, if $r \leq \frac{1}{30d^{\frac{3}{2}}}$, then $\cos(\phi_{T_{\iota(x)}\iota(M),W}) \geq \frac{29}{30}$. Since ${y} \in B^{\mathbb{R}^D}_{\xi_1}(0) \cap \iota(M)$, by Lemma \ref{angle lemma 2}, we have $\cos(\phi_{T_{\iota(x)}\iota(M), V}) \geq \frac{1}{\sqrt{2}}$. By (4) in Lemma \ref{properties of angle between subspaces} and a straightforward calculation, 
\begin{align}\label{angle upperbound}
\phi_{V, W} \leq \phi_{T_{\iota(x)}\iota(M),W}+ \phi_{T_{\iota(x)}\iota(M), V} \leq \alpha= \arccos(\frac{29}{30})+\arccos(\frac{1}{\sqrt{2}})<\frac{\pi}{2}.
\end{align}
 By  the inverse function theorem, $\mathcal{P}_{U,\iota(x)+{a}}({y})$ is a diffeomorphism from $B^{\mathbb{R}^D}_{\xi_1}(0) \cap \iota(M)$ onto its image. Choose $\xi=\frac{3\xi_1}{4}$. If $\|{a}\|_{\mathbb{R}^D} < \frac{\xi}{3}=\frac{\xi_1}{4}$, then the closure of $B^{\mathbb{R}^D}_{\xi}({a}) \cap \iota(M)$ is contained in $B^{\mathbb{R}^D}_{\xi_1}(0) \cap \iota(M)$. Hence,  $\mathcal{P}_{U,{a}}({y})$ is a diffeomorphism from the closure of $B^{\mathbb{R}^D}_{\xi}({a}) \cap \iota(M)$ onto its image. By Proposition \ref{reach topology}, $B^{\mathbb{R}^D}_{\xi}({a}) \cap \iota(M)$ is homeomorphic to $B^{\mathbb{R}^d}_{1}(0)$. 

{Proof of part 2 of Proposition \ref{projection forms coordinates}}

Let $O \subset \mathbb{R}^d$ be the image of $B^{\mathbb{R}^D}_{\xi}({a}) \cap \iota(M)$ under  $\mathcal{P}_{U,{a}}({y})$. Then, $O$ is homeomorphic to $B^{\mathbb{R}^d}_{1}(0)$. Let $\bar{O}$ be the closure of $O$ in $\mathbb{R}^d$. Suppose $\Phi: \bar{O} \rightarrow \mathbb{R}^D$ is the inverse of $\mathcal{P}_{U,{a}}({y})$. Then the restriction of $\Phi$ on $O$ is a chart of $\iota(M)$. Since $\|{a}\|_{\mathbb{R}^D} < \frac{\xi}{3}$, we  have $0= \iota(x) \in B^{\mathbb{R}^D}_{\xi}({a}) \cap \iota(M)$. Hence, ${u}_0=\mathcal{P}_{U, {a}}(0)=\mathcal{P}_{U, {a}}(\iota(x)) \in O$.  Based on the definition of $\mathcal{P}_{U,{a}}$ in Definition \ref{projection PXA}, $\Phi$ can be expressed as 
\begin{align*}
\Phi({u})={a}+U \begin{bmatrix}
{u} \\
F({u}), \\
\end{bmatrix}
\end{align*}
where $F({u}): \bar{O} \rightarrow \mathbb{R}^{D-d}$ is a smooth function. For any ${u} \in O$, $\Phi({u}) \in B^{\mathbb{R}^D}_{\xi_1}(0) \cap \iota(M)$. Hence, we apply the same argument as in the proof of part (a).  Recall that $W$ is the subspace generated by the column vectors of $UJ$.
If $V({u})$ is the tangent space of $\iota(M)$ at $\Phi({u})$ for ${u} \in O$.  By Lemma \ref{angle lemma 1}, $\cos(\phi_{T_{\iota(x)}\iota(M),W}) \geq  1-d^{\frac{3}{2}}r$. By Lemma \ref{angle lemma 2}, we have $\cos(\phi_{T_{\iota(x)}\iota(M), V({u})}) \geq  \sqrt{1-\frac{2(\xi_1+\|{a}\|_{\mathbb{R}^D})}{\tau_{\iota(M)}}}$. Hence,  by (4) in Lemma \ref{properties of angle between subspaces}, 
\begin{align}\label{complicated bound on the angle}
\phi_{V({u}), W} \leq \phi_{T_{\iota(x)}\iota(M),W}+ \phi_{T_{\iota(x)}\iota(M), V({u})} \leq \arccos\Big(\sqrt{1-\frac{2(\xi_1+\|{a}\|_{\mathbb{R}^D})}{\tau_{\iota(M)}}}\Big)+\arccos(1-d^{\frac{3}{2}}r).
\end{align}
In particular, by \eqref{angle upperbound}, for ${u} \in O$,
\begin{align}\label{angle upperbound 2}
\cos(\phi_{V({u}), W}) \geq \cos(\alpha) >0.5  
\end{align}

Suppose $B^{\mathbb{R}^d}_{R}({u}_0)$ is the largest open ball centered at ${u}_0$ and contained in $O$, i.e. there is ${u}_1$ on the boundary of $O$ such that $\|{u}_0-{u}_1\|_{\mathbb{R}^d}=R$. Thus, we have 
$$\|\iota(x)-\Phi({u}_1)\|^2_{\mathbb{R}^D}=\|\Phi({u}_0)-\Phi({u}_1)\|^2_{\mathbb{R}^D}=\|{u}_0-{u}_1\|^2_{\mathbb{R}^d}+\|F({u}_0)-F({u}_1)\|^2_{\mathbb{R}^{D-d}}.$$
Since $B^{\mathbb{R}^d}_{R}({u}_0)$ is contained in $O$, any point on  the segment between ${u}_0$ and ${u}_1$ is in $O$. By the mean value inequality, there is ${u}^* \in O$ on the segment between ${u}_0$ and ${u}_1$ such that
$$\|F({u}_0)-F({u}_1)\|_{\mathbb{R}^{D-d}} \leq \|DF({u}^*) {v}\|_{\mathbb{R}^{D-d}} \|{u}_0-{u}_1\|^2_{\mathbb{R}^d},$$
where ${v}=\frac{{u}_0-{u}_1}{\|{u}_0-{u}_1\|_{\mathbb{R}^d}}$. In other words, 
\begin{align}\label{bound the distance by the differential of the normal components}
\|\iota(x)-\Phi({u}_1)\|_{\mathbb{R}^D} \leq R \sqrt{1+\|DF({u}^*) {v}\|^2_{\mathbb{R}^{D-d}}}.
\end{align} 
Based on the definition of $\Phi$, $U\begin{bmatrix}
{v} \\
D F({u}^*){v} \\
\end{bmatrix}$ is a vector in $V({u}^*)$,  the tangent space of $\iota(M)$ at $\Phi({u}^*)$. Hence, $\frac{1}{ \sqrt{1+\|DF({u}^*) {v}\|^2_{\mathbb{R}^{D-d}}}} \begin{bmatrix}
{v} \\
D F({u}^*){v} \\
\end{bmatrix}$ is a unit vector in $U^\top V({u}^*)$. By (3) in Lemma \ref{properties of angle between subspaces}, $\phi_{U^\top V({u}^*), U^\top W}=\phi_{V({u}^*), W} $.  Observe that any unit vector in $U^\top W$ is in the form of  $\begin{bmatrix}
{w} \\
0\\
\end{bmatrix}$, where ${w}$ is a unit vector in $\mathbb{R}^d$. We choose the ${w}$ such that $\begin{bmatrix}
{w} \\
0\\
\end{bmatrix}^\top \frac{1}{ \sqrt{1+\|DF({u}^*) {v}\|^2_{\mathbb{R}^{D-d}}}} \begin{bmatrix}
{v} \\
D F({u}^*){v} \\
\end{bmatrix}$ attains the maximum, then we have
\begin{align}\cos(\phi_{V({u}^*), W})= & \cos(\phi_{U^\top V({u}^*), U^\top W}) \leq \begin{bmatrix}
{w} \\
0\\
\end{bmatrix}^\top \frac{1}{ \sqrt{1+\|DF({u}^*) {v}\|^2_{\mathbb{R}^{D-d}}}} \begin{bmatrix}
{v} \\
D F({u}^*){v}\end{bmatrix} \nonumber  \\
=& \frac{{w}^\top {v}}{ \sqrt{1+\|DF({u}^*) {v}\|^2_{\mathbb{R}^{D-d}}}} \leq \frac{1}{ \sqrt{1+\|DF({u}^*) {v}\|^2_{\mathbb{R}^{D-d}}}}. \nonumber 
\end{align}
By \eqref{bound the distance by the differential of the normal components}, we have $\|\iota(x)-\Phi({u}_1)\|_{\mathbb{R}^D} \leq \frac{R}{\cos(\phi_{V({u}^*), W})}$. Note that since ${u}_1$ is on the boundary of $O$, $\Phi({u}_1)$ is on the boundary of $B^{\mathbb{R}^D}_{\xi}({a}) \cap \iota(M)$. Hence, 
$$\xi =\|{a}-\Phi({u}_1)\|_{\mathbb{R}^D} =\|\iota(x)+{a}-\Phi({u}_1)\|_{\mathbb{R}^D} \leq \|\iota(x)-\Phi({u}_1)\|_{\mathbb{R}^D}+\|{a}\|_{\mathbb{R}^D} \leq  \frac{R}{\cos(\phi_{V({u}^*), W})}+\|{a}\|_{\mathbb{R}^D}.$$
We conclude that $R \geq (\xi-\|{a}\|_{\mathbb{R}^D}) \cos(\phi_{V({u}^*), W})$. By \eqref{complicated bound on the angle}, 
we have 
$$ \cos(\phi_{V({u}^*), W}) \geq (1-d^{\frac{3}{2}}r) \sqrt{1-\frac{2(\xi_1+\|{a}\|_{\mathbb{R}^D})}{\tau_{\iota(M)}}}- \sqrt{\frac{2(\xi_1+\|{a}\|_{\mathbb{R}^D})}{\tau_{\iota(M)}}} \sqrt{1-(1-d^{\frac{3}{2}}r)^2}.$$
The lower bound for $R$ follows by substituting $\xi_1=\frac{4\xi}{3}$ . 

To show that $0 \in O \subset \mathbb{R}^d$, we prove that $\|{u}_0\|_{\mathbb{R}^d} < R$.  First, by \eqref{angle upperbound 2},
$$R \geq (\xi-\|{a}\|_{\mathbb{R}^D}) \cos(\phi_{V({u}^*), W}) \geq \cos(\alpha) (\xi-\|{a}\|_{\mathbb{R}^D})>0.5(\xi-\|{a}\|_{\mathbb{R}^D}).$$ 
Moreover, $\|{u}_0\|_{\mathbb{R}^d} \leq \|{a}\|_{\mathbb{R}^D}$. Since $\|{a}\|_{\mathbb{R}^D} < \frac{\xi}{3}$, we have $\|{a}\|_{\mathbb{R}^D}<0.5(\xi-\|{a}\|_{\mathbb{R}^D})$. The conclusion follows.

\subsection{Proof of Theorem \ref{Construction of a chart main theorem}}
Based on Theorem \ref{local PCA  spectral behavior}, for all $x_i$, $i=1, \cdots ,n$, with probability greater than $1-\frac{1}{n^2}$, we have $\|\eta_i\| _{\mathbb{R}^D}\leq \epsilon^\beta$. When $\epsilon$ is small enough, we have $\epsilon^\beta <\epsilon<\frac{\tau_{\iota(M)}}{16}$. Hence, we have $\|\eta_i\| _{\mathbb{R}^D} \leq \epsilon^\beta < \delta<\frac{\tau_{\iota(M)}}{16}$.  If ${y}_i \in B^{\mathbb{R}^D}_{\delta}({y}_k)$, then by triangle inequality
$$\|\iota(x_i)- {y}_k\| _{\mathbb{R}^D}=\|{y}_i-\eta_i- {y}_k\| _{\mathbb{R}^D} \leq \|{y}_i- {y}_k\| _{\mathbb{R}^D}+\|\eta_i\| _{\mathbb{R}^D} < 3\delta.$$
Hence, $\iota(x_i) \in B^{\mathbb{R}^D}_{3\delta}({y}_k)$.
Moreover, 
\begin{align*}
U_{n,\epsilon}({y}_k)=\begin{bmatrix}
X_1 & 0 \\
0& X_2  \\
\end{bmatrix}+E. 
\end{align*}
By Theorem \ref{local PCA  spectral behavior}, $|E_{ij}|<C \epsilon^{2\min(\beta-1,1)}$, where $C$ is a constant depending on $d$, $D$, $P_m$, $C^2$ norm of $P$, the second fundamental form of $\iota(M)$ and its derivative and the Ricci curvature of $M$. Hence, when $\epsilon$ is small enough depending on $\beta$, $d$, $D$, $P_m$, $C^2$ norm of $P$, the second fundamental form of $\iota(M)$ and its derivative and the Ricci curvature of $M$,  $C \epsilon^{2\min(\beta-1,1)} \leq \frac{1}{30d^{\frac{3}{2}}}$. Note that $\mathcal{P}_{{y}_k}({y})=\mathcal{P}_{U_{n,\epsilon},\iota(x_k)+\eta_k}({y})$. The conditions in Proposition \ref{projection forms coordinates} are satisfied with $\xi=3\delta$, ${a}=\eta_k$  and $r=C \epsilon^{2\min(\beta-1,1)}$. (1) and (2) of Theorem \ref{Construction of a chart main theorem} follows. Since  $\|\eta_i\| _{\mathbb{R}^D} \leq \epsilon^\beta$,  we have $B=\frac{8/3 (3\delta)+2\|\eta_k\|_{\mathbb{R}^D})}{\tau_{\iota(M)}}  \leq \frac{8 \delta +2 \epsilon^\beta}{\tau_{\iota(M)}}$. (3) of Theorem \ref{Construction of a chart main theorem} follows. 

\section{Proof of Theorem \ref{GRMSE estimation}}\label{Appendix: proof of GRMSE estimation}
Since $\mathcal{Y}_{true} \subset \iota(M)$, $\texttt{dist}\{{y}_i, \iota(M)\} \leq \texttt{dist}({y}_i, \mathcal{Y}_{true})$ for all $i$. Hence, $G\{\mathcal{Y}, \iota(M)\} \leq G(\mathcal{Y}, \mathcal{Y}_{true})$. Let $N_r(x)$ be the number of samples in $B^{\mathbb{R}^D}_r\{\iota(x)\} \cap \mathcal{Y}_{true}$. Based on (3) in Theorem 2.4 in \citep{wu2020strong}, if $r \rightarrow 0$ as $m \rightarrow \infty$, then with probability greater than $1-\frac{1}{m^2}$,
$$\sup_{x \in M} \Big|\frac{N_r(x)}{m r^d}-\mathrm{q}(x)\Big| \leq C \sqrt{\frac{\log m}{m r^d}},$$
where $C$ is a constant depending on $d$, $D$, $C^{1}$ norm of $\mathrm{q}$, the curvature of $M$ and the second fundamental form of $\iota(M)$. Hence, if $\frac{\mathrm{q}_{\min}}{2} m r^d \geq 1$ and $\frac{\log m}{m r^d}\leq \frac{\mathrm{q}_{\min}^2}{4 C^2}$, then $N_r(x) \geq \frac{\mathrm{q}_m}{2} m r^d \geq 1$.  Note that $\frac{\mathrm{q}_{\min}}{2} m r^d \geq 1$ and $\frac{\log m}{m r^d}\leq \frac{\mathrm{q}_{\min}^2}{4 C^2}$ are satisfied when $r^d \geq \max(\frac{2}{\mathrm{q}_{\min} m}, \frac{4 C^2 \log m}{\mathrm{q}_{\min}^2 m})$. Hence, we choose $r^d = \max(\frac{2}{\mathrm{q}_{\min}}, (\frac{2C}{\mathrm{q}_{\min}})^2)\frac{\log m}{m}$.

Since $M$ is compact, for any ${y}_i$, let $\iota(x'_i)$ be the point that realizes $\texttt{dist}\{{y}_i, \iota(M)\}$. Then, with probability greater than $1-\frac{1}{m^2}$, there is a point $\iota(x_j) \in \mathcal{Y}_{true}$ such that $\iota(x_j) \in B^{\mathbb{R}^D}_r\{\iota(x'_i)\}$. Hence, 
\begin{align*}
\texttt{dist}({y}_i, \mathcal{Y}_{true}) -r \leq \|{y}_i-\iota(x_j) \|_{\mathbb{R}^D}-r & \leq  \|{y}_i-\iota(x_j) \|_{\mathbb{R}^D}-\|\iota(x'_i)-\iota(x_j) \|_{\mathbb{R}^D} \\
& \leq \|\iota(x'_i)-{y}_i\|_{\mathbb{R}^D}= \texttt{dist}\{{y}_i, \iota(M)\}.
\end{align*}
In conclusion, we have $0 \leq \texttt{dist}({y}_i, \mathcal{Y}_{true}) -  \texttt{dist}\{{y}_i, \iota(M)\}   \leq r$. Consequently,
$$0 \leq  \texttt{dist}({y}_i, \mathcal{Y}_{true}) ^2-  \texttt{dist}\{{y}_i, \iota(M)\}^2 \leq r (\texttt{dist}({y}_i, \mathcal{Y}_{true}) +  \texttt{dist}\{{y}_i, \iota(M))\} \leq  2 r\texttt{dist}\{{y}_i, \iota(M)\} +r^2.$$
By the definition of geometric root mean square error,
$$G(\mathcal{Y}, \mathcal{Y}_{true})^2-G\{\mathcal{Y}, \iota(M)\}^2 \leq 2r \frac{1}{n} \sum_{i=1}^n\texttt{dist}\{{y}_i, \iota(M)\} +r^2 \leq 2r G\{\mathcal{Y}, \iota(M)\}+r^2.$$
The conclusion follows.

\section{Brief review of the Diffusion map}\label{review of DM section}
We provide a briefly review of the Diffusion Map. Given samples $\{{y}_i\}_{i=1}^n \subset \mathbb{R}^D$, the Diffusion Map constructs a normalized graph Laplacian $L \in \mathbb{R}^{n \times n}$ by using the kernel $k({y},{y}')=\exp(-\|{y}-{y}'\|^2_{\mathbb{R}^D}/\epsilon_{DM}^2)$ as shown in the following steps.
\begin{enumerate}
\item
Let $W_{ij}=\frac{k({y}_i,{y}_j)}{q({y}_i) q({y}_j)} \in \mathbb{R}^{n \times n}$, $1 \leq i,j, \leq n$, where $q({y}_i)=\sum_{j=1}^n k({y}_i,{y}_j)$.
\item
Define an $n \times n$ diagonal matrix $\texttt{D}$ as $\texttt{D}_{ii}=\sum_{j=1}^{n} W_{ij}$, where $i=1,\ldots,n$. 
\item The normalized graph Laplacian $L$ is defined as $L=(\texttt{D}^{-1}W-I)/\epsilon^2_{DM} \in  \mathbb{R}^{n \times n}$.
\end{enumerate}
Suppose $(\mu_{j}, V_j)_{j=0}^{n-1}$ are the eigenpairs of $-L$ with $\mu_0 \leq \mu_1 \leq \cdots, \leq \mu_{n-1}$ and $V_i$ normalizing in $\ell^2$. Then $\mu_0=0$ and $V_0$ is a constant vector. The map $(V_1 , \cdots, V_{\ell})$ provides the coordinates of the data set $\{{y}_i\}_{j=1}^n$ in a low-dimensional space $\mathbb{R}^{\ell}$.

When $\{{y}_i=\iota(x_i)\}_{i=1}^n$ are samples from an isometrically embedded submanifold $\iota(M) \subset \mathbb{R}^D$,  then \citep{coifman2006diffusion} shows that $L$ approximates the Laplace-Beltrami operator $\Delta$ of $M$ pointwisely. In particular, when the boundary of $M$ is not empty, $L$ pointwisely approximates $\Delta$ with Neumann boundary condition. Moreover, when $M$ is a closed manifold, the spectral convergence rate of $-L$ to $-\Delta$ is discussed in \citep{DUNSON2021,cheng2022eigen,calder2022lipschitz}. Suppose the eigenvectors of $-L$ are normalized properly. 
\citep{DUNSON2021} show that the first $K$ eigenpairs of $-L$ approximate the corresponding eigenpairs of $-\Delta$ over $\{x_i\}_{i=1}^n$ whenever $\epsilon_{DM}$ is small enough depending on $K$ and $n$ is large enough depending on $\epsilon_{DM}$. Based on the results in spectral geometry, by choosing $\ell$ sufficiently large, the Diffusion Map $(V_1 , \cdots, V_{\ell})$ approximates the discretization of an embedding of $\iota(M)$ into $\mathbb{R}^{\ell}$.

\section{Estimate the dimension of the data set}\label{determine the dimension}
Suppose data $\{{y}_i\}_{i=1}^n \subset \mathbb{R}^D$ satisfy Assumption \ref{manifod with noise}.  In this section, we describe a method to estimate the dimension $d$ of the underlying manifold $M$ by using the Diffusion Map. Let $C_{n,\epsilon}({y}_k)$ be the local covariance matrix at ${y}_k$ constructed by  $\{{y}_i\}_{i=1}^n$ as defined in \eqref{Gaussian local covariance matrix}. Suppose $\lambda_{n,\epsilon,i}({y}_k)$ is the $i$th largest eigenvalue of $C_{n,\epsilon}({y}_k)$. Then, we define the mean of the $i$th eigenvalues of the local covariance matrices as
\begin{align}\label{average of eigenvalues}
\bar{\lambda}_{\epsilon,i}=\frac{1}{n}\sum_{k=1}^n \lambda_{n,\epsilon,i}({y}_k).
\end{align}

We demonstrate our method by using the following example. Consider the surface of the ellipsoid  in $R^3$ described by the equation 
$$\frac{x^2}{4}+\frac{y^2}{2.25}+z^2=1$$
We sample $800$ points  $\{(x'_i, y'_i, z'_i)\}_{i=1}^{800}$ uniformly on the surface. Let $R$ be an orthogonal matrix of $\mathbb{R}^3$. We rotate the surface by using $R$, i.e. we get $800$ points $\{(x''_i, y''_i, z''_i)\}_{i=1}^{800}$ where $$\begin{bmatrix} x''_i \\ y''_i  \\ z''_i \end{bmatrix}=R \begin{bmatrix}x'_i \\ y'_i  \\ z'_i \end{bmatrix}.$$ Hence, $\iota(x_i)=(0, \cdots, 0, x''_i, y''_i, z''_i, 0, \cdots, 0) \in \mathbb{R}^{30}$ ($x''_i$, $y''_i$,  and $z''_i$ are the $14$th, $15$th, and $16$th coordinates, respectively) is a point on a surface of the ellipsoid $\iota(M)$ in $\mathbb{R}^{30}$. Suppose $\{\eta_i\}_{i=1}^{800}$ are i.i.d samples from $\mathcal{N}(0, \sigma^2 I_{30 \times 30})$ where $\sigma=0.05$. Then  ${y}_i=\iota(x_i)+\eta_i$ is a noisy data point around the submanifold $\iota(M)$ in $\mathbb{R}^{30}$. We choose the bandwidths $\epsilon=0.3, 0.4, 0.5, 0.6$ and find $\{\bar{\lambda}_{\epsilon,i}\}_{i=1}^{30}$ by using $\{\iota(x_i)\}_{i=1}^{800}$. We show the plot of $\{\bar{\lambda}_{\epsilon,i}\}_{i=1}^{30}$ for different $\epsilon$ in Fig \ref{Eig0} . We can see that there are two large eigenvalues. Hence, for this example, when there is no noise, the eigenvalues of the local covariance matrices constructed from $\{\iota(x_i)\}_{i=1}^{800}$ are sufficient to determine the dimension of the underlying manifold. 
\begin{figure}[htb!]
\centering
\includegraphics[width=0.9 \columnwidth]{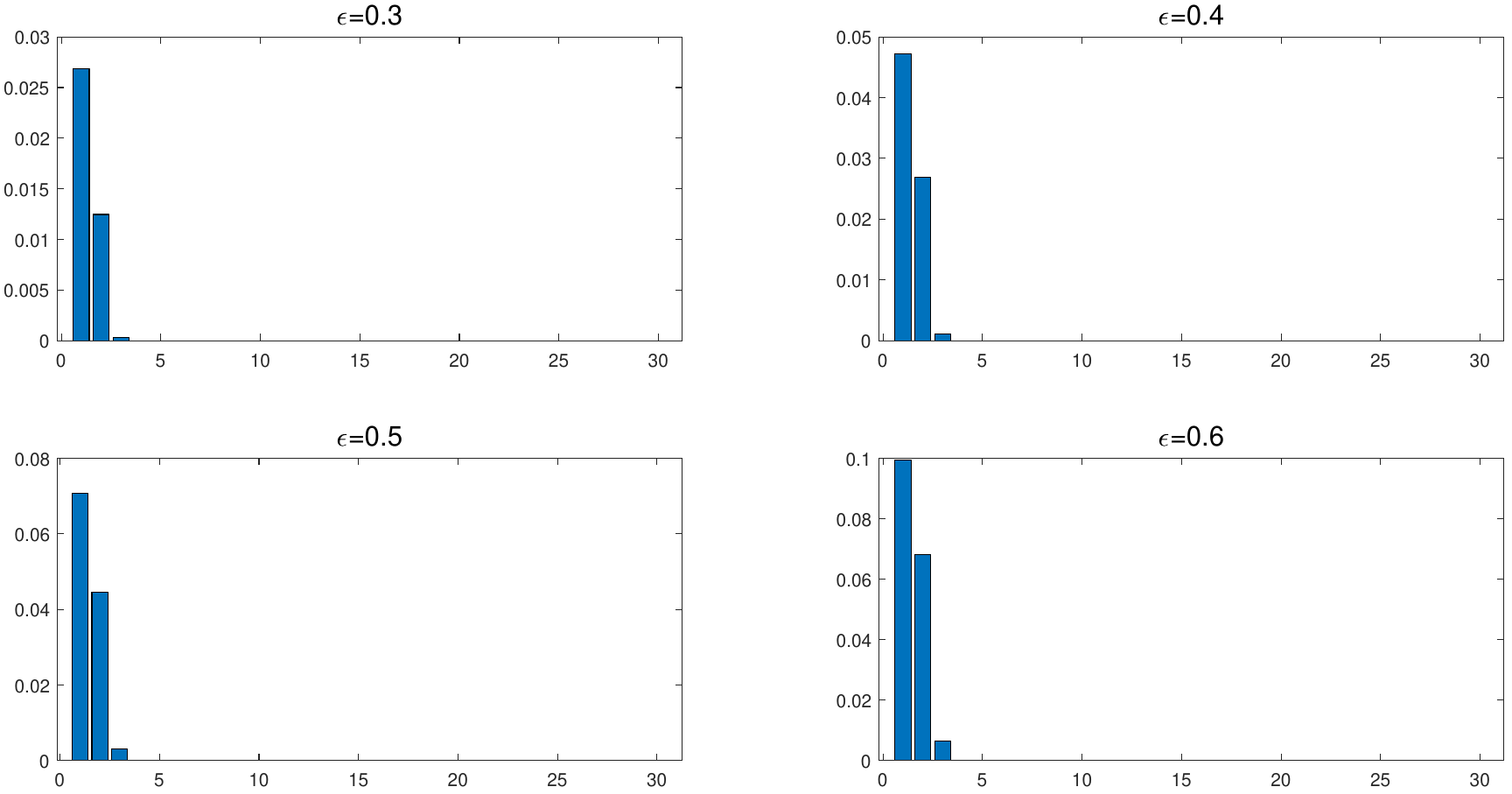}
\caption{ We choose the bandwidths $\epsilon=0.3, 0.4, 0.5, 0.6$ and find $\{\bar{\lambda}_{\epsilon,i}\}_{i=1}^{30}$ by using $\{\iota(x_i)\}_{i=1}^{800}$. In the above plots, the horizontal axis indicates the $i=1, \cdots, 30$ and the vertical axis indicates the value of corresponding $\bar{\lambda}_{\epsilon,i}$.}\label{Eig0}
\end{figure}

Next, we choose the bandwidths $\{\epsilon=0.1 \ell\}_{\ell=1}^{30}$ and find $\{\bar{\lambda}_{\epsilon,i}\}_{i=1}^{30}$ by using $\{{y}_i\}_{i=1}^{800}$.  We show the plot of $\{\bar{\lambda}_{\epsilon,i}\}_{i=1}^{30}$ for $\epsilon=0.5, 1, \cdots, 3$ in Fig \ref{Eig1} .  Due to the noise, we cannot determine the dimension of $\iota(M)$ correctly by using the eigenvalues of the local covariance matrices constructed from $\{{y}_i\}_{i=1}^{800}$. 

\begin{figure}[htb!]
\centering
\includegraphics[width=0.9 \columnwidth]{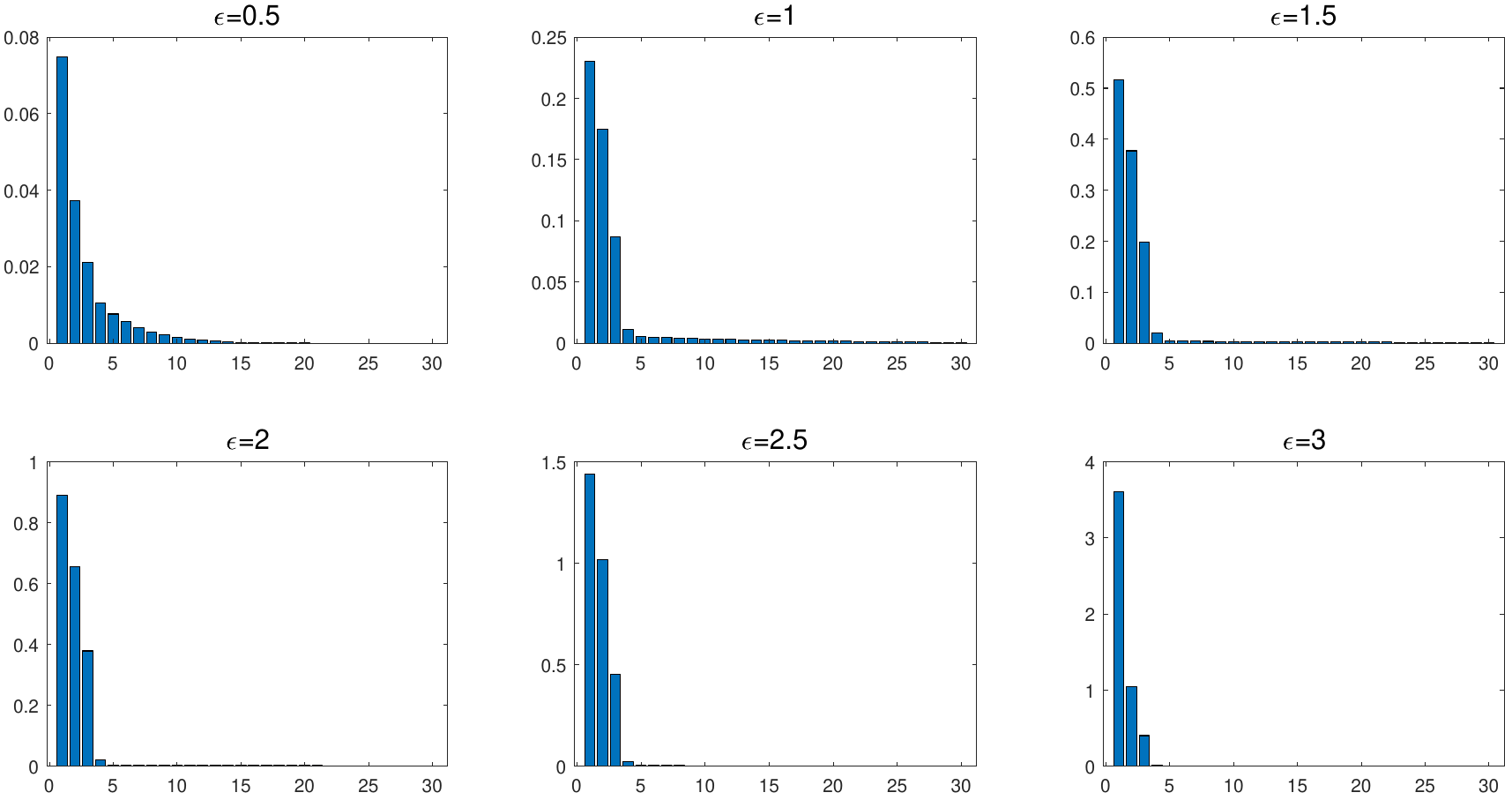}
\caption{For $\{\epsilon=0.1 \ell\}_{\ell=1}^{30}$, we find $\{\bar{\lambda}_{\epsilon,i}\}_{i=1}^{30}$ by using $\{{y}_i\}_{i=1}^{800}$. We plot $\{\bar{\lambda}_{\epsilon,i}\}_{i=1}^{30}$ for $\epsilon=0.5, 1, \cdots, 3$. In the above plots, the horizontal axis indicates $i=1, \cdots, 30$ and the vertical axis indicates the value of the corresponding $\bar{\lambda}_{\epsilon,i}$ .}\label{Eig1}
\end{figure}

The robustness of diffusion map to noise is discussed in \citep{el2016graph, shen2020scalability, dunson2020diffusion, ding2020impact}. Based on the description of the algorithm in Section \ref{review of DM section}, we  construct the normalized graph Laplacian $L \in \mathbb{R}^{n \times n}$ by using the kernel $k({y},{y}')=\exp(-\|{y}-{y}'\|^2_{\mathbb{R}^D}/\epsilon_{DM}^2)$ and the noisy data points $\{{y}_i\}_{i=1}^{800}$. We choose $\epsilon_{DM}=2$ and let $(\mu_{j}, V_j)_{j=0}^{n-1}$ be the eigenpairs of $-L$ with $\mu_0 \leq \mu_1 \leq \cdots, \leq \mu_{n-1}$ and $V_i$ normalized in $\ell^2$. 
We reconstruct the noisy data points $\{{y}_i\}_{i=1}^{800}$ in $\mathbb{R}^{2+j}$ by using $(V_1, \cdots, V_{2+j}) \in \mathbb{R}^{800 \times (2+j)}$ and denote these data points in $\mathbb{R}^{2+j}$ as $\mathcal{X}_j$ for $j=1, \cdots, 4$.  For each $ j=1,\cdots, 4$, the diffusion map $(V_1 , \cdots, V_{2+j})$ can be regarded as an approximation of a discretization of an embedding of $\iota(M)$ into $\mathbb{R}^{2+j}$ over the clean data point  $\{\iota(x_i)\}_{i=1}^{800}$ on $\iota(M)$. Hence, $\mathcal{X}_j$ are the samples around an embedded submanifold $N_j$ in $\mathbb{R}^{2+j}$ homeomorphic to $M$.  The dimension of $M$ is the same as the dimension of $N_j$.
For each $ j=1,\cdots, 4$,   we construct the local covariance matrices by using $\mathcal{X}_j$ with $\epsilon=0.3+0.1j$ and calculate  $\{\bar{\lambda}^j_{\epsilon,i}\}_{i=1}^{2+j}$. We plot $\{\bar{\lambda}^j_{\epsilon,i}\}_{i=1}^{2+j}$ for each $j$ in Figure \ref{Eig2}. We can see that in each plot there are $2$ large eigenvalues. Therefore, the dimension of the underlying manifold $N_j$ and the dimension of $M$ are $2$. 
\begin{figure}[htb!]
\centering
\includegraphics[width=0.9 \columnwidth]{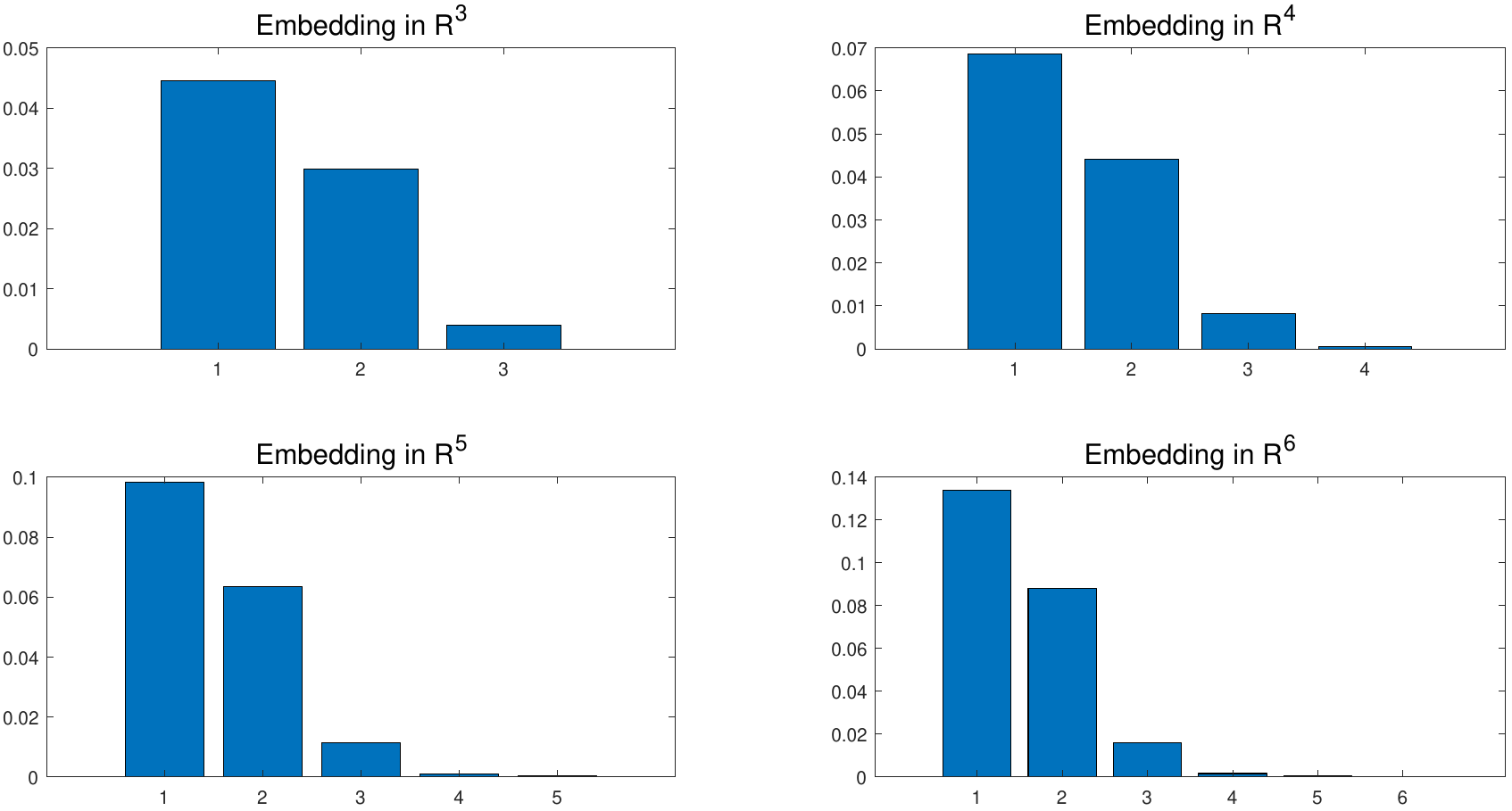}
\caption{$\mathcal{X}_1, \cdots, \mathcal{X}_4$ constructed by diffusion map are the samples around the embeddings of $M$ in $\mathbb{R}^3, \cdots, \mathbb{R}^6$, respectively. For each $ j=1,\cdots, 4$,  we construct the local covariance matrices  by 
using $\mathcal{X}_j$ with $\epsilon=0.3+0.1j$ and calculate  $\{\bar{\lambda}^j_{\epsilon,i}\}_{i=1}^{2+j}$. We plot $\{\bar{\lambda}^j_{\epsilon,i}\}_{i=1}^{2+j}$ for each $j$.  The top two plots correspond to $j=1,2$. The bottom two plots  correspond to $j=3,4$. In the above plots, the horizontal axis indicates $i=1, \cdots, 2+j$ and the vertical axis indicates the value of the corresponding $\bar{\lambda}^j_{\epsilon,i}$ .}\label{Eig2}
\end{figure}

We apply the above method to the examples in Section \ref{Numerics}.  For the Cassini Oval curve, using  the local covariance matrix and \eqref{average of eigenvalues} with the clean data $\{X(\theta_i), Y(\theta_i), Z(\theta_i)\}_{i=1}^{102}$ and $\epsilon=0.3$, we obtain $\bar{\lambda}_{\epsilon,1}=0.271$, $\bar{\lambda}_{\epsilon,2}=0.0078$, and $\bar{\lambda}_{\epsilon,3}=0.0001$. This result confirms that the clean data are sampled from a curve. In contrast, applying the local covariance matrix and \eqref{average of eigenvalues} to the noisy data $\mathcal{Y}$ with $\epsilon=0.4$ yields $\bar{\lambda}_{\epsilon,1}=0.2451$, $\bar{\lambda}_{\epsilon,2}=0.0473$, and $\bar{\lambda}_{\epsilon,3}=0.0122$. Here, $\bar{\lambda}_{\epsilon,1}$ is less dominant compared to the noiseless case due to the noise. To address this, we apply the diffusion map with $\epsilon_{DM}=0.2$ to reduce the dimension of $\mathcal{Y}$, obtaining $(V_1, V_2)$. Applying the local covariance matrix and  \eqref{average of eigenvalues} to the lower dimensional data corresponding to $(V_1, V_2)$ with $\epsilon=0.2$, we find $\bar{\lambda}_{\epsilon,1}=0.182$ and $\bar{\lambda}_{\epsilon,2}=0.0089$. Thus, our method successfully determines the dimension of the underlying manifold even in the presence of noise.

 For the torus case, using  the local covariance matrix and \eqref{average of eigenvalues} with the clean data $\{\iota(x_i)\}_{i=1}^{558}$ and $\epsilon=0.5$, we obtain $\bar{\lambda}_{\epsilon,1}=0.5898$, $\bar{\lambda}_{\epsilon,2}=0.2604$, and $\bar{\lambda}_{\epsilon,3}=0.0134$.  In contrast, applying the local covariance matrix and \eqref{average of eigenvalues} to the noisy data $\mathcal{Y}$ with $\epsilon=0.5$ yields $\bar{\lambda}_{\epsilon,1}=0.5125$, $\bar{\lambda}_{\epsilon,2}=0.2364$, and $\bar{\lambda}_{\epsilon,3}=0.068$. Here, the first two average eigenvalues are less dominant compared to the noiseless case due to the noise. We apply the diffusion map with $\epsilon_{DM}=0.3$, obtaining $(V_1, V_2, V_3, V_4)$. Applying the local covariance matrix and  \eqref{average of eigenvalues} to the  data corresponding to $(V_1, V_2, V_3, V_4)$ with $\epsilon=0.4$, we find $\bar{\lambda}_{\epsilon,1}=4.179$, $\bar{\lambda}_{\epsilon,2}=1.41$, $\bar{\lambda}_{\epsilon,3}=0.024$ and $\bar{\lambda}_{\epsilon,4}=0.0031$. 

The discussion is similar for the real projective space case in Section \ref{boundary and noise} in which we have a $3$ dimensional manifold. Using the local covariance matrix and \eqref{average of eigenvalues} with $1200$ clean samples on the manifold and $\epsilon=0.25$, we obtain $\bar{\lambda}_{\epsilon,1}=1.2\times 10^{-1}$, $\bar{\lambda}_{\epsilon,2}=5.97 \times 10^{-2}$,  $\bar{\lambda}_{\epsilon,3}=2.32 \times 10^{-2}$, $\bar{\lambda}_{\epsilon,4}=2.5\times 10^{-3}$, and $\bar{\lambda}_{\epsilon,5}=4 \times 10^{-4}$.  In contrast, applying the local covariance matrix and  \eqref{average of eigenvalues} to the noisy data $\mathcal{Y}$ with $\epsilon=0.3$ yields $\bar{\lambda}_{\epsilon,1}=1.8\times 10^{-1}$, $\bar{\lambda}_{\epsilon,2}=1.1 \times 10^{-1}$,  $\bar{\lambda}_{\epsilon,3}=5.78\times 10^{-2}$, $\bar{\lambda}_{\epsilon,4}=2.72\times 10^{-2}$, and $\bar{\lambda}_{\epsilon,5}=1.17 \times 10^{-2}$. The first three average eigenvalues are much less dominant compared to the noiseless case. We apply the diffusion map with $\epsilon_{DM}=0.25$ to reduce the dimension of $\mathcal{Y}$, resulting in $(V_1, \cdots, V_9)$. Using the local covariance matrix and  \eqref{average of eigenvalues} to the lower dimensional data corresponding to  $(V_1, \cdots, V_9)$ with $\epsilon=0.5$, we find $\bar{\lambda}_{\epsilon,1}=1.71\times 10^{-1}$, $\bar{\lambda}_{\epsilon,2}=5.17 \times 10^{-2}$,  $\bar{\lambda}_{\epsilon,3}=1.23 \times 10^{-2}$, $\bar{\lambda}_{\epsilon,4}=2.2\times 10^{-3}$, and $\bar{\lambda}_{\epsilon,5}=4.4 \times 10^{-4}$. These average eigenvalues closely match in order with those obtained from the clean samples, demonstrating that our proposed method significantly improves upon the results of directly applying local covariance matrix to the noisy data.

\section{An example for selection of $\epsilon$ and $\delta$ in \texttt{MrGap}}\label{bandwidth selection example}
In this section, we illustrate our method for selecting the parameters $\epsilon$ and $\delta$ in the \texttt{MrGap} algorithm. We consider $\iota(M)$ which is the unit circle in $\mathbb{R}^2$. 
We uniformly sample $100$ points on the circle, each with additive Gaussian noise with $\sigma=0.06$. Let $\mathcal{Y}=\{{y}_i\}_{i=1}^{100}$ be the noisy data around the circle. We also sample $\mathcal{Y}_{true}$ consisting of $10^5$ points on the circle. We use $\mathcal{Y}_{true}$ to approximate the geometric root mean square error from samples to $\iota(M)$ with $G(\mathcal{Y},\mathcal{Y}_{true})=0.0588$. Let $\epsilon=\{0.1, 0.2,\cdots, 1.8 \}$ and $\delta=\{0.1, 0.2,\cdots, 1.8 \}$. We consider all pairs $(\epsilon_i, \delta_j)$ with $\epsilon \leq \delta$. Therefore, there are $171$ distinct pairs. For each pair $(\epsilon_i, \delta_j)$, we maximize $L_{ij}=\sum_{k=1}^n \log  p_k ({Z}_k| A, \rho, \sigma)$ in Step 3 of Algorithm \ref{MrGap1} over $A, \rho, \sigma$.  We choose $(\epsilon_i, \delta_j)$ corresponding to the largest $L_{ij}$. When $\epsilon=1.4$ and $\delta=1.4$, $\frac{L_{ij}}{100}$ achieves maximum $338$ over all pairs of $(\epsilon_i, \delta_j)$.  The corresponding Gaussian process parameters are $A^{(0)}=0.9$, $\rho^{(0)}=1.2$ and $\sigma^{(0)}=\sqrt{0.008}$. For $\epsilon=1.4$ and $\delta=1.4$, suppose $\mathcal{X}_1$ is the denoised output after the first round of Algorithm \ref{MrGap1}. Then, $G(\mathcal{X}_1,\mathcal{Y}_{true})=0.0251$. We plot the $\frac{L_{ij}}{100}$ and the corresponding first round denoised outputs for different pairs of $(\epsilon_i, \delta_j)$ in Figure \ref{circle selection}.

\begin{figure}[htb!]
\centering
\includegraphics[width=0.9 \columnwidth]{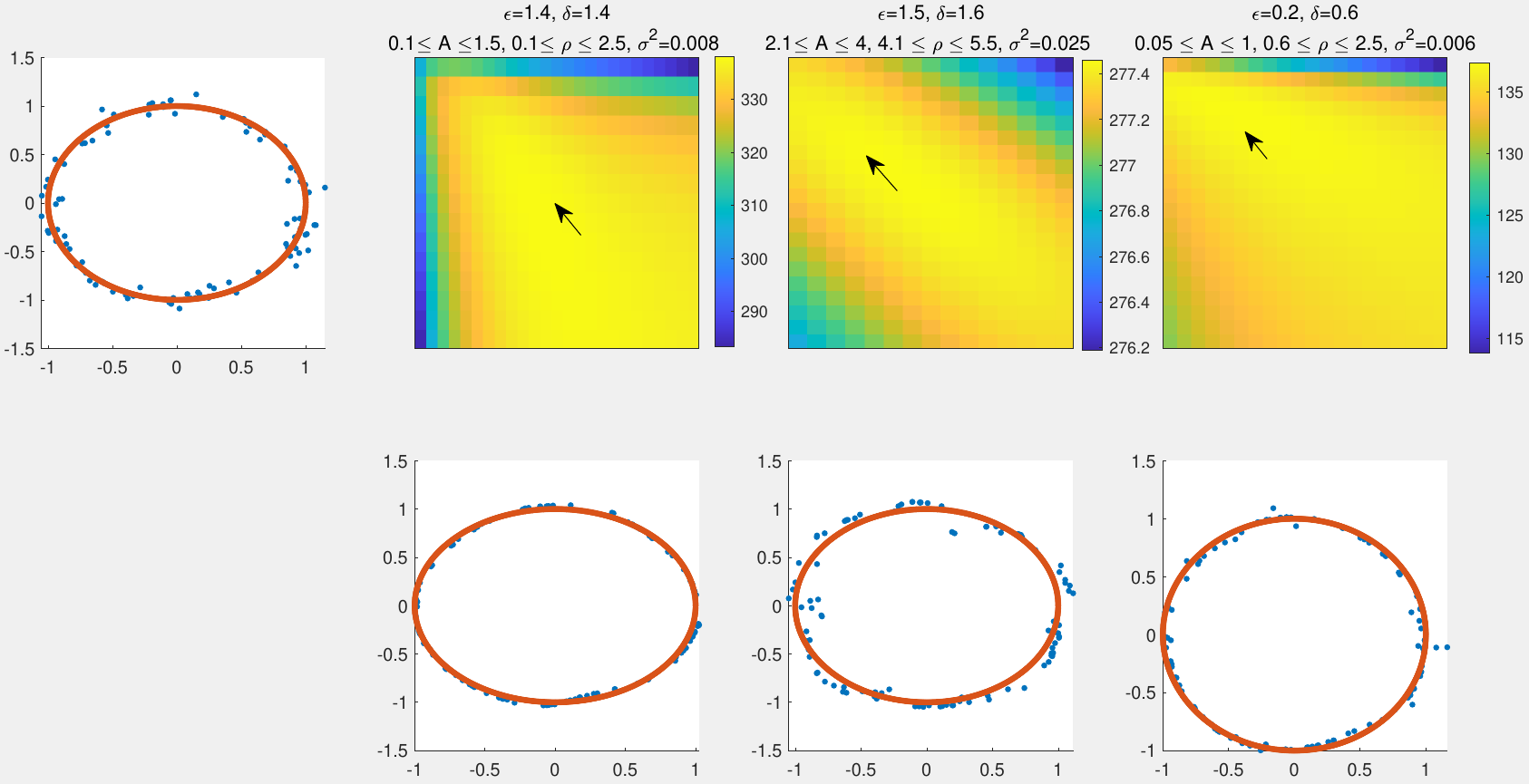}
\caption{Top row: The first plot shows the noisy data $\mathcal{Y}$ and the samples on the unit circle $\mathcal{Y}_{true}$. $G(\mathcal{Y},\mathcal{Y}_{true})=0.0588$.  For each pair $(\epsilon_i, \delta_j)$, suppose $\frac{L_{ij}}{100}$ achieves maximum with $A_{ij}$, $\rho_{ij}$ and $\sigma_{ij}$.  In the second to the fourth plot, we fix $\sigma_{ij}$ and we plot $\frac{L_{ij}}{100}$ over $A$ and $\rho$ in a neighborhood around $A_{ij}$ and $\rho_{ij}$. The horizontal axis is $\rho$ and the vertical axis is $A$. The arrows indicate the locations of the maximums. The maximums of $\frac{L_{ij}}{100}$ in the plots are $338$, $277.46$, and $137.3$ respectively. Bottom row: The plots of the first round denoised outputs correspond to $\epsilon_i$, $\delta_j$, $A_{ij}$, $\rho_{ij}$ and $\sigma_{ij}$ in the top row. The geometric root mean square error between the first round denoised outputs and $\mathcal{Y}_{true}$ are $0.0251$, $0.0769$, and $0.0366$ respectively. }\label{circle selection}
\end{figure}

\section{Exploration of the properties of \texttt{MrGap} on numerical examples}\label{numerical example curve}

\subsection{Discussion of the affine subspaces constructed in Step 1 of  Algorithm \ref{MrGap1} }\label{tangent space comparison}
Suppose that $\{\iota(x_i)\}_{i=1}^n$ are samples on $\iota(M)$ and $\{y_i\}_{i=1}^n$ are the corresponding noisy data points around $\iota(M)$. Recall from the description in Section \ref{denoise phase introduction} and Step 1 of Algorithm \ref{MrGap1} that we construct an affine subspace $\mathcal{H}_k = y_k + \mathcal{T}_k$ through $y_k$, where $\mathcal{T}_k$ is the subspace spanned by the first $d$ eigenvectors of $C_{n,\epsilon}(y_k)$ defined in \eqref{Gaussian local covariance matrix}. In this subsection, we present an example showing that Algorithm \ref{MrGap1} performs well even when the affine subspace $\mathcal{H}_k$ deviates significantly from the affine space $\iota(x_k) + T_{\iota(x_k)}\iota(M)$ tangent to $\iota(M)$ at $\iota(x_k)$. This example illustrates the conclusion of Theorem \ref{Construction of a chart main theorem}, namely, that Algorithm \ref{MrGap1} remains valid as long as $\mathcal{P}_{y_k}(y)$, defined in \eqref{regression PYK 0}, is a local diffeomorphism, regardless of the deviation between $\mathcal{H}_k$ and the corresponding affine space tangent to $\iota(M)$.

As described in Section \ref{Cassini Oval main article}, let $\iota(M)$ denote the Cassini Oval in $\mathbb{R}^3$.  Let $\{\iota(x_i)\}_{i=1}^{102}\subset \mathbb{R}^3$ and $\mathcal{Y}=\{y_i\}_{i=1}^{102}\subset \mathbb{R}^3$ be the points on $\iota(M)$ and the corresponding noisy data points respectively.  Following Section \ref{Cassini Oval main article}, we choose the scale parameters $\epsilon=0.3$ and $\delta=0.6$.
Let $\{\hat{y}^{(1)}_i\}_{i=1}^{102}$ be the outputs after the first iteration of Algorithm \ref{MrGap1}.
We focus on the point $\iota(x_{91})$. Define $\mathcal{H}_{91}= y_{91} + \mathcal{T}_{91}$ as the $1$ dimensional affine subspace through $y_{91}$, where $\mathcal{T}_{91}$ is the subspace generated by the first eigenvectors of $C_{102, 0.3}(y_{91})$ constructed from $\mathcal{Y}$.   Similarly, let $\hat{\mathcal{H}}_{91}=\hat{ y}^{(1)}_{91} + \hat{\mathcal{T}}_{91}$ be the $1$ dimensional affine subspace through $\hat{ y}^{(1)}_{91} $, where $ \hat{\mathcal{T}}_{91}$ is the subspace spanned by the first eigenvector of $C_{102, 0.3}(\hat{ y}^{(1)}_{91} )$ constructed from $\{\hat{y}^{(1)}_i\}_{i=1}^{102}$.  Let  $\iota(x_{91}) + T_{\iota(x_{91})}\iota(M)$ be the affine subspace tangent to $\iota(M)$ at $\iota(x_{91})$. The deviation between two affine subspaces can be quantified by the angle between them, which lies in $[0, \pi/2]$. The angle between $\mathcal{T}_{91}$ and $ T_{\iota(x_{91})}\iota(M)$ is $0.613\approx 35.6^{\circ}$. After the first iteration of Algorithm \ref{MrGap1}, the outputs are less noisy, and we indeed observe a smaller angle between $\hat{\mathcal{T}}_{91}$ and $T_{\iota(x_{91})}\iota(M)$, namely $0.233 \approx 13.3^{\circ}$. We plot $\mathcal{H}_{91}$ and $\hat{\mathcal{H}}_{91}$ , together with $\iota(x_{91}) + T_{\iota(x_{91})}\iota(M)$ in Figure \ref{affine spaces comparison}. It is worth noting that, due to the finite sample size, the affine subspace at $\iota(x_{91})$ generated from the clean samples on $\iota(M)$, that is, the subspace spanned by the first eigenvector of $C_{102, 0.3}(\iota(x_{91}))$ constructed from $\{\iota(x_i)\}_{i=1}^{102}$, still does not exactly coincide with $\iota(x_{91}) + T_{\iota(x_{91})}\iota(M)$. This result together with the results in Section \ref{Cassini Oval main article} demonstrate that \texttt{MrGap}
does not rely on the accurate estimation of the affine subspaces tangent to $\iota(M)$.

\begin{figure}[htb!]
\centering
\includegraphics[width=15cm,height=4cm]{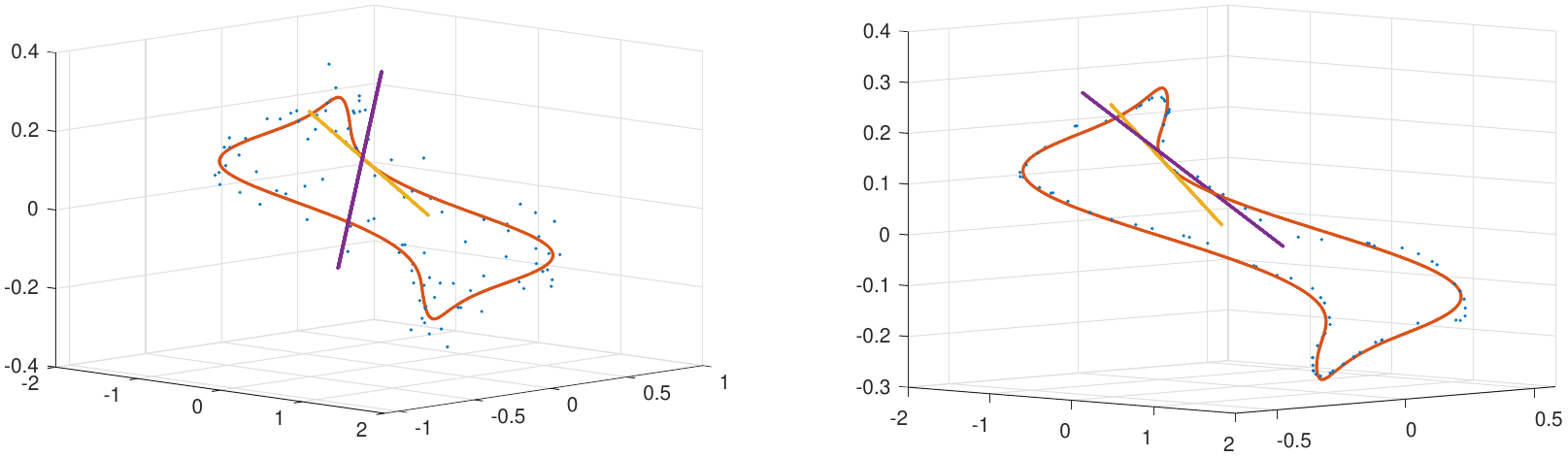}
\caption{ Left panel: The blue points represent $\mathcal{Y}=\{y_i\}_{i=1}^{102}$ and the red points represent $\iota(M)$. The yellow line and purple line represent  $\mathcal{H}_{91}$ and $\iota(x_{91}) + T_{\iota(x_{91})}\iota(M)$  respectively. The angle between them is $0.613\approx 35.6^{\circ}$. Right panel: The blue points represent $\{\hat{y}^{(1)}_i\}_{i=1}^{102}$  and the red points represent $\iota(M)$. The yellow line and purple line represent   $\hat{\mathcal{H}}_{91}$ and $\iota(x_{91}) + T_{\iota(x_{91})}\iota(M)$  respectively. The angle between them is $0.233 \approx 13.3^{\circ}$. }\label{affine spaces comparison}
\end{figure}

\subsection{Stability of Algorithm \ref{MrGap2} under different orders of interpolations}
Suppose $\{\hat{{y}}_i\}_{i=1}^{n}$ are the final denoised output from Algorithm \ref{MrGap1}.  Recall that Algorithm \ref{MrGap2} iteratively interpolates points around each $\hat{{y}}_i$ based on the index order $i$. In this subsection, we demonstrate through a numerical simulation that the output of Algorithm \ref{MrGap2} remains stable regardless of the interpolation order.

Let $\mathcal{Y}=\{{y}_i\}_{i=1}^{102}$ be the noisy data points around the Cassini Oval in Section \ref{Numerics}. Suppose  $\{{y}^{(1)}_i\}_{i=1}^{102}$ are the output after applying the first round of Algorithm \ref{MrGap1} to $\mathcal{Y}$. Next, we generate $100$ different permutations of $\{{y}^{(1)}_i\}_{i=1}^{102}$. We denote these permutations as $\{\mathcal{Y}^{(1)}_j\}_{j=1}^{100}$. For each permutation $\mathcal{Y}^{(1)}_j$, we apply Algorithm \ref{MrGap2} with $\epsilon=0.3$ and $\delta=0.6$.  We choose $K=20$ and we apply the same estimated covariance parameters in the last round of  Algorithm \ref{MrGap1} i.e.  $A^{(1)}=0.048$, $\rho^{(1)}=0.3$ and $\sigma^{(1)}=\sqrt{2 \times 10^{-5}}$. Therefore, for each permutation $\mathcal{Y}^{(1)}_j$, we construct $\tilde{\mathcal{X}}_j$ consisting of $2040$ interpolation points. Let  $\mathcal{Y}_{true}$ be the $10^5$ points sampled on the Cassini Oval described in Section \ref{Numerics}. We evaluate the performance of each group of interpolations through the geometric root mean square error between $\tilde{\mathcal{X}}_j$ and the Cassani Oval through $\mathcal{Y}_{true}$.  Let $\mathcal{E}_j=G(\tilde{\mathcal{X}}_j,\mathcal{Y}_{true})$ for $j=1, \cdots, 100$.    The mean of $\{\mathcal{E}_j\}_{j=1}^{100}$ is $0.0213$ with a small standard deviation $2.16 \times 10^{-4}$.

\subsection{Noisy data set around a submanifold with small reach}
We consider $\iota(M)$ which is a $2$-d Cassini oval with a small reach. For $\theta \in [0, 2\pi)$,
\begin{align}
& X(\theta)=\big\{\cos(2\theta)+(\cos(2\theta)^2+0.7)^{1/2}\big\}^{1/2}\cos(\theta), \nonumber \\
& Y(\theta)=\big\{1.07\cos(2\theta)+(\cos(2\theta)^2+0.2)^{1/2}\big\}^{1/2}\sin(\theta). \nonumber 
\end{align}
Let $\mathcal{Y}_{true}$ be the collection of $10^5$ points on $\iota(M)$. Suppose ${y}_i=\iota(x_i)+\eta_i$.  

In the first example, we show that to reconstruct $\iota(M)$ accurately, it is necessary that $\eta_i$ is small compared to the reach. Suppose $\mathcal{Y}=\{{y}_i\}_{i=1}^{102}$ in which $\eta_i \sim \mathcal{N}(0, {\sigma}^2 I_{3 \times 3}),$ with $\sigma=0.04$. We plot $\mathcal{Y}$ and $\mathcal{Y}_{true}$ in Figure \ref{small reach large noise}.  We choose $\epsilon=0.6$ and $\delta=0.6$ by maximizing the sum of the marginal likelihood functions in the first round of  Algorithm \ref{MrGap1}. We iterate  Steps 1-3 of Algorithm \ref{MrGap1} twice. The estimated covariance parameters in the last round  are  $A^{(1)}=0.1$, $\rho^{(1)}=0.5$ and $\sigma^{(1)}=\sqrt{0.012}$. We apply Algorithm \ref{MrGap2} with the covariance parameters and output $\mathcal{X}_2$ consisting of $612$ points. The denoised outputs $\mathcal{X}_1$ and the interpolations $\mathcal{X}_2$ are shown in Figure \ref{small reach large noise}. Given the large noise and small reach, $\mathcal{Y}$ is more likely to distribute around a figure 8 curve rather than a simple closed curve. Thus, \texttt{MrGap} improperly connects two regions on $\iota(M)$, resulting in the denoised points and interpolations falling along the figure 8 curve.
\begin{figure}[htb!]
\centering
\includegraphics[width=1 \columnwidth]{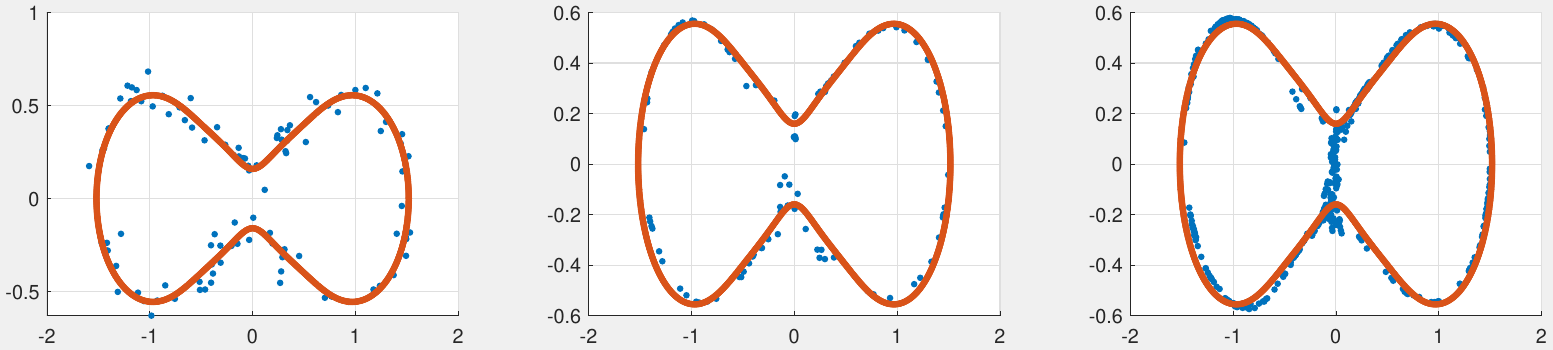}
\caption{Left: The plot of $\mathcal{Y}$(blue) and $\mathcal{Y}_{true}$(red). Middle: The plot of $\mathcal{X}_1$(blue) and $\mathcal{Y}_{true}$(red). Right: The plot of $\mathcal{X}_2$(blue) and $\mathcal{Y}_{true}$(red)}\label{small reach large noise}
\end{figure}

In the second example, we apply the same $\iota(M)$. We demonstrate that regions with large curvature cannot be adequately reconstructed when the surrounding data are too sparse. Suppose $\mathcal{Y}=\{{y}_i\}_{i=1}^{30}$ in which $\eta_i \sim \mathcal{N}(0, {\sigma}^2 I_{3 \times 3}),$ with $\sigma=0.04$. We plot $\mathcal{Y}$ and $\mathcal{Y}_{true}$ in Figure \ref{small reach sparse data}.  We choose $\epsilon=0.6$ and $\delta=0.6$ by maximizing the sum of the marginal likelihood functions in the first round of Algorithm \ref{MrGap1}. We iterate  Steps 1-3 of Algorithm \ref{MrGap1} twice. The estimated covariance parameters in the last round  are  $A^{(1)}=0.24$, $\rho^{(1)}=1$ and $\sigma^{(1)}=\sqrt{0.0005}$. We apply Algorithm \ref{MrGap2} with the covariance parameters and output $\mathcal{X}_2$ consisting of $600$ points. The denoised outputs $\mathcal{X}_1$ and the interpolations $\mathcal{X}_2$ are shown in Figure \ref{small reach sparse data}. The highlighted regions in the figure have large curvature. Due to the lack of samples in these areas, interpolations either fail to reconstruct the region or produce a flatter representation compared to the actual curve.

\begin{figure}[htb!]
\centering
\includegraphics[width=1 \columnwidth]{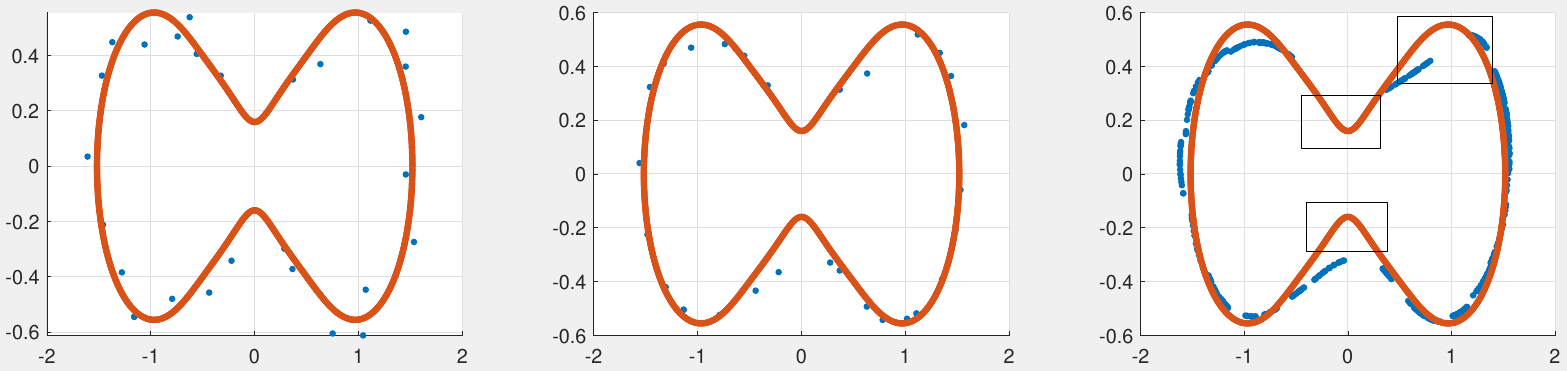}
\caption{Left: The plot of $\mathcal{Y}$(blue) and $\mathcal{Y}_{true}$(red). Middle: The plot of $\mathcal{X}_1$(blue) and $\mathcal{Y}_{true}$(red). Right: The plot of $\mathcal{X}_2$(blue) and $\mathcal{Y}_{true}$(red)}\label{small reach sparse data}
\end{figure}
  
\subsection{ \texttt{MrGap} on disconnected manifolds}\label{section:disconnected}
Recall from Assumption \ref{manifod with noise} that we assume the manifold $M$ is connected. However, the assumption of connectedness is made for simplicity in the statements of the main theorems and is not necessary for the functionality of \texttt{MrGap}. We now introduce the following assumption:
\begin{assumption}
Suppose $M=\cup_{i=1}^m M_i$ where each $M_i$ is a $d$-dimensional smooth, closed, and connected Riemannian manifold and the $m$ manifolds are mutually disjoint. $M$ is isometrically embedded in $\mathbb{R}^D$ though $\iota: M\to \mathbb{R}^D$. $P$ is a probability density function on $M$, which is smooth, bounded from below by $P_{m}>0$ and bounded from above by $P_{M}$. The observed data follow
$y_i=\iota(x_i)+\eta_i$, with $x_i \sim P$ and $\eta_i \sim \mathcal{N}(0, \sigma^2 I_{D \times D})$ independently for $i=1,\ldots,n$.  
\end{assumption}

Because the $m$ manifolds are disjoint, we have 
\begin{equation*}
\mathfrak{m}_0 =\min_{i \not= j}\min_{x\in M_i, x' \in M_j}\| \iota(x)-\iota(x')\|_{\mathbb{R}^D}> 0.
\end{equation*} 
Let $\tau_{\iota(M_i)}$ denote the reach of the embedded manifold $\iota(M_i)$. Suppose the relations between $\epsilon$, $\sigma$, $\delta$, and $n$ in Theorem \ref{Construction of a chart main theorem} hold. Further assume that $\epsilon<\mathfrak{m}_0 /3$ and $\delta < \min(\mathfrak{m}_0 /3, \tau_{\iota(M_1)}/16, \cdots, \tau_{\iota(M_m)}/16)$. Under these conditions, the conclusions of Theorem \ref{Construction of a chart main theorem} hold. In particular,  if $x_k \in M_i$, then with high probability, $B^{\mathbb{R}^D}_{3\delta}(y_k) \cap \iota(M_j) =\emptyset$ for $j\not=i$, and  $B^{\mathbb{R}^D}_{\epsilon}(y_k)$ and $B^{\mathbb{R}^D}_{\delta}(y_k) $ do not contain any $y_{k'}$  corresponding to $x_{k'} \in M_j$ for $j\not=i$. Consequently, $C_{n,\epsilon}(y_k)$ is  constructed only using $y_{k'}$ corresponding to $x_{k'}$ on $M_i$.  The map $\mathcal{P}_{y_k}(y)$ as defined in \eqref{regression PYK 0} is a diffeomorphism from $B^{\mathbb{R}^D}_{3\delta}(y_k) \cap \iota(M_i)$ onto its image $O_k \subset \mathbb{R}^d$, and $O_k$ is homeomorphic to $B^{\mathbb{R}^d}_{1}(0)$.  In other words, applying \texttt{MrGap} directly over $\{y_i\}_{i=1}^n$ to reconstruct $\iota(M)$ is equivalent to applying \texttt{MrGap} to reconstruct each $\iota(M_i)$ separately.  We illustrate the performance of \texttt{MrGap} on the following disconnected manifold without testing whether $M$ is connected.

Let $\iota(M)=S_1 \cup S_2 \subset \mathbb{R}^3$ be a pair of linked circles in $\mathbb{R}^3$, where $S_1$ is the unit circle in the XY plane centered at $(0,0,0)$ and $S_2$ is the unit circle in the YZ plane centered at $(0,1,0)$.  Let $\{\iota(x_i)\}_{i=1}^{100}$ and $\{\iota(x_i)\}_{i=101}^{200}$ be the uniform samples distributed on $S_1$ and $S_2$ respectively. Suppose $\eta_{i} \sim \mathcal{N}(0, {\sigma}^2 I_{3 \times 3}),$ with $\sigma=0.07$.  Suppose $y_i=\iota(x_i)+\eta_i$, $i=1, \cdots, 200$, and $\mathcal{Y}=\{y_i\}_{i=1}^{200}$. We uniformly sample $\mathcal{Y}_{true}$ consisting of $5 \times 10^5$ points on $\iota(M)$.  For any sample points $\mathcal{S}$,  $G\{\mathcal{S}, \iota(M)\}$ is approximated by $G(\mathcal{S},\mathcal{Y}_{true})$. Without denoising, we obtain $G(\mathcal{Y},\mathcal{Y}_{true})=0.098$. 

We choose the scale parameters $\epsilon=0.6$ and $\delta=0.6$  in \texttt{MrGap}. We iterate Steps 1-3 of Algorithm \ref{MrGap1} twice. The estimated covariance parameters  in the first round are  $A^{(0)}=0.11$, $\rho^{(0)}=1.2$ and $\sigma^{(0)}=\sqrt{0.007}$, while the values in the last round are  $A^{(1)}=0.22$, $\rho^{(1)}=1$ and $\sigma^{(1)}=\sqrt{1 \times 10^{-5}}$.  The denoised outputs are $\mathcal{X}_1=\{\hat{y}_i\}_{i=1}^{200}$, with $G(\mathcal{X}_1,\mathcal{Y}_{true})=0.0321$. In the interpolation phase, we construct $200$ charts and interpolate $K=10$ points in each chart. 
The outputs are $\mathcal{X}_2=\{\tilde{y}_i\}_{i=1}^{2000}$ with  $G(\mathcal{X}_2,\mathcal{Y}_{true})=0.0304$.   A comparison of $\mathcal{Y}$, $\mathcal{X}_1$, and $\mathcal{X}_2$ with $\mathcal{Y}_{true}$ is provided in Figure \ref{linkedcircles}.   

\begin{figure}
\begin{subfigure}
\centering
\includegraphics[width=15cm,height=4cm]{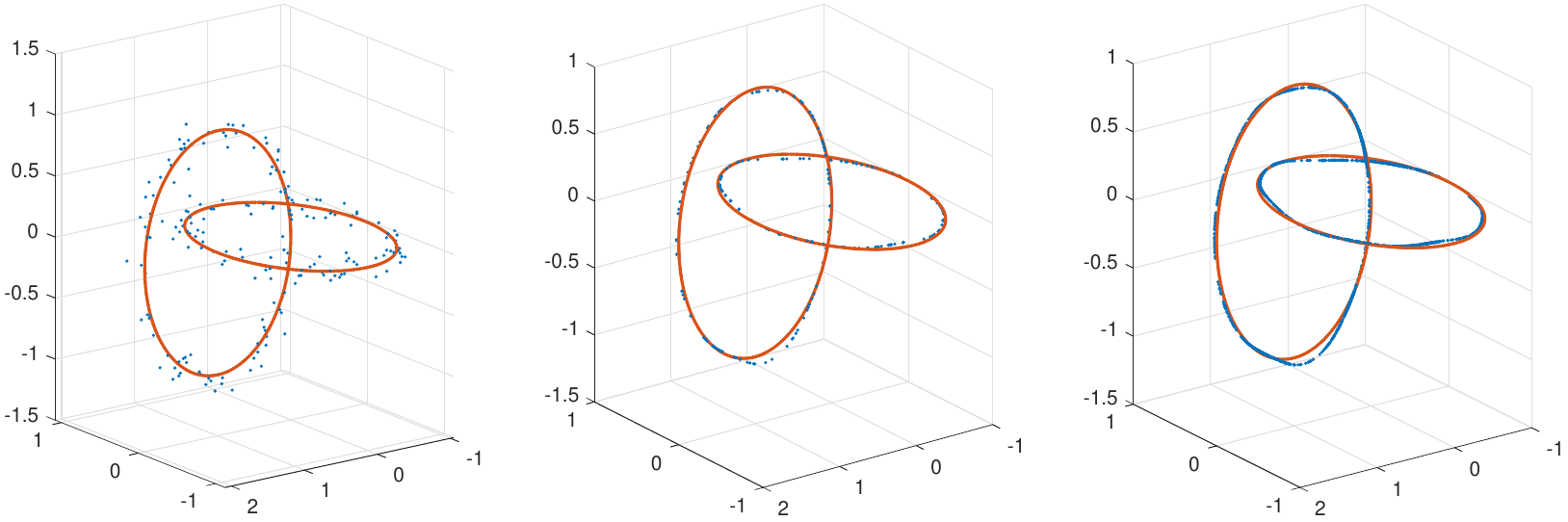}
\end{subfigure}
\caption{Red curves represent $\mathcal{Y}_{true}$ consisting of $5 \times 10^5$ points on $\iota(M)$. Left panel shows data $\mathcal{Y}$ (blue) and $\mathcal{Y}_{true}$ ($G(\mathcal{Y},\mathcal{Y}_{true})=0.098$). Middle panel shows denoised samples $\mathcal{X}_1$  (blue) by \texttt{MrGap}  and $\mathcal{Y}_{true}$ ($G(\mathcal{X}_1,\mathcal{Y}_{true})=0.0312$).  Right panel shows interpolations $\mathcal{X}_2$  (blue) by \texttt{MrGap} and $\mathcal{Y}_{true}$ ($G(\mathcal{X}_2,\mathcal{Y}_{true})=0.0304$). }\label{linkedcircles}
\end{figure}

Consider the more general case, when $M=\cup_{i=1}^m M_i$ where each $M_i$ is a $d_i$-dimensional smooth, closed, and connected Riemannian manifold and the $m$ manifolds are mutually disjoint. $M$ is isometrically embedded in $\mathbb{R}^D$ though $\iota: M\to \mathbb{R}^D$. Since each connected component may have a different dimension,  \texttt{MrGap} cannot be applied directly. Given the noisy data $\mathcal{Y}$ around $M$, we propose the following method:
\begin{enumerate}
\item
Apply spectral clustering \citep{ng2001spectral} or  topological data analysis methods \citep{niyogi2008finding, Fasy2014, robinson2017hypothesis, meng2022randomness} to separate $\mathcal{Y}$ into groups $\{\mathcal{Y}_i\}_i^m$.
\item
For each group $\mathcal{Y}_i$, apply the method in Section \ref{determine the dimension} to detect the dimension $d_i$ of the corresponding manifold $M_i$. 
\item Apply \texttt{MrGap} to each $\mathcal{Y}_i$ with the corresponding detected dimension $d_i$.
\end{enumerate}

We illustrate the above method by the following simulation. Let $\iota(M)=\iota(M_1) \cup \iota(M_2) \subset \mathbb{R}^3$, where $\iota(M_1)$ is the circle of radius $1.6$ in the XY plane centered at the origin and $\iota(M_2)$ is the sphere of radius $0.8$ centered at the origin.  Let $\{\iota(x_i)\}_{i=1}^{230}$ and $\{\iota(x_i)\}_{i=231}^{430}$ be the uniform samples distributed on $\iota(M_1)$ and $\iota(M_2)$ respectively. Suppose $\eta_{i} \sim \mathcal{N}(0, {\sigma}^2 I_{3 \times 3}),$ with $\sigma=0.08$.  Suppose $y_i=\iota(x_i)+\eta_i$, $i=1, \cdots, 430$, and $\mathcal{Y}=\{y_i\}_{i=1}^{430}$. We uniformly sample $\mathcal{Y}_{true}$ consisting of $4 \times 10^6$ points on $\iota(M)$ with $1.5 \times 10^6$ points on $\iota(M_1)$ and $2.5 \times 10^6$ points on $\iota(M_2)$.  For any sample points $\mathcal{S}$,  $G\{\mathcal{S}, \iota(M)\}$ is approximated by $G(\mathcal{S},\mathcal{Y}_{true})$. Without denoising, we obtain $G(\mathcal{Y},\mathcal{Y}_{true})=0.0994$. 

We apply the spectral clustering method to $\mathcal{Y}$. We construct the random walk graph Laplacian matrix from $\mathcal{Y}$ and find the eigenvector corresponding to the second smallest eigenvalue of the matrix. The eigenvector is plotted over $\mathcal{Y}$  in Figure \ref{sphereandcircles}. The values of the eigenvector’s entries partition $\mathcal{Y}$ into two groups $\mathcal{Y}_1$ and $\mathcal{Y}_2$. We apply the method in Section \ref{determine the dimension} to $\mathcal{Y}_1$ and $\mathcal{Y}_2$ respectively, and detect the dimensions $d_1=1$ and $d_2=2$.   

We apply  \texttt{MrGap} to $\mathcal{Y}_1$ and $\mathcal{Y}_2$ separately. For $\mathcal{Y}_1$, we choose the scale parameters $\epsilon=0.8$ and $\delta=0.8$. We iterate Steps 1-3 of Algorithm \ref{MrGap1} twice. The estimated covariance parameters  in the first round are  $A^{(0)}=0.4$, $\rho^{(0)}=4.2$ and $\sigma^{(0)}=\sqrt{0.005}$, while the values in the last round are  $A^{(1)}=0.21$, $\rho^{(1)}=0.8$ and $\sigma^{(1)}=\sqrt{1 \times 10^{-5}}$.  The denoised outputs are $\mathcal{X}^1_1=\{\hat{y}_i\}_{i=1}^{230}$. In the interpolation phase, we construct $230$ charts and interpolate $K=10$ points in each chart. The outputs are $\mathcal{X}^1_2=\{\tilde{y}_i\}_{i=1}^{2300}$. For $\mathcal{Y}_2$, we choose the scale parameters $\epsilon=0.9$ and $\delta=0.9$. We iterate Steps 1-3 of Algorithm \ref{MrGap1} twice. The estimated covariance parameters  in the first round are  $A^{(0)}=0.74$, $\rho^{(0)}=3.5$ and $\sigma^{(0)}=\sqrt{0.007}$, while the values in the last round are  $A^{(1)}=1.9$, $\rho^{(1)}=5$ and $\sigma^{(1)}=\sqrt{0.002}$.  The denoised outputs are $\mathcal{X}^2_1=\{\hat{y}_i\}_{i=231}^{430}$. In the interpolation phase, we construct $200$ charts and interpolate $K=10$ points in each chart. The outputs are $\mathcal{X}^2_2=\{\tilde{y}_i\}_{i=2301}^{4300}$. We have $G(\mathcal{X}^1_1\cup\mathcal{X}^2_1,\mathcal{Y}_{true})=0.0337$ and $G(\mathcal{X}^1_2 \cup\mathcal{X}^2_2,\mathcal{Y}_{true})=0.0347$
\begin{figure}
\begin{subfigure}
\centering
\includegraphics[width=15cm,height=4cm]{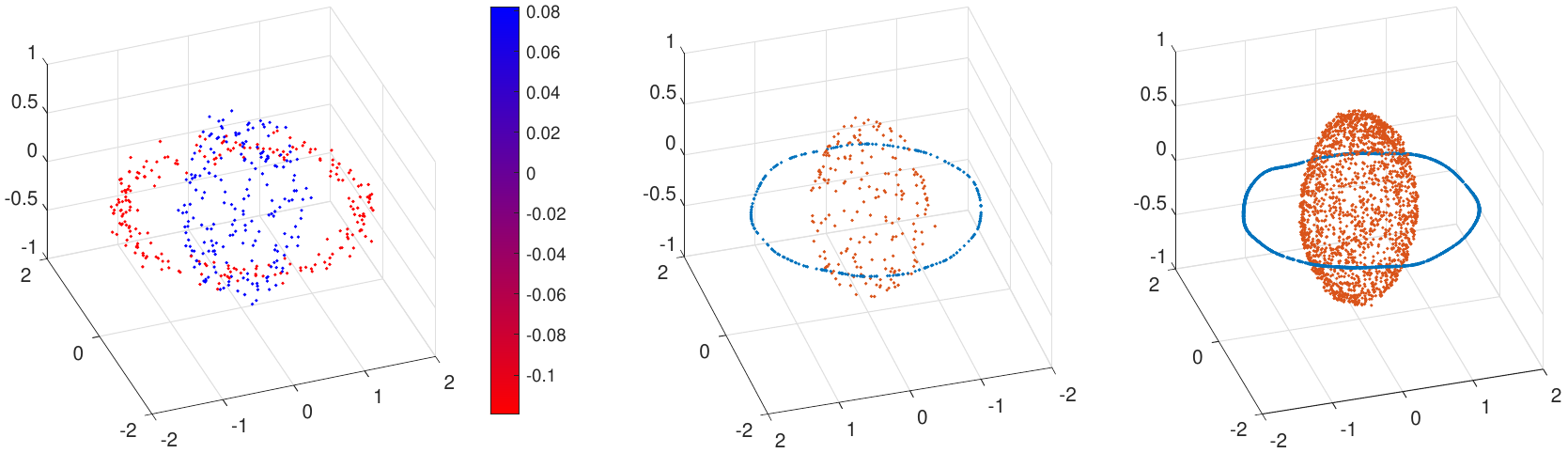}
\caption{Left panel shows the plot of the eigenvector corresponding to the second smallest eigenvalue of the random walk graph Laplacian matrix over $\mathcal{Y}$.  The values of the eigenvector’s entries  partition  $\mathcal{Y}$ into two groups $\mathcal{Y}_1$ and $\mathcal{Y}_2$. Middle panel shows denoised samples $\mathcal{X}^1_1$ (blue) and $\mathcal{X}^2_1$  (red) correspondingly to $\mathcal{Y}_1$ and $\mathcal{Y}_2$. ($G(\mathcal{X}^1_1\cup\mathcal{X}^2_1,\mathcal{Y}_{true})=0.0337$).  Right panel shows interpolations $\mathcal{X}^1_2$ (blue) and $\mathcal{X}^2_2$  (red).  ($G(\mathcal{X}^1_2\cup\mathcal{X}^2_2,\mathcal{Y}_{true})=0.0347$). }\label{sphereandcircles}
\end{subfigure}
\end{figure}

\section{Additional simulations}\label{boundary and noise}
\subsection{Performance of \texttt{MrGap} on a real projective space}
In this section, we evaluate the performance of \texttt{MrGap} on a simulated example in which 
$\iota(M)$ is an embedded $3$-real projective space, $RP(3)$, in $\mathbb{R}^{10}$. The space $RP(3)$ is $3$ dimensional closed connected manifold diffeomorphic to $SO(3)$. Letting $S^3$ denote the $3$ dimensional unit sphere: $S^3=\{[z_1, z_2, z_3, z_4]^\top| z_1^2+z_2^2+z_3^2+z_4^2=1\} \subset \mathbb{R}^4$, $\iota(M)$ is explicitly constructed through the following map: 
\begin{align}
\Phi({z}):\quad \quad \quad S^3 \quad \quad \quad &\rightarrow \quad  \iota(M) \subset \mathbb{R}^{10} \\
{z}= [z_1, z_2, z_3, z_4]^\top \quad &\rightarrow \quad [z_1^2, z_2^2,z_3^2,z_4^2,z_1z_2, z_1z_3, z_1z_4, z_2z_3, z_2z_4, z_3z_4]^\top, \nonumber
\end{align}
with 
$\Phi$ inducing a diffeomorphism $\tilde{\Phi} : RP(3) \rightarrow  \iota(M) \subset \mathbb{R}^{10}$. 
We uniformly sample points $\{{z}_i\}_{i=1}^{1200}$ on $S^3$ to obtain non-uniform samples $\{\Phi({z}_i)\}_{i=1}^{1200}$ on $\iota(M)$. Suppose $\eta_{i} \sim \mathcal{N}(0, {\sigma}^2 I_{10 \times 10}),$ with $\sigma=0.04$, ${y}_i=\Phi({z}_i)+\eta_i$ for $i=1, \cdots, 1200$, and $\mathcal{Y}=\{{y}_i\}_{i=1}^{1200}$. 
We uniformly sample $\{{z}_i\}_{i=1}^{10^6}$ on $S^3$ to obtain  $\mathcal{Y}_{true}=\{\Phi({z}_i)\}_{i=1}^{10^6}$, which is used to approximate the geometric root mean square error. Without denoising, we obtain $G(\mathcal{Y},\mathcal{Y}_{true})=0.1074$.  
 
We apply the method in Section \ref{determine the dimension} of the Supplementary Material and estimate the dimension $d=3$. A detailed procedure of the estimation for this example is provided in that section.  We apply \texttt{MrGap} with $\epsilon=0.5$ and $\delta=0.5$.  We iterate Steps 1-3 of Algorithm \ref{MrGap1} twice. The estimated covariance parameters in the first round are  $A^{(0)}=0.05$, $\rho^{(0)}=1$ and $\sigma^{(0)}=\sqrt{0.0015}$, while the values in the second round are  
 $A^{(1)}=0.08$, $\rho^{(1)}=1.2$ and $\sigma^{(1)}=0.02$.  The denoised outputs are $\mathcal{X}_1=\{\hat{{y}}_i\}_{i=1}^{1200}$, with 
$G(\mathcal{X}_1,\mathcal{Y}_{true})=0.0532$.  We interpolate $K=10$ points around each of the $1200$ data points.
 The outputs are $\mathcal{X}_2=\{\tilde{{y}}_i\}_{i=1}^{12000}$ with  $G(\mathcal{X}_2,\mathcal{Y}_{true})=0.0593$. 

\subsection{\texttt{MrGap} on manifold with boundary}\label{subsection: performance on manifold with boundary}
Let $\iota(M)$ be a half torus  in $\mathbb{R}^3$ parametrized by $u\in [0, 2\pi)$ and $v \in (-\frac{\pi}{2}, \frac{\pi}{2})$
\begin{align}\label{half torus eq}
&X(u,v) =\{3+0.8\cos(u)\}\cos(v), \\
&Y(u,v) =\{3+0.8\cos(u)\}\sin(v),\nonumber \\
&Z(u,v)=0.8\sin(u).\nonumber 
\end{align}
We randomly sample $400$ points $\{\iota(x_i)\}_{i=1}^{400}$ based on the uniform density function on $\iota(M)$.  Let $\eta_{i} \sim \mathcal{N}(0, {\sigma}^2 I_{3 \times 3}),$ with $\sigma=0.12$, ${y}_i=\iota(x_i)+\eta_i$ for $i=1, \cdots, 400$, and $\mathcal{Y}=\{{y}_i\}_{i=1}^{400}$. 
We randomly sample $79585$ points based on the uniform density function on $ \iota(M)$ to form $\mathcal{Y}_{true}$.  We calculate the geometric root mean square error from $\mathcal{Y}$ to $\mathcal{Y}_{true}$ which is $G(\mathcal{Y},\mathcal{Y}_{true})=0.1159$.  We plot $\mathcal{Y}$ in Figure \ref{toruscomparison}. We apply \texttt{MrGap}  with $\epsilon=1$ and $\delta=1$. We iterate Steps 1-3 of Algorithm \ref{MrGap1} twice.  The estimated covariance parameters in the last round are $A^{(1)}=1.3$, $\rho^{(1)}=5$ and $\sigma^{(1)}=\sqrt{0.002}$. The final denoised outputs are $\mathcal{X}_1=\{\hat{{y}}_i\}_{i=1}^{400}$ with $G(\mathcal{X}_1,\mathcal{Y}_{true})=0.0646$.  We plot $\mathcal{X}_1$  in Figure \ref{toruscomparison}.  Note that $\eta_i$ can be decomposed into two components: $\eta_i^\top$ which is tangent to $\iota(M)$ at $\iota(x_i)$ and $\eta_i^\bot$ which is perpendicular to $\iota(M)$ at $\iota(x_i)$. Due to the boundary property of the manifold, when $\iota(x_i)$ is near the boundary of $\iota(M)$, the $\delta$ neighborhood around the corresponding noisy data ${y}_i$ may contain fewer data points than when $\iota(x_i)$ is away from the boundary. Since \texttt{MrGap} relies on a Gaussian process to reconstruct a chart associated with ${y}_i$, we may have fewer predictors and responses for the regression of the Gaussian process, which could reduce its effectiveness. Thus, the denoised data points may lie in an outward extension of $\iota(M)$ from the boundary, as we show in region 1 of Figure \ref{toruscomparison}. We choose $K=20$ in applying Algorithm \ref{MrGap2}.  The outputs are $\mathcal{X}_2=\{\tilde{{y}}_i\}_{i=1}^{8000}$ with $G(\mathcal{X}_2,\mathcal{Y}_{true})=0.0634$.  We plot $\mathcal{X}_2$  in Figure \ref{toruscomparison}.  
\begin{figure}[htb!]
\centering
\includegraphics[width=0.9 \columnwidth]{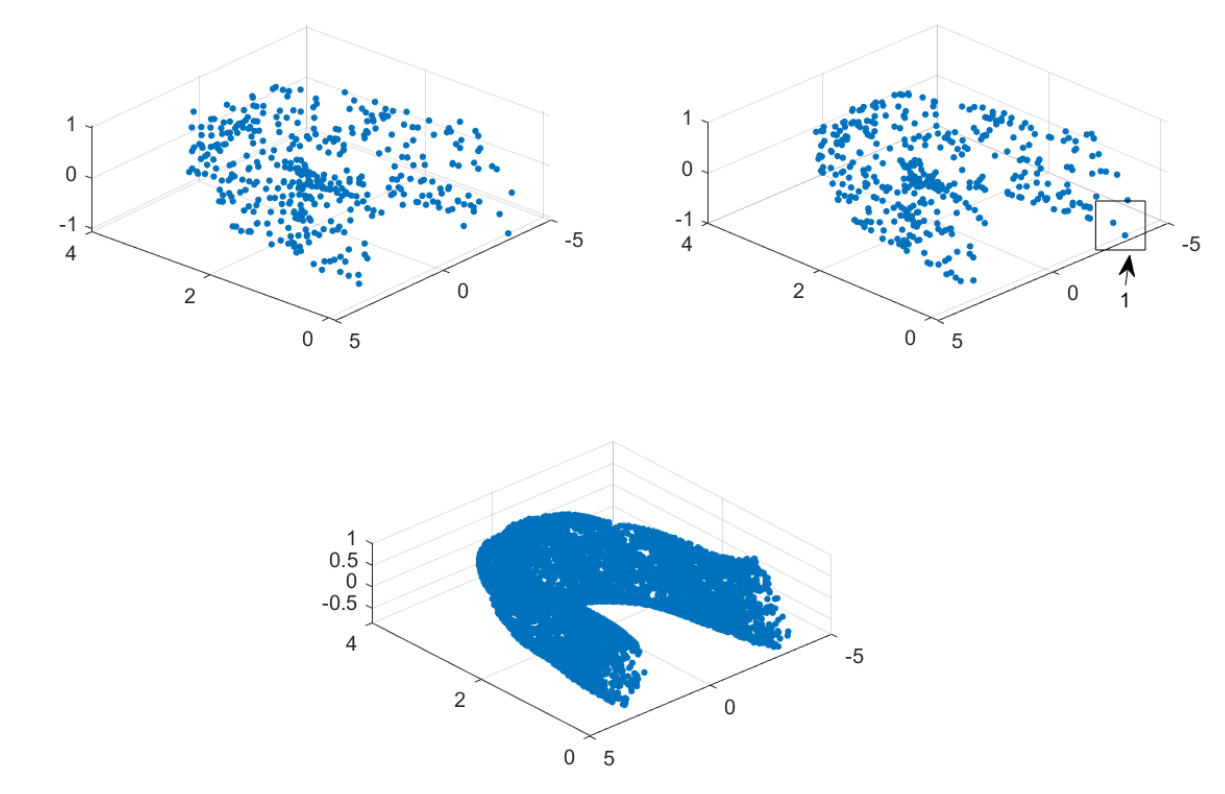}
\caption{$\mathcal{Y}$ contains $400$ noisy points around the half torus $\iota(M)$ and $\mathcal{Y}_{true}$ contains $79585$ points on the torus $\iota(M)$. $\mathcal{X}_1$ contains $400$ denoised points on $\iota(M)$ by the \texttt{MrGap} algorithm. $\mathcal{X}_2$ contains $8000$ interpolated points on $\iota(M)$ by the \texttt{MrGap} algorithm. We use $\mathcal{Y}_{true}$ to estimate the geometric root mean square error from samples to $\iota(M)$. Top left: Original data $\mathcal{Y}$ with $G(\mathcal{Y},\mathcal{Y}_{true})=0.1159$. Top right: Fits from \texttt{MrGap} $\mathcal{X}_1$ with $G(\mathcal{X}_1,\mathcal{Y}_{true})=0.0646$. Region 1 includes the denoised points on an extension $\iota(M)$ from the boundary. Bottom: Interpolated fits from \texttt{MrGap} $\mathcal{X}_2$ with $G(\mathcal{X}_2,\mathcal{Y}_{true})=0.0634$.}\label{toruscomparison}
\end{figure}
\subsection{\texttt{MrGap} on data with bounded noise}
The idea of the proof of the main result, Theorem \ref{Construction of a chart main theorem}, can be summarized as follows. Consider two conditions: (1) ${y}_i$ should not deviate significantly from $\iota(x_i)$ compared to the reach of $\iota(M)$, $\tau_{\iota(M)}$, and (2) there exists a $d$ dimensional affine subspace through ${y}_i$ with an orthonormal basis deviating from an orthonormal  basis of the tangent space at $\iota(x_i)$ by an order $1$ term depending on $d$. When these conditions are met, this affine subspace can be used to construct a chart for $\iota(M)$. When $\eta_i$ is a bounded error rather than Gaussian,  condition (1) can be satisfied by proposing an error bound depending on the reach, while the behavior of orthonormal eigenvectors of $C_{n,\epsilon}({y}_i)$ can be explored by a similar method as in Theorem \ref{local PCA  spectral behavior}. Hence, we expect the result of Theorem \ref{Construction of a chart main theorem} to hold with suitable relations between the error bound on $\eta_i$, $\epsilon$, $\delta$, $n$, and $\tau_{\iota(M)}$.  We illustrate the performance of \texttt{MrGap} when the noise is bounded in the following example.

Let $\iota(M)$ be a half torus in $\mathbb{R}^3$ described by \eqref{half torus eq}. Suppose $\mathcal{Y}_{true}$ consists of $79585$ points on $\iota(M)$. $\iota(M)$ is contained in a larger box $[-4.3,4.3] \times [-4.3,4.3] \times [-1.2, 1.2]$ in $\mathbb{R}^3$.  We sample $3000$ points $\mathcal{Y}_b$ non-uniformly from this box. $\mathcal{Y}$ consists of $319$ points $\{{y}_i \}_{i=1}^{319}$ such that if ${y}_i  \in \mathcal{Y}_b$ and $G({y}_i,\mathcal{Y}_{true})<0.25$, then ${y}_i \in \mathcal{Y}$.  Hence, $\mathcal{Y}$ are the points around $\iota(M)$ with bounded non-isotropic noise. We calculate the geometric root mean square error from $\mathcal{Y}$ to $\mathcal{Y}_{true}$ which is $G(\mathcal{Y},\mathcal{Y}_{true})=0.1201$. We apply \texttt{MrGap} with $\epsilon=1$ and $\delta=1$. We iterate  Steps 1-3 of Algorithm \ref{MrGap1} twice.  The estimated covariance parameters in the last round  are $A^{(1)}=0.6$, $\rho^{(1)}=3$ and $\sigma^{(1)}=\sqrt{0.002}$. The final denoised outputs are $\mathcal{X}_1=\{\hat{{y}}_i\}_{i=1}^{319}$ with $G(\mathcal{X}_1,\mathcal{Y}_{true})=0.0698$. At last, we choose $K=20$ in applying Algorithm \ref{MrGap2}.  The outputs are $\mathcal{X}_2=\{\tilde{{y}}_i\}_{i=1}^{6380}$ with  $G(\mathcal{X}_2,\mathcal{Y}_{true})=0.0524$. 

\bibliographystyle{plain}
\bibliography{bib}
\end{document}